\documentclass{article}

    \usepackage[final]{neurips_2022}

\usepackage[utf8]{inputenc} % allow utf-8 input
\usepackage[T1]{fontenc}    % use 8-bit T1 fonts
\usepackage[backref=page]{hyperref}       % hyperlinks
\usepackage{url}            % simple URL typesetting
\usepackage{booktabs}       % professional-quality tables
\usepackage{amsfonts}       % blackboard math symbols
\usepackage{nicefrac}       % compact symbols for 1/2, etc.
\usepackage{microtype}      % microtypography
\usepackage{xcolor}         % colors

\usepackage{appendix}
\usepackage{amssymb,amsfonts,amsmath,amstext,amsthm,mathrsfs}
\usepackage{graphics,latexsym,epsfig,psfrag,wrapfig,comment,paralist}
\usepackage{xspace}
\usepackage{subcaption,bm,multirow}
\usepackage{booktabs}
\usepackage{mathdots}
\usepackage{bbold}
\usepackage{centernot}
\usepackage{prettyref}
\usepackage{listings}
\usepackage{tikz}
\usepackage{pgfplots}
\usetikzlibrary{shapes}
\usetikzlibrary{positioning, arrows, matrix, calc, patterns, fit}
\usepackage{caption}
\usepackage{array}
\usepackage{mdwmath}
\usepackage{multirow}
\usepackage{mdwtab}
\usepackage{eqparbox}
\usepackage{multicol}
\usepackage{amsfonts}
\usepackage{tikz}
\usepackage{multirow,bigstrut,threeparttable}
\usepackage{amsthm}
\usepackage{array}
\usepackage{bbm}
\usepackage{epstopdf}
\usepackage{mdwmath}
\usepackage{mdwtab}
\usepackage{eqparbox}
\usepackage{tikz}
\usepackage{latexsym}
\usepackage{amssymb}
\usepackage{bm}
\usepackage{graphicx}
\usepackage{mathrsfs}
\usepackage{epsfig}
\usepackage{psfrag}
\usepackage{setspace}

\newcommand{\x}{\mathbf x}

\newcommand{\y}{\mathbf y}

\usepackage{wrapfig}

\usepackage{cleveref}

\newtheorem{theorem}[]{Theorem}
\newtheorem{definition}[]{Definition}
\newtheorem{corollary}[]{Corollary}
\newtheorem{assumption}[]{Assumption}
\crefname{assumption}{Assumption}{assumption}
\newtheorem{lemma}[]{Lemma}

\newtheorem{remark}[]{Remark}

\usepackage{bbm}
\usepackage{graphicx}
\usepackage{amsmath,amssymb,amsthm,amsfonts}

\usepackage{paralist}
\usepackage{bm}
\usepackage{xspace}
\usepackage{url}
\usepackage{prettyref}
\usepackage{boxedminipage}
\usepackage{wrapfig}
\usepackage{ifthen}
\usepackage{color}
\usepackage{xspace}

\usepackage{amsmath,amsthm,amsfonts,amssymb}
\usepackage{mathtools}
\usepackage{hyperref}
\usepackage{graphicx}

\usepackage{hyperref}
\usepackage{nicefrac}

\newtheorem*{definition*}{Definition}

\DeclareMathOperator*{\argmax}{arg\,max}
\DeclareMathOperator*{\argmin}{arg\,min}
\usepackage{subcaption}

\usepackage[utf8]{inputenc}

\newcommand{\tv}{\mathsf{TV}}

\newcommand{\BC}{\texttt{BC} }
\newcommand{\MM}{\texttt{MM} }
\newcommand{\Nexp}{N_{\mathrm{exp}}}
\newcommand{\Nrep}{N_{\mathrm{replay}}}
\newcommand{\Drep}{D_{\mathrm{replay}}}

\newcommand{\piBC}{\pi^{\texttt{BC}}}
\newcommand{\piRE}{\pi^{\texttt{RE}}}
\newcommand{\piEMM}{\pi^{\texttt{MM}}}
\newcommand{\mmd}{\texttt{Mimic-MD} }
\newcommand{\RE}{\texttt{RE} }
\newcommand{\dem}{\pi^E}
\newcommand{\Rlint}{\mathcal{R}_{\mathrm{lin,t}}}

\newcommand{\Rlin}{\mathcal{R}_{\mathrm{lin}}}

\newcommand{\Fmax}{F_{\mathrm{max}}}
\newcommand{\Nmax}{\mathcal{N}_{\mathrm{max}}}

\newcommand{\thetaBC}{\widehat{\theta}^{\textsf{BC}}}
\newcommand{\ent}{\mathcal{E}_{\Theta,n,\delta}}

\newcommand{\eNtH}{\mathcal{E}_{\Theta_t,\Nexp,\delta/H}}

\newcommand{\eref}[1]{(\ref{#1})}
\newcommand{\sref}[1]{Sec. \ref{#1}}

\newcommand{\figref}[1]{Fig. \ref{#1}}

\usepackage{xcolor}
\definecolor{expert}{HTML}{008000}
\definecolor{error}{HTML}{f96565}

\newcommand{\norm}[1]{\left\lVert #1 \right\rVert}

\usepackage{color-edits}
\addauthor{sw}{blue}

\usepackage{thmtools}
\usepackage{thm-restate}

\usepackage{tikz}
\usetikzlibrary{arrows,calc} 
\newcommand{\tikzAngleOfLine}{\tikz@AngleOfLine}
\def\tikz@AngleOfLine(#1)(#2)#3{%
\pgfmathanglebetweenpoints{%
\pgfpointanchor{#1}{center}}{%
\pgfpointanchor{#2}{center}}
\pgfmathsetmacro{#3}{\pgfmathresult}%
}

\declaretheoremstyle[
    headfont=\normalfont\bfseries, 
    bodyfont = \normalfont\itshape]{mystyle}

\usepackage[linesnumbered,algoruled,boxed,lined]{algorithm2e}
\Crefname{algocf}{Algorithm}{Algorithms}

\title{ Minimax Optimal Online Imitation Learning via Replay Estimation }

\author{%
  Gokul Swamy\thanks{Equal contribution. Correspondence to \texttt{gswamy@cmu.edu} and \texttt{nived.rajaraman@berkeley.edu}.} \\
  Carnegie Mellon University\\
  \texttt{gswamy@cmu.edu} \\
  \And
  Nived Rajaraman$^*$ \\
  UC Berkeley \\
  \texttt{nived.rajaraman@berkeley.edu} \\
  \And
  Matthew Peng \\
  UC Berkeley \\
  \And
  Sanjiban Choudhury \\
  Cornell University \\
  \And
  J. Andrew Bagnell \\
  Aurora Innovation and Carnegie Mellon University \\
  \And
  Zhiwei Steven Wu \\
  Carnegie Mellon University
  \And
  Jiantao Jiao \\
  UC Berkeley \\
  \And
  Kannan Ramchandran \\
  UC Berkeley
}

\begin{document}

\maketitle

\begin{abstract}
Online imitation learning is the problem of how best to mimic expert demonstrations, given access to the environment or an accurate simulator. Prior work has shown that in the \textit{infinite} sample regime, exact moment matching achieves value equivalence to the expert policy. However, in the \textit{finite} sample regime, even if one has no optimization error, empirical variance can lead to a performance gap that scales with $H^2 / \Nexp$ for behavioral cloning and $H / \sqrt{\Nexp}$ for online moment matching, where $H$ is the horizon and $\Nexp$ is the size of the expert dataset. We introduce the technique of \textit{replay estimation} to reduce this empirical variance: by repeatedly executing cached expert actions in a stochastic simulator, we compute a \textit{smoother} expert visitation distribution estimate to match. In the presence of parametric function approximation, we prove a meta theorem reducing the performance gap of our approach to the \textit{parameter estimation error} for offline classification (i.e. learning the expert policy). In the tabular setting or with linear function approximation, our meta theorem shows that the performance gap incurred by our approach achieves the optimal $\widetilde{O} \left( \min({H^{3/2}} / {\Nexp}, {H} / {\sqrt{\Nexp}} \right))$ dependency, under significantly weaker assumptions compared to prior work. We implement multiple instantiations of our approach on several continuous control tasks and find that we are able to significantly improve policy performance across a variety of dataset sizes.
\end{abstract}

\section{Introduction}
 In \textit{online} imitation learning (IL), one is given access to \textit{(a)} a fixed set of expert demonstrations and \textit{(b)} an environment or simulator to perform rollouts in. Many online IL approaches fall under the umbrella of solving a \textit{moment matching} problem between learner and expert trajectory distributions \cite{ziebart2008maximum, ho2016generative},
 \begin{equation}
    \min_{\pi \in \Pi} \sup_{f \in \mathcal{F}} \mathbb{E}_{\pi}[f(s, a)] - \mathbb{E}_{\dem}[f(s, a)], \label{eq:mm}
\end{equation}
where $\mathbb{E}_\pi [\cdot]$ denotes the expectation over a random trajectory $\{ (s_1,a_1),\cdots,(s_H,a_H) \}$ generated by rolling out $\pi$. \citet{swamy2021moments} show that for an appropriate choice of $\mathcal{F}$, approximate solutions of \eqref{eq:mm} have a performance gap linear in the horizon, the gold standard for sequential problems. However, a key assumption in their work is that expert moments ($\mathbb{E}_{\dem}[f(s, a)]$) can be estimated arbitrarily well from the available demonstrations. The resulting bounds are therefore purely a function of \textit{optimization error}. Moving to the finite-sample regime introduces an additional concern: the \textit{statistical error} that stems from the randomness in the data-generating process. When a small set of demonstrations are used in an adversarial optimization procedure like \eref{eq:mm}, the learner may choose to take incorrect actions in order to match noisy moments estimated from the dataset, leading to policies that perform poorly at test-time. Ideally, one would solve this problem by querying the expert for more demonstrations, as in the work of \citet{ross2011reduction}. However, when we are unable to do so, we still have to grapple with the question of ``\textit{how can we smooth out a noisy empirical estimate of expert moments?}"
\begin{figure*}[t]
\centering
\includegraphics[width=0.9\textwidth]{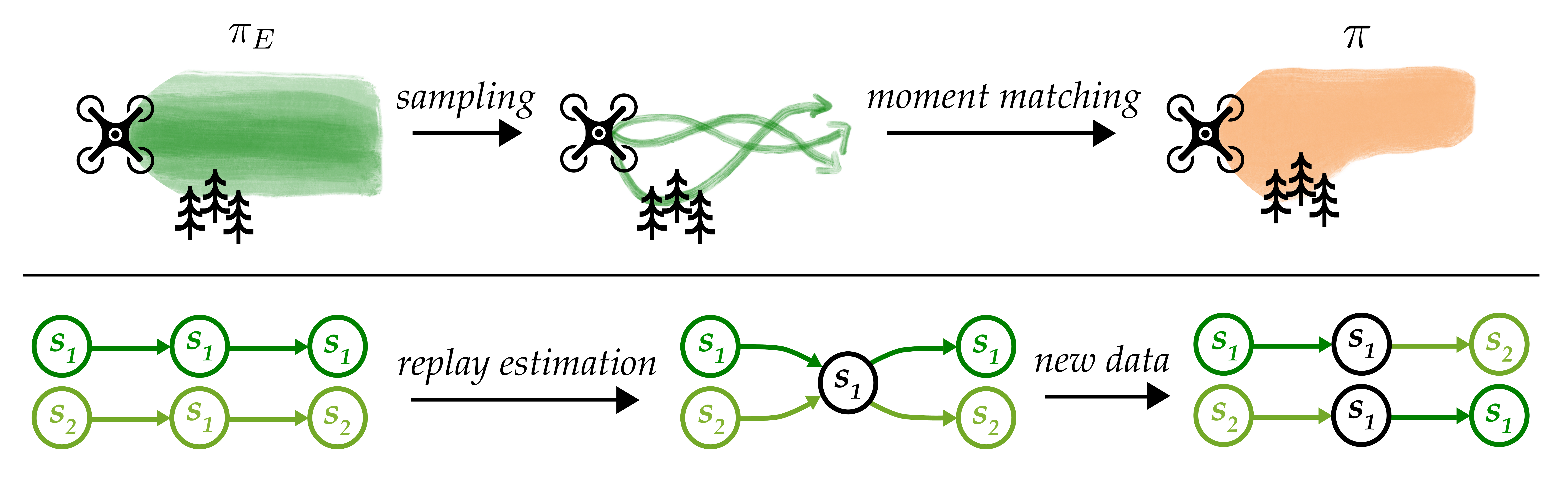}
\caption{\textit{Top}: Attempting to exactly match a finite-sample approximation of expert moments can cause a learner to reproduce chance occurrences (e.g. the relatively unlikely flight through the trees). This can lead to policies that perform poorly at test time (e.g. because the learner flies through the trees relatively often). \textit{Bottom}: Replay estimation reduces the empirical variance in expert demonstrations by repeatedly executing observed expert actions in a stochastic simulator. By generating new trajectories (e.g. $s_1 \rightarrow s_1 \rightarrow s_2$ on the right) that are consistent with expert actions, one can augment the original demonstration set and compute expert moments more accurately.
\label{fig:ffig}}
\end{figure*}

Our answer to this question is the technique of \textit{replay estimation}. In its most basic form, replay estimation consists of repeatedly executing observed expert actions within a stochastic simulator, terminating rollouts whenever one ventures out of the support of the expert demonstrations. Effectively, this approach stitches together parts of different trajectories to generate a smoothed estimate of expert moments. By using the simulator where we know the expert's actions, we can generate more diverse training data that is nevertheless consistent with the expert demonstrations. We argue that this technique is at once conceptually simple, practically feasible, and minimax optimal in several settings. Formally, we prove that in the worst case, behavioral cloning has a performance gap $\propto {H^2} / {\Nexp}$, online empirical moment matching, $\propto {H} / {\sqrt{\Nexp}}$, and our approach of replay estimation $\propto \min \{ {H^{3/2}} / {\Nexp}, {H} / {\sqrt{\Nexp}} \}$ in the tabular setting as well as with linear function approximation.
Our key insight is that \textit{\textbf{we can use a combination of simulated and empirical rollouts to optimally estimate expert moments.}} We can then plug this improved estimate into a variety of moment-matching algorithms, for strong test-time policy performance. More explicitly, our work makes the following three contributions:
\begin{enumerate}
    \item We extend replay estimation (\RE) \citet{rajaraman2020toward} beyond the tabular and deterministic setting by introducing the notion of a \textit{soft membership oracle} and \textit{prefix weights}.\vspace{-0.3em}
    \item We show how to instantiate the membership oracles for IL with parametric function approximation and prove a meta-theorem relating the imitation gap of \RE to the parameter estimation error for offline classification on the dataset (\Cref{thm:gfa}). Instantiating our main result in the case of linear function approximation, we show how to achieve the best known imitation gap of $\widetilde{O} \left( H^{3/2} d^{5/4}/\Nexp \right)$ under significantly weaker assumptions compared to prior work \cite{rajaraman2021on}. \vspace{-0.3em}
    \item We give multiple practical options for constructing performant membership oracles. We then use these approximate oracles to significantly improve the performance of online IL on several continuous control tasks across a variety of dataset sizes. We also investigate the differences between our proposed oracles.
\end{enumerate}

We sketch the benefits of replay estimation before providing theoretical and empirical evidence to support our claims.

\section{The Replay Estimator} \label{sec:RE}
We begin with a tabular vignette to illustrate our key insight in greater detail. We compare two algorithms: offline behavioral cloning (\texttt{BC}) \cite{pomerleau1989alvinn} and online moment matching (\texttt{MM}) \cite{swamy2021moments}. Throughout, we focus on learning policies from finite samples.

\textbf{Suboptimality of Empirical Moment Matching.} Consider the MDP in \figref{fig:mm_sucks}, where the expert always takes the green action. Doing so puts them in $s_1$ or $s_2$ with equal probability. Given that the expert is deterministic and there are few states, \BC could easily recover the expert's policy by learning to simply output the observed green action on both states, even when there are very few demonstrations. 

\begin{wrapfigure}{r}{0.4\textwidth}
    \begin{tikzpicture}[scale=1, transform shape]
    \node (a) [draw, very thick, circle, color=expert] at (0.0, 1) {$s_1$};
    \node (c) [draw, very thick, circle, color=expert]  at (1, 1) {$s_2$};
    \node (f) [very thick, color=expert] at (0.5, 2.25) {$\text{Unif} (\{ s_1, s_2 \})$};
    \draw [->, very thick, color=expert] (a) to[bend left] (f);
    \draw [->, very thick, color=expert] (c) to[bend right] (f);

    % Self loops
    \node [circle, minimum size=0.5cm](g) at ([{shift=(270:0.4)}]c){};
    \coordinate (h) at (intersection 2 of c and g);
    \coordinate (i) at (intersection 1 of c and g);
    \tikzAngleOfLine(g)(i){\AngleStart}
    \tikzAngleOfLine(g)(h){\AngleEnd}
    \draw[very thick,->, color=error]%
    let \p1 = ($ (g) - (i) $), \n2 = {veclen(\x1,\y1)}
    in
        (g) ++(270:0.5) node{}
        (i) arc (\AngleStart-360:\AngleEnd:\n2);
    \end{tikzpicture}
\includegraphics[width=0.18\textwidth]{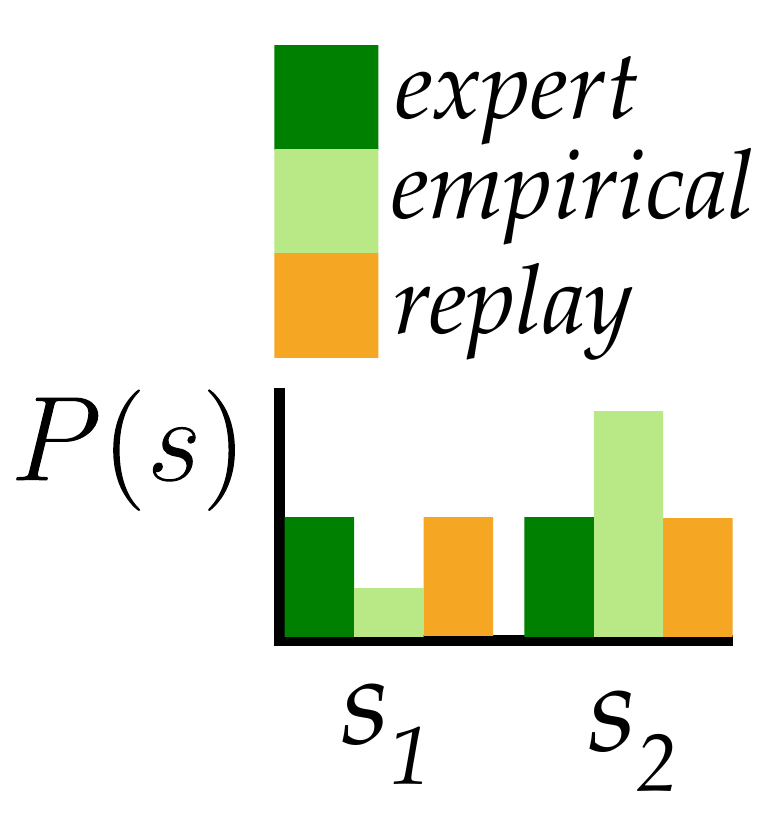}

    \caption{An MDP where the expert always takes the green action that puts them in the uniform distribution over $s_1$ and $s_2$. Because of full expert support, \BC will learn to always take this action at both states. However, if the empirical state distribution is more tilted towards $s_2$, \MM will take the incorrect red action.}
    \label{fig:mm_sucks}
\end{wrapfigure}
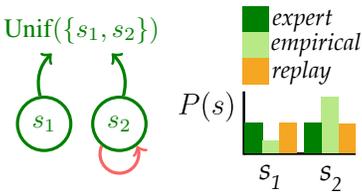

Now, what would happen if we tried to match moments of the expert's state-action visitation distribution for this problem? It is rather unlikely that we see \textit{exactly} equal probabilities for both states in the observed data. If by chance we see $s_2$ more than we see $s_1$, the learner might realize that the only way to match the observed state distribution (a prerequisite for matching the observed state-action distribution) is to occasionally take the red action at $s_2$. In general, this could cause the learner to spend an unnecessary amount of time in $s_2$ which may be undesirable (e.g. if $s_2$ corresponds to the tree-filled area in \figref{fig:ffig} (top)). The core issue we hope to illustrate in this example is that by treating the empirical estimate of the expert's behavior as perfectly accurate, distribution matching can force the learner to take incorrect actions to minimize training error, leading to test-time performance degradation. As we will discuss in \sref{sec:theory}, this can lead to slow statistical rates $\propto H / \sqrt{\Nexp}$.

\textbf{Suboptimality of Behavioral Cloning.} Because it does not account for the covariate shift that results from policy action choices, behavioral cloning can lead to a quadratic compounding of errors and poor test time performance \cite{ross2011reduction}. Consider, for example, the MDP in \figref{fig:bc_sucks}.

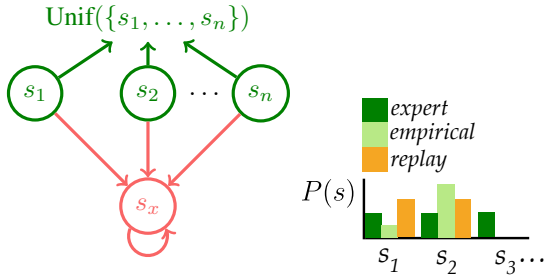
\begin{wrapfigure}{l}{0.6\textwidth}
    \centering
    \begin{tikzpicture}[scale=1, transform shape]
    \node (a) [draw, very thick, circle, color=expert] at (0.0, 0) {$s_1$};
    \node (b) [draw, very thick, circle, color=expert] at (1.5, 0) {$s_2$};
    \node (c) [draw, very thick, circle, color=expert]  at (3, 0) {$s_n$};
    \node (d) [circle]  at (2.25, 0) {$\ldots$};
    \node (e) [draw, very thick, circle, color=error]  at (1.5, -1.5) {$s_x$};
    \node (f) [very thick, color=expert] at (1.5, 1) {$\text{Unif} (\{ s_1 ,\dots, s_n \})$};
    \draw [->, very thick, color=expert] (a) to (f);
    \draw [->, very thick, color=expert] (b) to (f);
    \draw [->, very thick, color=expert] (c) to (f);
    \draw [->, very thick, color=error] (a) to (e);
    \draw [->, very thick, color=error] (b) to (e);
    \draw [->, very thick, color=error] (c) to (e);

    % Self loops
    \node [circle, minimum size=0.5cm](g) at ([{shift=(270:0.4)}]e){};
    \coordinate (h) at (intersection 2 of e and g);
    \coordinate (i) at (intersection 1 of e and g);
    \tikzAngleOfLine(g)(i){\AngleStart}
    \tikzAngleOfLine(g)(h){\AngleEnd}
    \draw[very thick,->, color=error]%
    let \p1 = ($ (g) - (i) $), \n2 = {veclen(\x1,\y1)}
    in
        (g) ++(270:0.5) node{}
        (i) arc (\AngleStart-360:\AngleEnd:\n2);
    \end{tikzpicture}
    \includegraphics[width=0.24\textwidth]{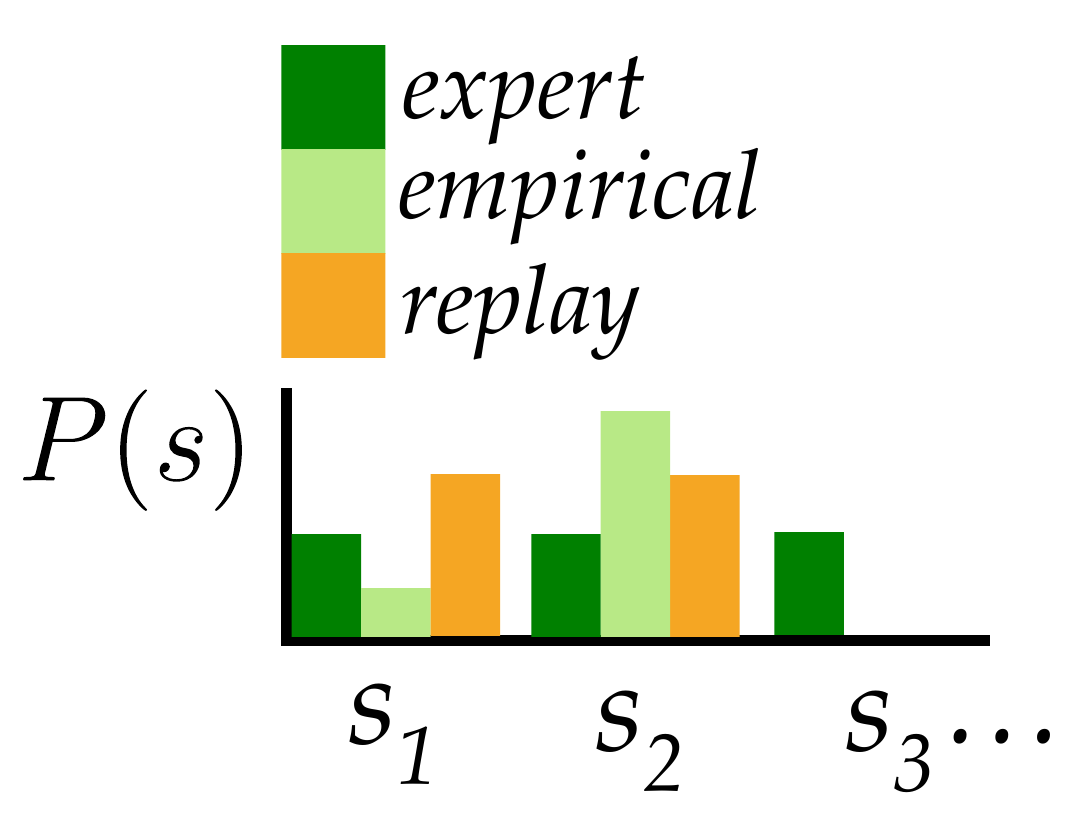}
    \caption{The expert always takes the green action, which places it in a uniform distribution over $s_1, \dots, s_n$. At states where we have demonstrations (e.g. $s_1$, $s_2$), both \BC and \MM will take the same, correct action. However, at states where we have no demonstrations (e.g. $s_3$), \MM will correctly take the green action to get back to states with demonstration support, while \BC might not.}
    \label{fig:bc_sucks}
\end{wrapfigure}

Let us assume that the expert always takes the green action, dropping them in a state in the top row with uniform probability. In a small demonstration set, we might not see expert actions at some states in the top row. At all such states, \BC will have no idea of what to do. In contrast, \MM will take the green action as doing so might send the learner back to a state with positive demonstration support. Thus for this problem, \MM will recover the optimal policy while \BC will not. As we will discuss in \sref{sec:theory}, this leads to errors $\propto H^2 / \Nexp$ in the worst case.

\textbf{Replay Estimation.} The previous two examples show us that there exist simple MDPs for which \BC or \MM will not recover the expert's policy. This begs the question: is it possible to do \textit{better than both worlds} and recover the optimal policy on both problems with a single algorithm? It turns out it is indeed possible to do so, via the technique of \textit{replay estimation}. In its simplest form, replay estimation involves playing a cached expert action whenever possible and re-starting the rollout if one ventures out of the support of the demonstrations. Then, one appends these rollouts to the demonstration set, treating them as additional training data -- while biased, they are consistent with observed expert behavior. Intuitively, repeated simulation has a \textit{smoothing} effect on the training data as doing so marginalizes out the statistical error that comes from the stochasticity of the dynamics. We can see this point more explicitly by considering the above two MDP examples: in \figref{fig:mm_sucks}, repeatedly playing the green action and appending these rollouts to the expert dataset would bring us much closer to a uniform distribution over $s_1$ and $s_2$. Similarly, in \figref{fig:bc_sucks}, replay estimation would bring us toward a uniform distribution over the states $\{ s_1,\cdots,s_n \}$ in the expert demonstrations.

We could then plug in this improved distribution estimate into the \MM procedure \eqref{eq:mm}. Notice how doing so would cause \MM to be highly likely to recover the optimal policy on both MDPs. For example, in \figref{fig:mm_sucks}, replay estimation would make the learner much less likely to play the red action in $s_2$, It turns out this fact is sufficient to establish \textit{statistical optimality} in the tabular setting and with linear function approximation, with an error rate $\propto \min({H^{3/2}} / {\Nexp}, {H} / {\sqrt{\Nexp}})$. In short, replay estimation is a practical technique for reducing some of the finite-sample variance in expert demonstrations that enables \MM to perform optimally in the finite sample regime. We now provide some intuition on how to generalize this approach to beyond the tabular setting.

\textbf{Leaving the Tabular Setting.} The prior work of \citet{rajaraman2020toward} considers the tabular setting; this characteristic makes it easy to answer the question of ``\textit{on what states do we know the expert's action?}" To enable us to answer this question more generally, we introduce the notion of a \textit{membership oracle} $\mathcal{M}: \mathcal{S} \rightarrow \{0, 1\}$. Explicitly, $\mathcal{M}(s) = 1$ for states where we know the expert's action well (e.g. states where we have lots of similar demonstrations) and $\mathcal{M}(s) = 0$ otherwise.

We can then compute expert moments by splitting on the output of the membership oracle:
\begin{align} 
    \mathbb{E}_{\dem}[f(s,a)] = \underbrace{\mathbb{E}_{\dem} \left[f(s, a) \mathbbm{1} (\mathcal{M}(s) = 1) \right]}_{(i)} + \underbrace{\mathbb{E}_{\dem} \left[ f(s, a) \mathbbm{1} (\mathcal{M}(s) = 0) \right]}_{(ii)}
\end{align}
Note that the indicators in $(i)$ and $(ii)$ are complements of each other, rendering the above sum a valid estimate of the expert moment. As we know the expert action well wherever $\mathcal{M}(s) = 1$, simulated rollouts of the \BC policy approximates $(i)$ well; on the other hand we resort to a naive empirical estimate to approximate $(ii)$, as we do not know enough about the expert's action at these states to accurately generate additional demonstrations via \BC rollouts. In general, we relax $\mathcal{M}$ to a \textit{soft membership oracle} \cite{ZADEH1965338}, in order to handle uncertainty in how well we know the expert's action at a given state. We proceed by first analyzing the statistical properties of applying \texttt{MM} to this bipartite estimator before discussing practical constructions of performant membership oracles.

\vspace{-1em}

\section{Theoretical Analysis}
\label{sec:theory}
\begin{wrapfigure}{r}{0.4\textwidth}
    \centering
    \includegraphics[width=0.375\columnwidth]{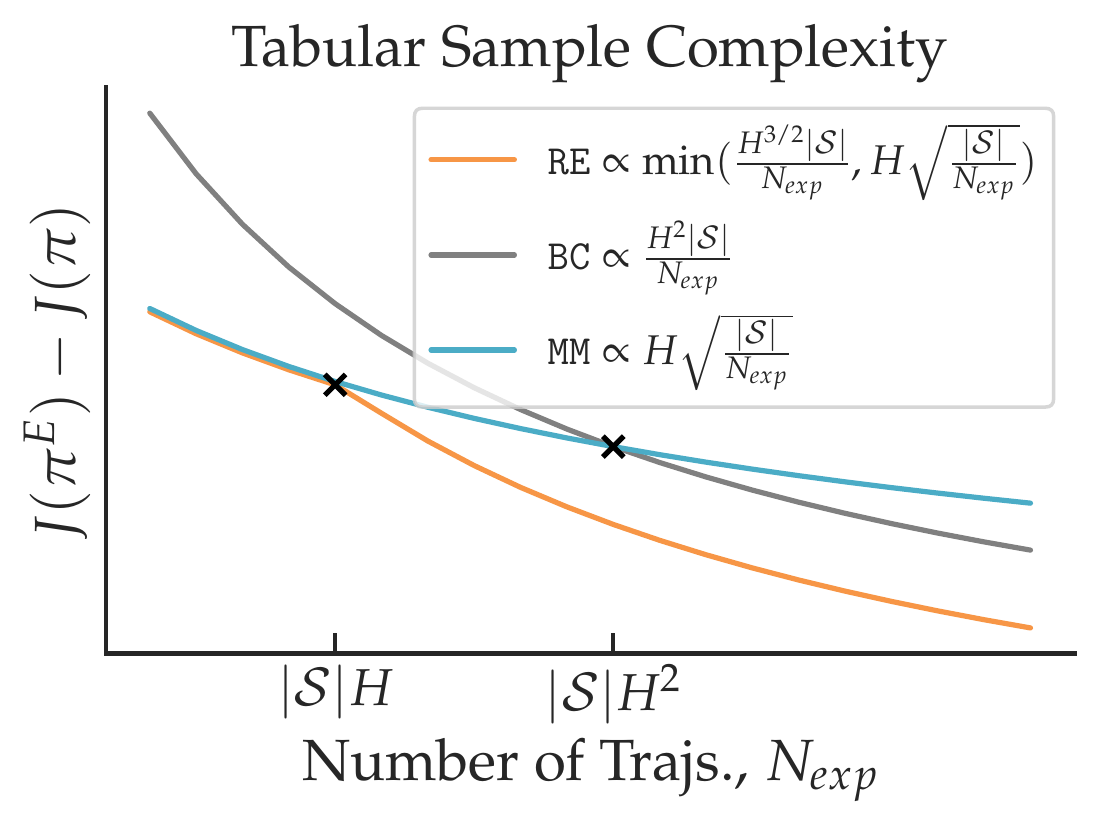}
    \caption{In the tabular setting, \RE inherits the superior low-data performance of moment-matching approaches and is able to perform better than both \MM and \BC with enough data.}
    \label{fig:samp}
\end{wrapfigure}

The proofs of all results in this section are deferred to \cref{app:proofs}. We begin by introducing some notation.

\textbf{Notation.} Let $\Delta(X)$ denote the probability simplex over set $X$ and let $\gtrsim, \lesssim, \asymp$ respectively denote greater than, lesser than and equality up to constants. We study the IL problem in the episodic MDP setting with state space $\mathcal{S}$, action space $\mathcal{A}$ and horizon $H$. %For simplicity, we ignore the discount factor.
We assume that the transition, reward function and policies can be non-stationary. The MDP transition is denoted $P = \{ \rho,P_1,\cdots,P_{H-1} \}$, where $\rho$ is the initial state distribution and $P_t : \mathcal{S} \times \mathcal{A} \to \Delta(\mathcal{S})$, while the reward function is denoted $r = \{ r_1,\cdots,r_H \}$ where $r_t : \mathcal{S} \times \mathcal{A} \to [0,1]$. In the online imitation learning setting, the learner has access to a finite dataset $D$ of $\Nexp$ trajectories (i.e. the sequences of states visited and actions played) generated by rolling out expert policy $\dem$. Importantly, the learner \textit{does not observe rewards during rollouts.} The fundamental goal of the learner is to learn a policy $\widehat{\pi}$ such that the \textit{imitation gap},
$J(\dem) - J(\widehat{\pi})$ is small. Here, $J(\pi)$ denotes the expected value of the policy $\pi$, $\mathbb{E}_{\pi}\left[ \sum_{t=1}^H r_t (s_t,a_t) \right]$. $\mathbb{E}_D [\cdot]$ denotes the empirical expectation computed using trajectories from the dataset $D$.

\textbf{Behavior cloning.} A standard approach for imitation learning is behavioral cloning (\texttt{BC}) \cite{pomerleau1989alvinn}, which trains a classifier from expert states to expert actions. More formally, for the $0$-$1$ loss $\ell$ (or a continuous proxy) \BC minimizes the empirical classification error,
\begin{align} \label{eq:bc}
    \piBC \gets \argmin_{\pi \in \Pi} \mathbb{E}_D \left[ \frac{1}{H} \sum\nolimits_{t=1}^H \ell (\pi_t (s_t) , a_t ) \right].
\end{align}
It is known from \citet{ross2011reduction} and \citet{rajaraman2020toward} that in the tabular setting, the expected imitation gap for \BC is always $\lesssim |\mathcal{S}|H^2/\Nexp$.  The fast $1 / \Nexp$ rate comes from the fact that on states where the learner observes the deterministic expert's actions, \BC simply replays them. Thus, the performance gap is proportional to the mass of unseen states, which decays as $1 / \Nexp$. The $H^2$ dependence comes from the fact that a single mistake can take the learner out of the distribution of the expert, causing it to make mistakes for the rest of the horizon. Importantly, this bound is tight -- i.e. there exists an MDP instance on which \BC incurs this imitation gap (See \Cref{sec:BC-lb}).

\begin{theorem}[Theorem~6.1 of \cite{rajaraman2020toward}] \label{theorem:bclb}
Then, there exists a tabular MDP instance such that the \BC incurs, $\mathbb{E} \left[ J(\pi^E) - J(\piBC) \right] \gtrsim \min \left\{ H, {|\mathcal{S}|H^2} / {\Nexp} \right\}$.
\end{theorem}

\textbf{Moment matching.}  Many standard algorithms in the IL literature fall under the empirical moment matching framework (e.g. GAIL \citet{ho2016generative}, MaxEnt IRL \citet{ziebart2008maximum}); see Table 3 of \citet{swamy2021moments} for more examples. Reward moment matching corresponds to finding a policy which best matches the state-action visitation measure of $\pi^E$, in the sense of minimizing an Integral Probability Metric (IPM) \cite{muller1997integral}.
In the finite sample setting, the \textit{empirical moment matching} learner $\piEMM$ attempts to best match the empirical state-visitation measure. Namely,
\begin{align} \label{eq:emm}
    \piEMM \in \argmin_{\pi \in \Pi} \sup_{ f \in \mathcal{F}} \mathbb{E}_{\pi} \!  \left[\frac{ \sum_{t=1}^H f_t (s_t,a_t)}{H} \right] \! {-} \mathbb{E}_{D} \! \left[\frac{ \sum_{t=1}^H f_t (s_t,a_t)}{H} \right].
\end{align}

First we show an upper bound on the imitation gap incurred by empirical moment matching.

\begin{restatable}{thm}{folklore}
\label{theorem:folklore}
Consider the empirical moment matching learner $\piEMM$ (\cref{eq:emm}), instantiated with an appropriate discriminator class $\mathcal{F}$. %as initialized in \Cref{rem:1}.
The imitation gap satisfies $ \mathbb{E} \left[ J(\pi^E) - J(\piEMM) \right] \lesssim H \sqrt{|\mathcal{S}|/\Nexp}$.
\end{restatable}

The proof of this result can be found in \Cref{sec:MM-ub}. It turns out that this guarantee is essentially tight for empirical moment matching, answering an open question from \citet{rajaraman2020toward}.

\begin{restatable}{thm}{tvlb}
\label{theorem:tvlb}
If $H \ge 4$, there is a tabular IL instance with $2$ states and actions on which with constant probability, the empirical moment matching learner (\cref{eq:emm}) incurs, $J(\pi^E) - J(\piEMM) \gtrsim H/\sqrt{\Nexp}$.
\end{restatable}

The proof of this result is deferred to \cref{sec:MM-lb}.
The proof of this lower bound exploits the fact that the data generation process in the dataset is inherently random. Consider a slight modification of the MDP instance shown in \cref{fig:mm_sucks}, where the reward function is $0$ for $t=1$. For $t \ge 2$, the transition function is absorbing at both states; the reward function equals $1$ at the state $s_1$ for any action and is $0$ everywhere else. Then, the expert state distribution at time $2$ and every time thereon is in uniform across the two states, $\{ 1/2, 1/2\}$. However, in the dataset $D$, the learner sees a noisy realization of this distribution in the dataset of the form $\{ 1/2 - \delta, 1/2 + \delta \}$ for $|\delta| \approx \pm 1/\sqrt{\Nexp}$. Because of this noise, the empirical moment matching learner may be encouraged to \textit{deviate from the expert's observed behavior} and pick the red action at $s_2$ as this results in a better match to the empirical \textit{state} visitation measures at every point in the rest of the episode - a prerequisite to matching the empirical state-action visitation measure. The learner is willing to pick an action different from what the expert played in order to better match the \textit{inherently noisy} empirical state-action visitation distribution.

\begin{remark} \label{rem:bothbad} Theorems~\ref{theorem:bclb} and \ref{theorem:tvlb} are separate lower bound IL instances against the performance of \BC and empirical moment matching. On the uniform mixture of the two MDPs (i.e. deciding the underlying MDP based on the outcome of a fair coin), with constant probability, \textit{both} $J(\pi^E) - J(\piBC) \gtrsim |\mathcal{S}|H^2/\Nexp$ and $J(\pi^E) - J(\piEMM) \gtrsim H/\sqrt{\Nexp}$. On this mixture instance, training both \BC and empirical moment matching and choosing the better of the two is also statistically suboptimal. \end{remark}

\vspace{-0.5em}
\subsection{Replay Estimation} A natural question at this point is whether there is an algorithm that is \textit{better than both worlds}, i.e. algorithm which can outperform the worst case imitation gap of empirical moment matching and \BC. The answer to this question in the tabular setting was recently provided by \citet{rajaraman2020toward} who propose \mmd, achieving better performance than both \BC and \MM. This improvement is possible because \BC does not use any dynamics information and \MM does not leverage the knowledge of where expert actions are known. However, it is unclear how to extend this approach beyond the tabular setting as the algorithm relies on a large measure of states being visited in the demonstrations.

\begin{algorithm}[htb]
	%\caption{\textsc{Chen+Price}}
	\caption{Replay Estimation (\RE)}
	\label{alg:re}
% 	\begin{algorithmic}[1]
	    \KwIn{Expert demonstrations $D$, policy class $\Pi$, moment class $\mathcal{F} = \bigoplus_{t=1}^H \mathcal{F}_t$, simulator \texttt{SIM}, \texttt{ALG} which returns a membership oracle given a dataset\;}
		
		Partition the dataset $D$ into $D_1$ and $D_2$\;
		 
		Using \texttt{ALG}, train the membership oracle $\mathcal{M}$ on $D_1$\;
		
		Train $\piBC$ using behavior cloning on $D_1$\;
        
        Roll out $\piBC$ in \texttt{SIM} $\Nrep$ times to construct a new dataset, $\Drep$\;
        
        Define prefix weights $\mathcal{P}(s_{1 \dots t-1}) = \prod\nolimits_{t'=1}^{t-1} \mathcal{M} (s_{t'}, t')$\;

        Define,
        \begin{align}
    \widehat{E} (f) &= \mathbb{E}_{\Drep} \left[ \frac{1}{H} \sum\nolimits_{t=1}^H f_t(s_t, a_t) \left( \mathcal{P}(s_{1 \dots t}) \right) \right] + \mathbb{E}_{D_2} \left[ \frac{1}{H} \sum\nolimits_{t=1}^H f_t(s_t,a_t) \left( 1 - \mathcal{P}(s_{1 \dots t}) \right) \right]. \nonumber
    \end{align}
	
	\KwOut{$\piRE$, a solution to the moment-matching problem:
\begin{align} \label{eq:pire}
    \argmin_{\pi \in \Pi} \sup_{f \in \mathcal{F}} \mathbb{E}_\pi \left[ \frac{1}{H} \sum\nolimits_{t=1}^H f_t (s_t,a_t) \right] - \widehat{E} (f)
\end{align}}
\end{algorithm}

\vspace{-0.75em}
To handle this challenge, we introduce the notion of a \textit{soft membership oracle} \cite{ZADEH1965338}, $\mathcal{M}: \mathcal{S} \times [H] \to [0, 1]$ which captures the learner's inherent uncertainty in the expert's actions at a state at each point in an episode. The soft membership oracle assigns high weight to a state if \BC is likely to closely agree with the expert policy and gives a lower weight to states where \BC is likely to be inaccurate. By this definition, if the membership oracle is consistently large at all the states visited in a trajectory, we can be confident that \textit{a trajectory generated by \BC is as though it was a rollout from the expert policy}. Formally, for any function $g$ and time $t = 1,\cdots,H$, we have the decomposition,
\begin{align} 
    \mathbb{E}_{\pi^E} [g(s_t,a_t)] &= \underbrace{\mathbb{E}_{\pi^E} \left[g(s_t, a_t) \mathcal{P} (s_{1 \cdots t}) \right]}_{(i)} + \underbrace{\mathbb{E}_{\pi^E} \left[g(s_t, a_t) \left(1 - \mathcal{P} (s_{1 \cdots t}) \right) \right]}_{(ii)} \label{eq:t}
\end{align}
where $\mathcal{P} (s_{1 \cdots t})$ is defined as the \textit{prefix weight} $\prod_{t'=1}^{t} \mathcal{M} (s_{t'}, t')$. We need to use prefix weights instead of the single-sample weights sketched in the previous section to account for the probability of \BC getting to the current state in the same manner the expert would have.
Because of the high accuracy of \BC on segments with high prefix weights, in \cref{eq:t}, $(i)$ can be approximated by replacing the expectation over $\pi^E$ by that over $\piBC$, i.e. replay estimation. On the other hand, since the prefix weight is low on the remaining trajectories in $(ii)$, we know that \BC is inaccurate, so we resort to using a simple empirical estimate to estimate this term.

While we leave the particular choice of the soft membership oracle flexible, intuitively, states at which \BC closely agrees with the expert policy should be given high weight while where those where \BC is inaccurate should be weighted lower. In \Cref{sec:praxis}, we discuss several practical approaches to designing such a soft membership oracle. We first prove a generic policy performance guarantee for the outputs of our algorithm as a function of the choice of $\mathcal{M}$.

\begin{theorem} \label{thm:main}
Consider the policy $\piRE$ returned by \Cref{alg:re}. Assume that $\pi^E \in \Pi$ and the ground truth reward function $r_t \in \mathcal{F}_t$, which is assumed to be symmetric ($f_t \in \mathcal{F}_t \iff -f_t \in \mathcal{F}_t$) and bounded (For all $f_t \in \mathcal{F}_t$, $\| f_t \|_\infty \le 1$). Choose $|D_1|, |D_2| = \Theta (\Nexp)$ and suppose $\Nrep \to \infty$. With probability $\ge 1 - 3 \delta$,
\begin{align}
    J (\dem) - J (\piRE) \lesssim & \ \mathcal{L}_1 + \mathcal{L}_2 
    + \frac{\log \left( \Fmax H/\delta \right)}{\Nexp}
\end{align}
where $\Fmax \triangleq \max_{t \in [H]} |\mathcal{F}_t|$, and,
\begin{align}
    \mathcal{L}_1 \triangleq H^2 \ \mathbb{E}_{\dem} \! \left[ \frac{\sum_{t=1}^H \mathcal{M} (s_t,t) \tv \left( \dem_t (\cdot|s_t) , \piBC_t (\cdot|s_t) \right)}{H} \right],
\end{align}
\begin{align}
    \mathcal{L}_2 \triangleq H^{3/2} \sqrt{ \frac{\log \left( \Fmax H/\delta \right)}{\Nexp} \frac{\sum_{t=1}^H \mathbb{E}_{\dem} \left[ 1 - \mathcal{M} (s_t, t) \right]}{H}}. \nonumber
\end{align}
\end{theorem}
We discuss a proof of this result in \Cref{sec:RE-ub} and include bounds when $\Nrep$ is finite.

\begin{remark} \label{rem:inffamily}
Note that \Cref{thm:main} can be extended to infinite function families using the standard technique of replacing $|\mathcal{F}_t|$ by, $\mathcal{N} (\mathcal{F}_t, \xi, \| \cdot \|_\infty)$ the $\xi$ log-covering number (metric entropy) of $\mathcal{F}_t$ in the $L_\infty$ norm, for $\xi = \frac{1}{\Nexp H}$. Likewise, we may appropriately replace $\Fmax$ by $\Nmax \triangleq \max_{t \in [H]} \mathcal{N} (\mathcal{F}_t, \frac{1}{\Nexp H}, \| \cdot \|_\infty)$. For ease of exposition here, we stick to the case where $\mathcal{F}_t$ is finite.
\end{remark}

The term $\mathcal{L}_1$ measures how accurate \BC is on states from expert trajectories where $\mathcal{M}(s_t, t)$ is large. Intuitively, if we set $\mathcal{M}(s_t, t) = 1$ on states where \BC is accurate and $\mathcal{M}(s_t, t) = 0$ elsewhere, we would expect this term to be small. $\mathcal{L}_2$ can be thought of a measure of \BC's coverage: it tells us how much of the expert's visitation distribution we believe \BC to be inaccurate on. If \BC has good coverage (i.e. $1 - \mathcal{M}(s_t, t)$ is small on expert trajectories), we expect this term to be small.

Prima facie, one might think that because $\mathcal{L}_1$ resembles the imitation gap of \BC and $\mathcal{L}_2$ resembles that of \MM, \RE can only perform as well as the best of \BC ($\propto H^2 / \Nexp$) and \MM ($\propto H / \sqrt{\Nexp}$) on a given instance. However, with a careful choice of $\mathcal{M}$, one can achieve ``better than both worlds'' statistical rates. In particular, since \RE is a generalization of \mmd of \cite{rajaraman2020toward}, in the tabular setting, an appropriately initialized version of \RE achieves the optimal imitation gap of $\min \left\{ \frac{|\mathcal{S}| H^{3/2}}{\Nexp} , H \sqrt{\frac{|\mathcal{S}|}{\Nexp}} \right\} \log \left( \frac{|\mathcal{S}| H}{\delta} \right)$ and strictly improves over both \BC and \MM.

We now show how to instantiate the membership oracle for parametric function approximation and provide a statistical guarantee under a particular margin assumption.

\section{Parametric Function Approximation: Reduction to Offline Classification}

While \BC can be thought of as a reduction of IL to offline classification, the algorithm does not take into account the knowledge of the transition of the MDP. This is reflected in the quadratic dependency in the horizon, also known as error compounding \cite{ross2011reduction}. This problem is certainly encountered in practice \cite{pomerleau1989alvinn}, and \cite{rajaraman2020toward, swamy2021moments} show that this issue is a problem for \textit{any} offline algorithm (i.e. which doesn't interact with the MDP, or know its transition). This begs the question how can we enable algorithms such as \BC which reduce IL to offline classification to benefit from interaction with the environment.

In this section, under the known transition model, we study a novel reduction of IL with parametric function approximation to \textit{parameter estimation in offline classification}, which we define formally. We provide a provable guarantee, assuming the learner has access to a classification oracle and the underlying function class admits a Lipschitz parameterization.

\begin{definition}[IL with function-approximation] \label{A1-gfa}
In this setting, for each $t \in [H]$, there is a parameter class $\Theta_t \subseteq \mathbb{B}_2^d$, the unit $L_2$ ball in $d$ dimensions, and an associated function class $\{ f_{\theta_t} : \theta_t \in \Theta_t \}$. For each $t \in [H]$ there exists an unknown $\theta^E_t \in \Theta_t$ such that $\forall s \in \mathcal{S}$,
\begin{align}
    \pi^E_t (s) = \argmax_{a \in \mathcal{A}} f_{\theta^E_t} (s,a).
\end{align}
\end{definition}

\begin{definition}[Lipschitz parameterization]
A function class $\mathcal{G} = \{ g_{\theta} : \theta \in \Theta \}$ where $g_{\theta} (\cdot) : \mathcal{X} \to \mathbb{R}$ is said to satisfy $L$-Lipschitz parameterization if, $\| g_{\theta} (\cdot) - g_{\theta'} (\cdot) \|_\infty \le L \| \theta - \theta' \|_2$. In other words, for each $x \in \mathcal{X}$, $g_{\theta} (x)$ is an $L$-Lipschitz function in $\theta$, in the $L_2$ norm.
\end{definition}

The Lipschitz parameterization condition can be interpreted as saying that a small change to the underlying classifier does not all of a sudden change the label of a large mass of points. In particular, points which are classified with a large enough ``margin'' continue to stay in the same class even if the underlying classifier/parameter is perturbed.

\begin{assumption} \label{A:lip:original}
For each $t$, the class $\{ f_{\theta_t} : \theta_t \in \Theta_t \}$ is $L$-Lipschitz in its parameterization, $\theta_t \in \Theta_t$.
\end{assumption}

Recall that \BC is a reduction from IL to offline classification, and assumes that the learner has access to an oracle which, given samples from an underlying ground truth classifier, returns a classifier with small generalization error. In particular, treating the expert dataset at each time $t$ as a collection of $\Nexp$ samples from a classifier $\pi_t (\cdot|\cdot)$ mapping states to actions, the \BC learner uses the oracle to train $H$ classifers on these datasets. \cite{ross2011reduction} show that the imitation gap of the policy induced by these classifiers is bounded by the average generalization error of the $H$ learned classifiers, up to a factor of $H^2$.

In order to extend \RE (which uses the transition of the MDP) to deal with parametric function approximation, and show error guarantees which surpass that achieved by \BC, we assume that the learner has access to a slightly stronger offline classification oracle, which, given access to a dataset of classification examples, \textit{approximately returns the underlying ground truth parameter}. More formally,

\begin{assumption}[Offline classification oracle] \label{A:classorc}
We assume that the learner has access to a multi-class classification oracle, which given $n$ examples of the form, $(s^i, a^i)$ where $s^i \overset{\text{i.i.d.}}{\sim} \mathcal{D}$ and $a^i = \argmax_{a \in \mathcal{A}} f_{\theta^*} (s^i,a)$, returns a $\hat{\theta} \in \Theta$ such that, with probability $\ge 1 - \delta$, $\| \hat{\theta} - \theta^* \|_2 \le \mathcal{E}_{\Theta,n,\delta}$.
\end{assumption}
We assume that this classification oracle is used by \RE to train the \BC policy in Line 3 of \Cref{alg:re}.

A careful reader might note that \Cref{A:classorc} asks for a slightly stronger requirement than just finding a classifier with small generalization error (which need not be close to the ground truth $\theta^*$). The generalization error, i.e. $\mathbb{Pr}_{s \sim \mathcal{D}} \left[ \argmax_{a \in \mathcal{A}} f_{\theta^*} (s,a) \ne \argmax_{a \in \mathcal{A}} f_{\hat{\theta}} (s,a) \right]$ in the notation of \Cref{A:classorc}, was previously studied in \cite{DBLP:journals/corr/DanielySBS13} for multi-class classification. The authors show that up to log-factors in the number of classes (i.e. the number of actions), the \textit{Natarajan dimension} characterizes the generalization error of the best learner, which scales as $\Theta(1/n)$ given $n$ classification examples. For example, in linear classification, which we study later in \Cref{sec:RE-linear-ub}, under some mild regularity assumptions on the input distribution $\mathcal{D}$, we show that the generalization guarantee carries over to approximately learning the parameter $\theta$. In particular, we show that under said regularity conditions, for linear classification (i.e. $f_\theta = \langle \theta, \cdot \rangle$), $\mathcal{E}_{\Theta,n,\delta} \lesssim \frac{d + \log(1/\delta)}{n}$.

Recall that we assume that the learner trains a sequence of $H$ classifiers on the state-actions observed in the expert dataset at each time $t$, using the offline classification oracle in \Cref{A:classorc}. We denote the resulting offline classifiers as $\thetaBC = (\thetaBC_1,\cdots,\thetaBC_H)$.

We are now ready to define the membership oracle under which we study \RE below,
\begin{align} \label{eq:memdef}
    &\mathcal{M} (s,t) {=} \begin{cases}
    +1 \quad \!\! & \text{if } \exists a \in \mathcal{A} \text{ such that, } \forall a' \in \mathcal{A},\ f_{\thetaBC_t} (s,a) {-} f_{\thetaBC_t} (s,a') \ge 2 L \eNtH \\
    0 &\text{otherwise.}
    \end{cases}
\end{align}
The intuitive interpretation of $\mathcal{M}$ is that, state which are classified by \BC as some action with a significant margin are assigned as $+1$, and the remaining states are assigned as $0$ by the membership oracle.

Finally, we impose an assumption on the IL instances we study. We assume that the classification problems solved by \BC at each $t \in [H]$ satisfy a margin condition.

\begin{assumption}[Weak margin condition] \label{A:wm}
For $t \in [H]$ and $\theta \in \Theta_t$, define $a^\theta_s = \argmax_{a \in \mathcal{A}} f_{\theta} (s,a)$ as the classifier output on the state $s$. The weak margin condition with parameter $\mu > 0$ assumes that for each $t$, there is no classifier $\theta \in \Theta_t$ such that for a large mass of states, $f_{\theta} (s_t,a^\theta_{s_t}) - \max_{a \ne a^\theta_{s_t} } f_{\theta} (s_t,a)$, i.e. the ``margin'' from the nearest classification boundary, is small. Formally, the weak-margin condition with parameter $\mu$ states that for any $\eta \le 1/\mu$,
\begin{align} \label{eq:wmc}
    \forall \theta \in \Theta_t,\ \ \mathrm{Pr}_{\dem} \left( f_{\theta} (s_t, a^\theta_{s_t}) - \max_{a \ne a^\theta_{s_t} } f_{\theta} (s_t,a) \ge \eta \right) \ge e^{- \mu \eta}.
\end{align}

\end{assumption}

The weak margin condition only assumes that there is at least an exponentially small (in $\eta$) mass of states with margin at least $\eta$. A smaller $\mu$ indicates a larger mass away from any decision boundary.

\begin{remark}
It suffices to assume that for each $t$, \cref{eq:wmc} is only true for the singular choice $\theta = \thetaBC_t \in \Theta_t$, for our main guarantee (\Cref{thm:gfa}) to hold.
\end{remark}

Under the previously discussed assumptions, we provide a strong guarantee for \RE which uses the classification oracle in \Cref{A:classorc} to define \BC, and the membership oracle as defined in \cref{eq:memdef}.

\begin{theorem} \label{thm:gfa}
For IL with parametric function approximation, under Assumptions~\ref{A:lip:original} to \ref{A:wm}, appropriately instatiating \RE ensures that with probability $\ge 1 - 4\delta$,
\begin{align} \label{eq:gfa}
    J(\dem) - J(\piRE) \lesssim H^{3/2} \sqrt{ \frac{\mu  L \log \left( F_{\max} H/\delta \right)}{\Nexp} \frac{\sum_{t=1}^H \eNtH}{H}} + \frac{\log \left( F_{\max} H /\delta \right)}{\Nexp}.
\end{align}
where $\Fmax = \max_{t \in [H]} |\mathcal{F}_t|$ is as defined in \Cref{thm:main}.

Note that we impose the same assumptions on the policy and discriminator classes employed by \RE as in \Cref{thm:main}. Namely, $(i)$ $\pi^E \in \Pi$, and for each $t\in[H]$, $(ii)$ the ground truth reward function $r_t \in \mathcal{F}_t$, $(iii)$ $\mathcal{F}_t$ is symmetric, i.e. $f_t \in \mathcal{F}_t \iff -f_t \in \mathcal{F}_t$, and $(iv)$ $\mathcal{F}_t$ is $1$-bounded, i.e. for all $f_t \in \mathcal{F}_t$, $\| f_t \|_\infty \le 1$.
\end{theorem}

As discussed in \Cref{rem:inffamily}, this result can be extended to infinite discriminator families by replacing $|\mathcal{F}_t|$ by the appropriate log-covering number of $\mathcal{F}_t$ in $L_\infty$ norm.

The intuition behind the result is as follows. By \Cref{A:classorc}, the learner is able to approximately learn $\thetaBC_t \approx \theta^*_t$ at each time $t$. Since the discriminator functions are Lipschitz (\Cref{A:lip:original}) and there are not too many states classified with small margin (\Cref{A:wm}), this means that the states classified with large margin by $\thetaBC_t$ are correctly classified by $\theta^E_t$, while the states which are close to a decision boundary (induced by \BC) may be misclassified by \BC. Therefore, we may set the membership oracle as $+1$ on states classified by \BC with a large margin, and $0$ on states classified with small margin. In particular, the membership oracle considered in \cref{eq:memdef} ensures that on states at which $\mathcal{M} (s,t) > 0$, $\dem_t (\cdot|s) = \piBC_t (\cdot|s)$. Likewise, the states at which $\mathcal{M} (s,t) = 0$ correspond to the states which are classified by \BC with a small margin (i.e. are close to a decision boundary), the probability mass of which is bounded by the weak margin condition, \Cref{A:wm}. All in all, in the language of \Cref{thm:main}, these results ensure that $\mathcal{L}_1 = 0$ and $\mathcal{L}_2$ is significantly smaller than $H^{3/2} \sqrt{\log (\Fmax H /\delta) / \Nexp}$ which essentially results in the proof of \Cref{thm:gfa}.

In order to interpret \Cref{thm:gfa}, we draw the connection back with \BC and \MM. As discussed earlier, \cite{ross2011reduction} prove the best known general statistical guarantee for \BC in terms of the generalization error of the underlying classifiers. In particular, with probability $\ge 1 - \delta$,
\begin{align}
    J(\dem) - J(\piBC) \le \texttt{Gap} (\piBC) \triangleq H^2 \frac{\sum_{t=1}^H \eNtH^{\text{class}}}{H}
\end{align}
Here, $\ent^{\text{class}}$ denotes the best achievable $0$-$1$ generalization error for multi-class classification: in the notation of \Cref{A:classorc}, there exists a learner $\hat{\theta}$, such that on any classification instance $\mathbb{E}_{s \sim \mathcal{D}} [ \max_{a \in \mathcal{A}} f_{\hat{\theta}}(s,a) \ne \max_{a \in \mathcal{A}} f_{\theta^*} (s,a) ] \le \ent^{\text{class}}$ with probability $\ge 1 - \delta$. On the other hand, the best general statistical guarantee for \MM is,
\begin{align}
&J(\dem) - J(\piEMM) \le \texttt{Gap} (\piEMM) \triangleq H \sqrt{\frac{\log \left( F_{\max} H/\delta \right)}{\Nexp}}
\end{align}
(note that we can extend to infinite classes using the covering number argument in \Cref{rem:inffamily}).

Now, whenever the statistical error for parameter estimation matches with the best statistical error for generalization in offline classification, namely, $\ent \asymp \ent^{\text{class}}$ up to problem dependent constants, the guarantee in \Cref{thm:gfa} can be reinterpreted as,
\begin{align} \label{eq:REinterp}
    J(\dem) - J(\piRE) \lesssim \texttt{Gap} (\piEMM) \sqrt{\frac{\texttt{Gap} (\piBC)}{H}}.
\end{align}
The interpretation of this result is that, whenever \BC admits a non-trivial gap on the imitation gap, namely $\texttt{Gap} (\piBC) \ll H$, the performance gap in \cref{eq:REinterp} is $\ll \texttt{Gap} (\piEMM)$. Furthermore, from \cite{DBLP:journals/corr/DanielySBS13}, for multi-class classification, $\ent^{\text{class}} \lesssim \frac{\left( \mathfrak{n}_{\Theta} + \log (1/\delta) \right) \log (n) \log |\mathcal{A}|}{n}$ where $\mathfrak{n}_{\Theta}$ denotes the Natarajan dimension of the function class $\{ f_{\theta} : \theta \in \Theta \}$ and $\mathcal{A}$ denotes the set of labels (which here, are the set of actions). Therefore, from \cref{eq:REinterp}, we get the guarantee,
\begin{align} \label{eq:RE-finalub}
    J(\dem) - J(\piRE) \le \widetilde{O} \left( \frac{H^{3/2}}{\Nexp} \left( \log (\Fmax) \times \frac{\sum_{t=1}^H \mathfrak{n}_{\Theta_t}}{H} \right)^{1/2} \right),
\end{align}
whenever the underlying classification problem allows $\ent^{\text{class}} \asymp \ent$ up to problem dependent constants. Note that the polylogarthmic factors in \cref{eq:RE-finalub} are in $|\mathcal{A}|$, $\Nexp$ and $1/\delta$. Under these conditions, we essentially recover a performance guarantee for \RE which scales as $H^{3/2}/\Nexp$. This improves on the quadratic $H$ dependence incurred by \BC, and is optimal in the worst case, even in the tabular setting with just $3$ states, as shown in \cite{rajaraman2021provably}. This guarantee also suggests a natural measure of complexity for IL - the average Natarajan dimension, $\frac{\sum_{t=1}^H \mathfrak{n}_{\Theta_t}}{H}$ multiplied by the maximum log-covering number of the discriminator (or reward) class, $\log (\Fmax)$.

From this discussion, an important problem stands out:
    \textit{For what kind of classification problems is $\ent \asymp \ent^{\text{class}}$?}
While we do not answer this question in its full generality, focusing on the special case of linear classification and provide a set of sufficient conditions under which $\ent \asymp \ent^{\text{class}}$ up to problem dependent constants. In conjunction with \cref{eq:RE-finalub}, this results in a novel guarantee for IL with linear function approximation, under significantly weaker conditions compared to prior work. We defer the discussion of these results to \Cref{sec:RE-linear-ub}.

Our guarantees for \RE in the presence of parametric function approximation and linear function approximation give us reasons to expect \RE to improve the performance of \MM. We now turn our attention to confirming this holds true in practice.

\vspace{-0.25em}
\section{Practical Algorithm}
\vspace{-0.25em}
\label{sec:praxis}
When considering the implementation of \RE (Alg. \ref{alg:re}) in practice, two main questions arise:
\begin{enumerate}
    \item How does one construct a membership oracle in practice, especially when action spaces may be continuous? %In particular, continuous action spaces make it difficult, both computationally and statistically, to identify states which are classified with margin as the membership oracle in \cref{eq:memdef} suggests,
    \item How does one design good solvers for the moment matching problem in \cref{eq:pire}?
\end{enumerate}
We now provide potential answers to both of these questions.

\textbf{Membership Oracle. } Recall from the interpretation of \Cref{thm:main} that the membership oracle $\mathcal{M}$ intuitively should capture how uncertain \BC is about the expert's action at a state. For continuous action spaces, ideally one would assign the membership oracle at that state based on an appropriate notion of distance between the action played by \BC and that played by the expert. For example, for a sigmoid function $\sigma$ and constants $\mu, \beta$,
\begin{equation} \label{eq:memexp}
    \mathcal{M}_{\texttt{EXP}}(s, t) = \sigma \left (\frac{\mu -\norm{\piBC(s) - \dem(s)}_2}{\beta}  \right),
\end{equation}
However, in the non-interactive setting where the demonstrator cannot be queried at states, we can only approximate this quantity. The first approximation we propose is inspired by Random Network Distillation (RND) \citep{burda2018exploration}, used by \citet{wang2019random} to estimate the support of the expert policy. We instead propose to use RND as a measure of the epistemic uncertainty of \BC about expert actions. That is,
\begin{equation} \label{eq:memrnd}
    \mathcal{M}_{\texttt{RND}}(s, t) = \sigma \left (\frac{\mu -\norm{\piBC(s) - \widehat{\piBC}(s)}_2}{\beta}  \right),
\end{equation}
where $\widehat{\piBC}$ is a network trained to imitate the output of the classifier $\piBC$ on the states observed in the expert dataset. To train $\widehat{\piBC}$, we evaluate $\piBC$ on states observed in the expert dataset, and plug this new dataset into the standard \BC pipeline.

We can also utilize other uncertainty measures like the disagreement of an ensemble to measure epistemic uncertainty, which has previously shown success on various simulated sequential decision making tasks \citep{pathak2019selfsupervised}. Past work by \citet{brantley2019disagreement} proposes to regularize the standard \BC (classification) error, by the variance of an ensemble of independently trained \BC learners at the states visited by the learner's policy. This encourages the learner to mimic \BC on the states where the action predicted by all the policies in the ensemble are similar, and avoid states where they are different (i.e. the variance at these states is high).
In contrast, we define,
\begin{equation} \label{eq:memvar}
    \mathcal{M}_{\texttt{VAR}}(s, t) = \sigma \left (\frac{\mu - \text{Var}(\pi^{\texttt{BC} (1)} (s), \dots, \pi^{\texttt{BC} (k)} (s))}{\beta}  \right),
\end{equation}
where $\{ \pi^{\texttt{BC} (1)}, \cdots, \pi^{\texttt{BC} (k)} \}$ are \BC policies trained with different initializations, which produces sufficient diversity when using deep networks as function approximators. 

Lastly, we can also use the maximum difference across the ensemble as a measure of uncertainty:
\begin{equation} \label{eq:memmax}
    \mathcal{M}_{\texttt{MAX}}(s, t) = \sigma \left (\frac{\mu - \max_{i, j \in [k]}\norm{\piBC_i(s) - \piBC_j (s)}_2}{\beta}  \right),
\end{equation}
as suggested by the work of \citet{kidambi2021morel}. We compare these four choices below. For computing prefix weights, we use the average of distances up till the current timestep instead of the sum of distances that would follow directly from the idealized algorithm presented in \sref{sec:theory}. This modification serves to improve the numerical stability of our method.

\begin{figure*}[t]
\centering
\includegraphics[width=0.28\textwidth]{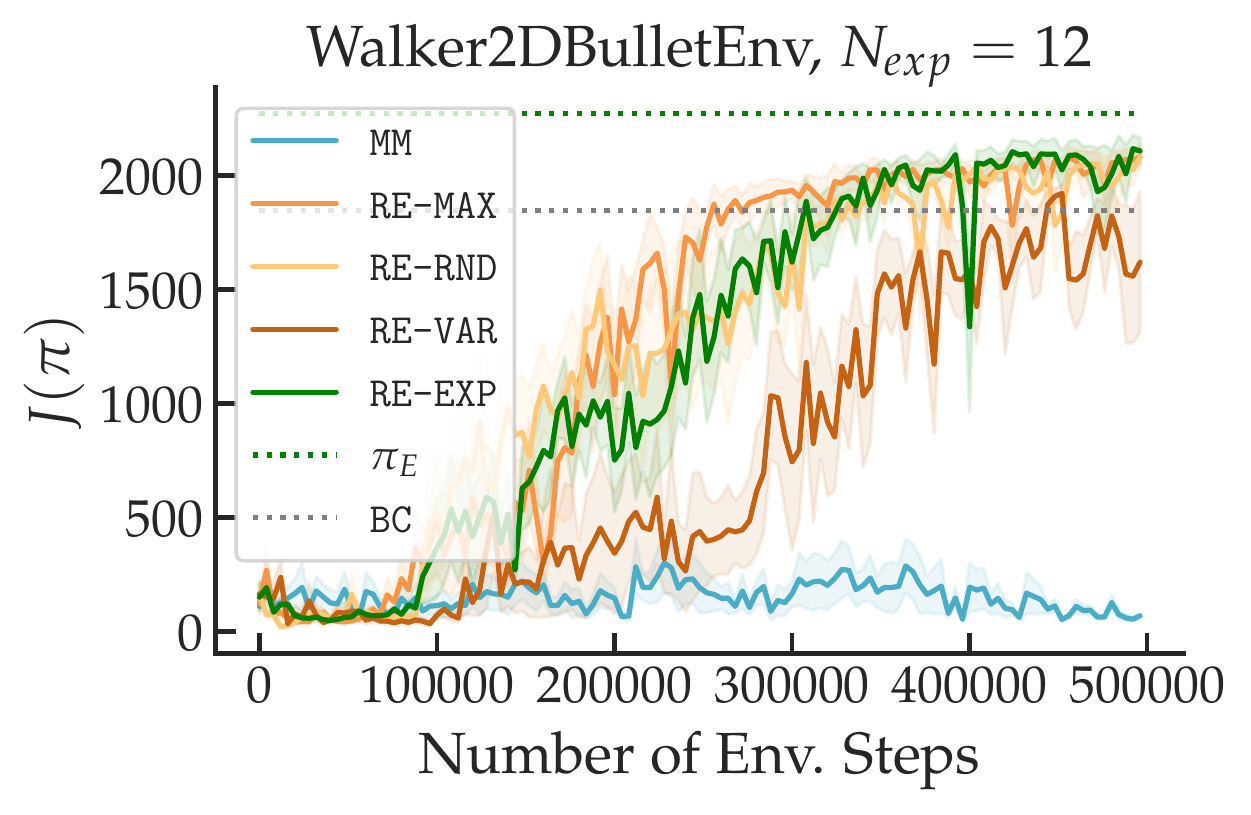}
\includegraphics[width=0.28\textwidth]{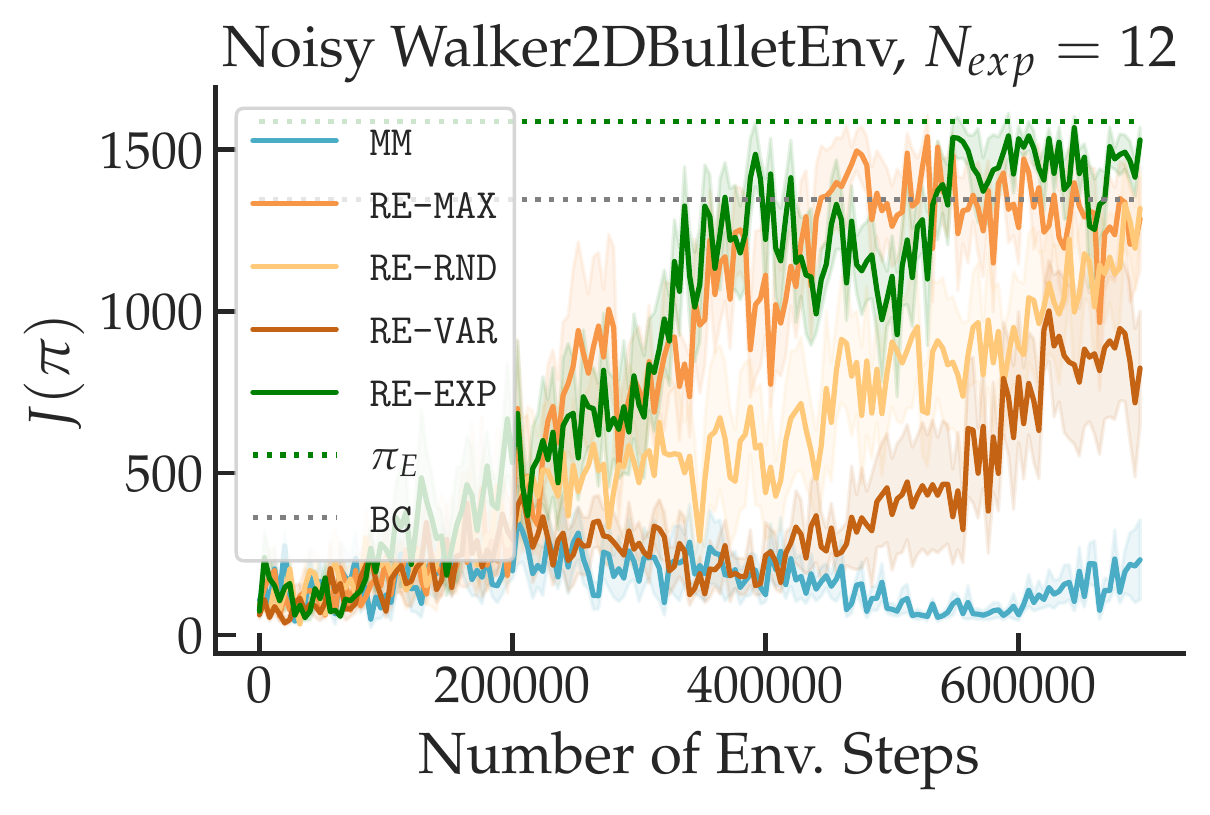}
\includegraphics[width=0.21\textwidth]{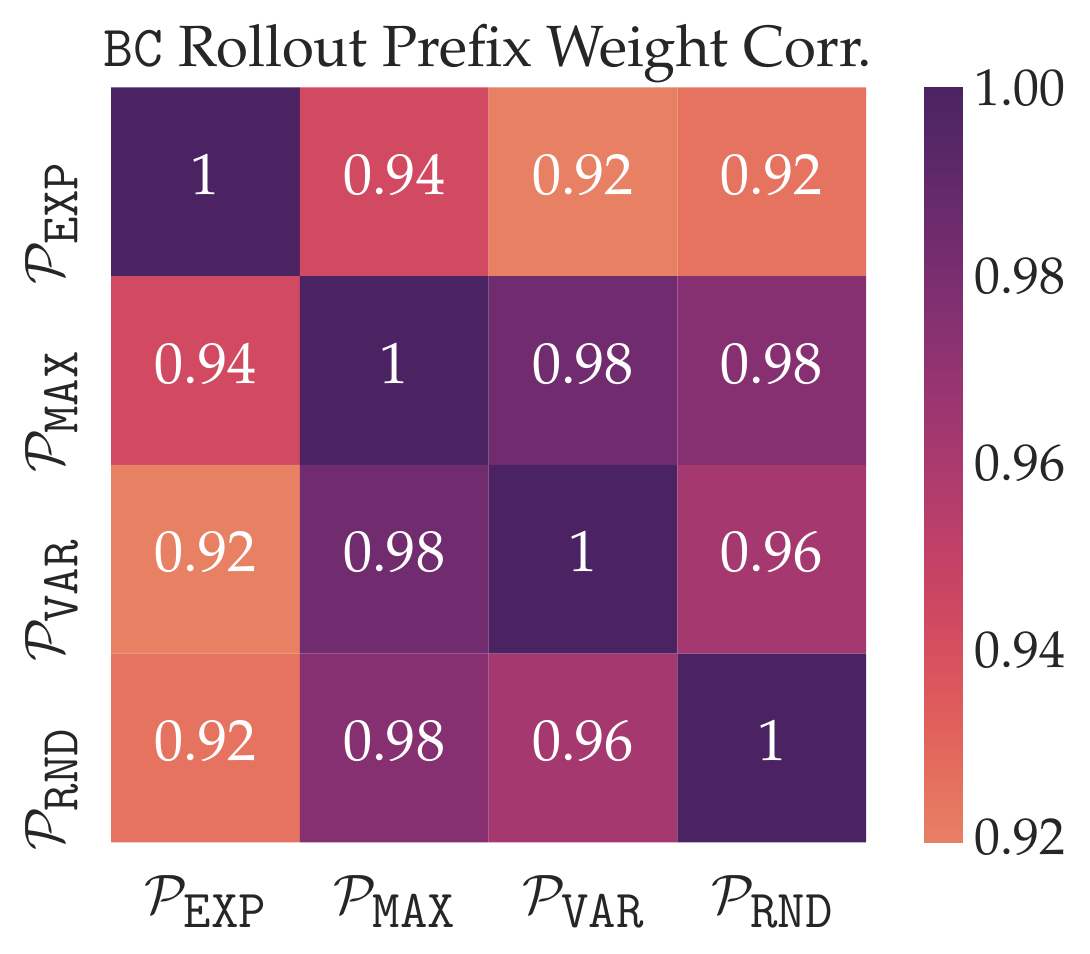}
\caption{\textbf{Left:} All variants of \RE are able to nearly match expert performance while \MM struggles to make any progress. \textbf{Center:} We add i.i.d. noise to the environment to make the control problem more challenging. \RE is still able to match expert performance, unlike \MM. \textbf{Right:} We compute correlations between the idealized prefix weights of $\mathcal{M}_{\texttt{EXP}}$ and the other oracles and see $\mathcal{M}_{\texttt{MAX}}$ correlate most. \label{fig:res}}
\end{figure*}
\begin{figure*}[t]
\centering
\includegraphics[width=0.245\textwidth]{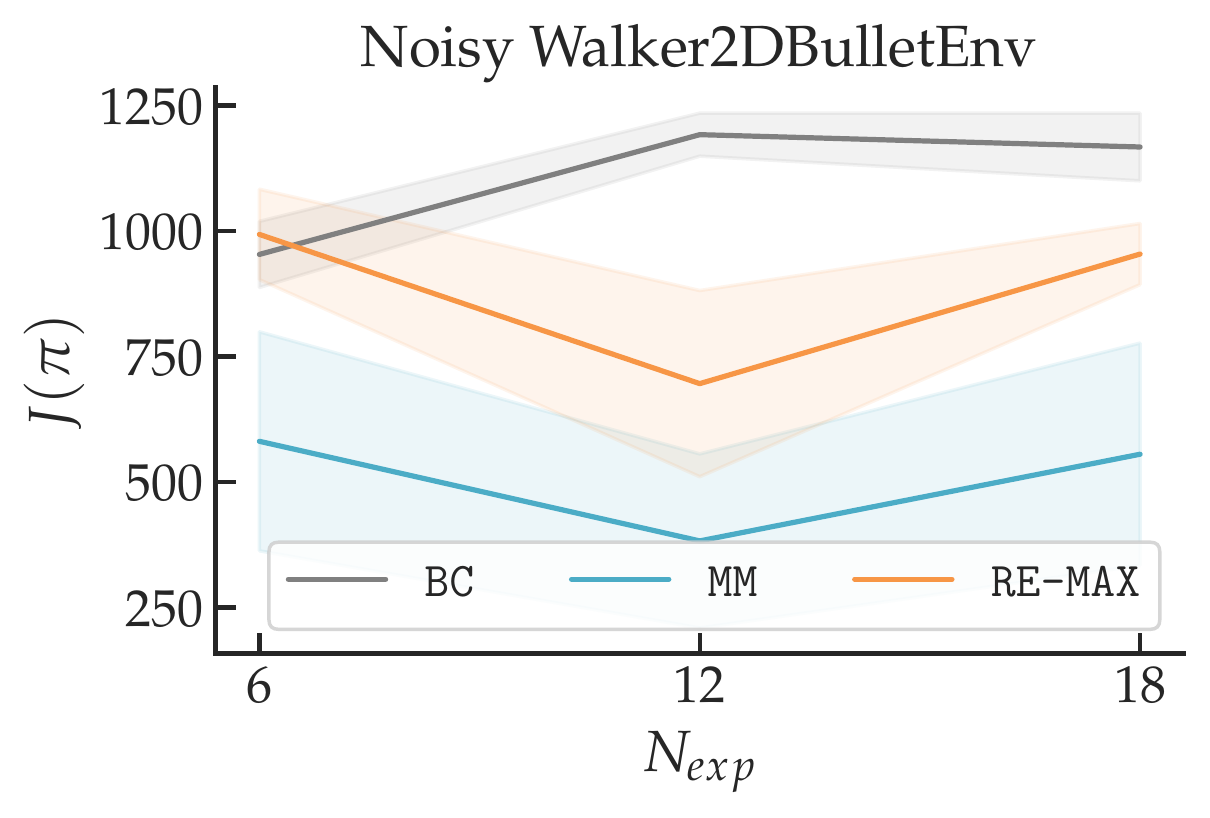}
\includegraphics[width=0.245\textwidth]{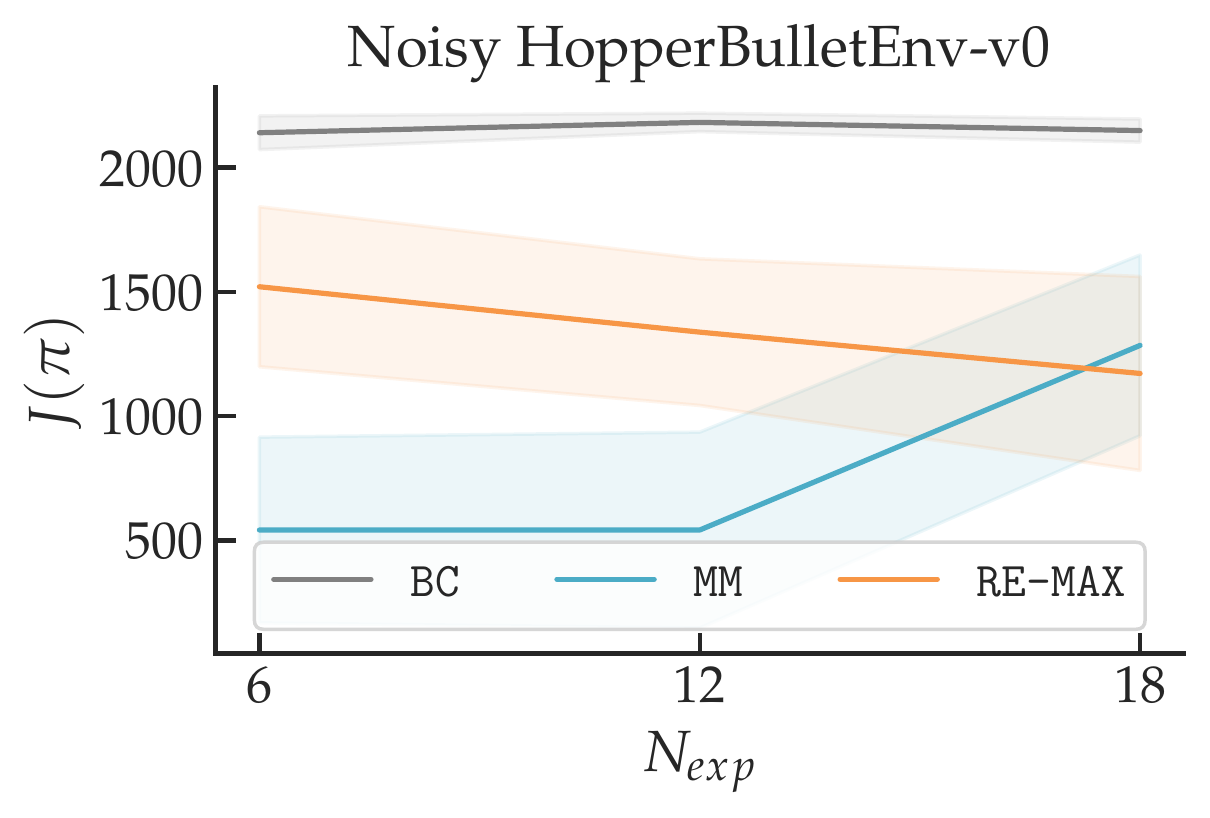}
\includegraphics[width=0.245\textwidth]{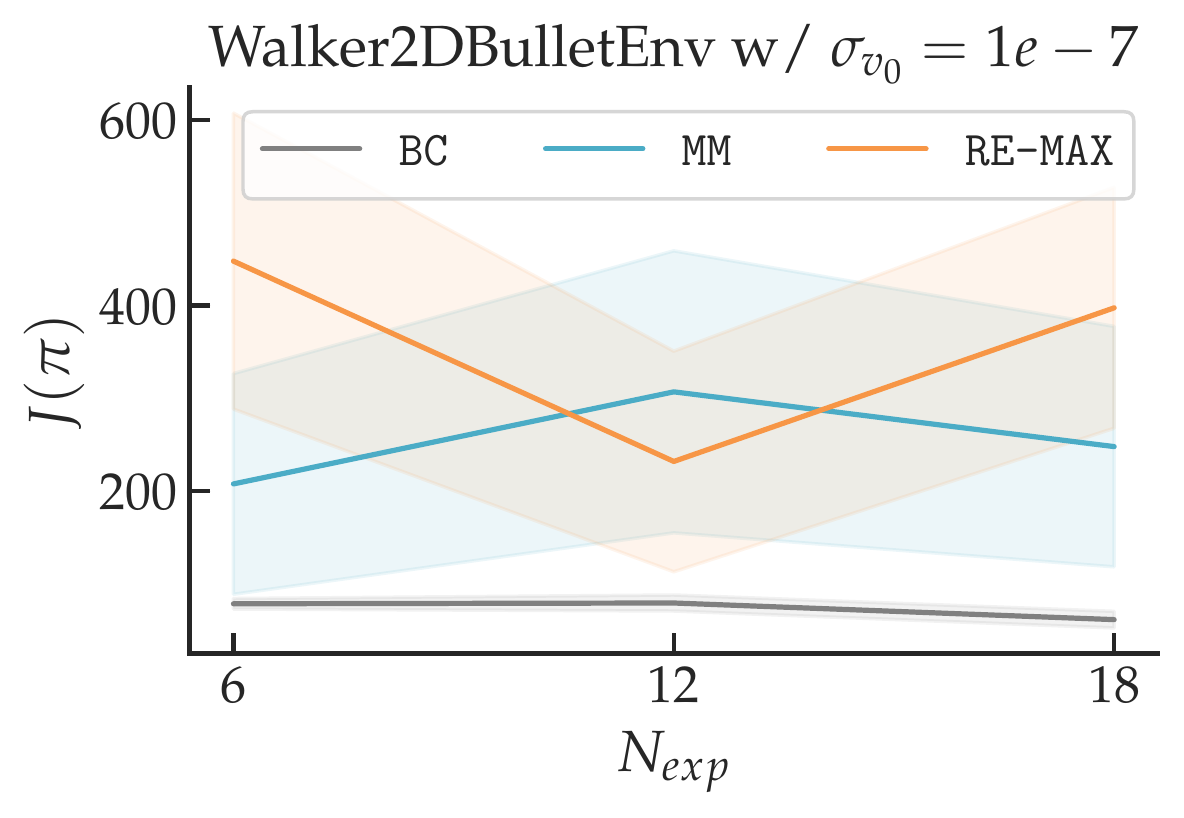}
\includegraphics[width=0.245\textwidth]{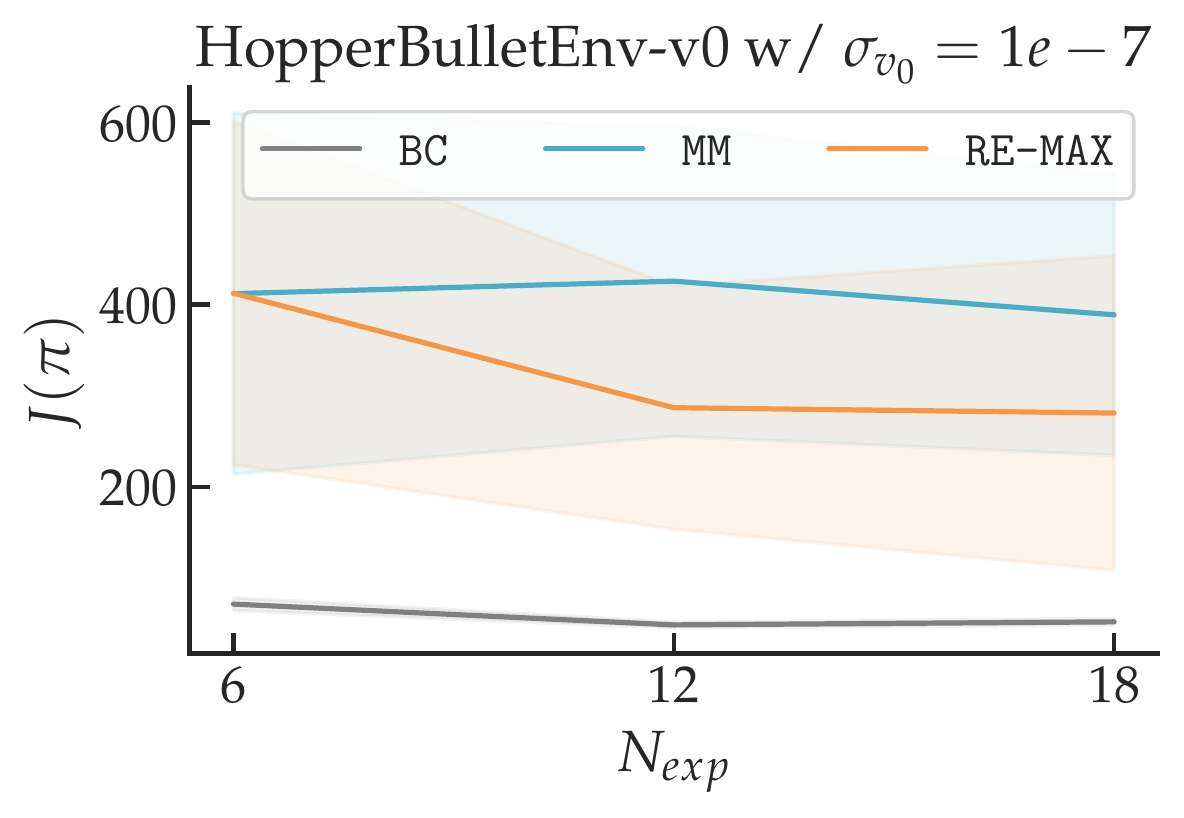}
\caption{We see \texttt{RE} with $\mathcal{M}_{\texttt{MAX}}$ improve the performance of \MM on the Noisy Walker2DBulletEnv and HopperBulletEnv tasks. We see \texttt{RE} (and \MM) out-perform \BC on the initial-state-perturbed Walker2DBulletEnv and HopperBulletEnv tasks \label{fig:moreres}.}
\end{figure*}

\textbf{Moment Matching. } 

We implement approximate Nash equilibrium computation of \eqref{eq:pire} by running a no-regret learner against a best-response counterpart \citep{swamy2021moments}. Our approach is related to the GAIL algorithm of \citet{ho2016generative} which we improve in $4$ ways: \textit{(i)} we use a general Integral Probability Metric \cite{muller1997integral} instead of the Jensen-Shannon Divergence used in the original paper which improves the representation power of the the discriminator,
\textit{(ii)} we add in gradient penalties to the discriminator, which improves convergence rates \cite{gulrajani2017improved}, \textit{(iii)} we solve the entropy-regularized forward problem via Soft-Actor Critic \cite{haarnoja2018soft} as the policy optimizer, for improved sample efficiency, and \textit{(iv)} we use optimistic mirror descent instead of gradient descent as our optimization algorithm for both players, giving us faster convergence to Nash equilibria, both in theory \cite{syrgkanis2015fast} and in practice \cite{daskalakis2017training}. Together, these changes lead to an implementation which \textit{significantly out-performs the original}, giving us a strong baseline to compare against. We include an ablation to confirm this fact in \Cref{app:more_res}. %
We emphasize that the \RE technique can be used to improve \textit{any} online moment matching algorithm and that the above description is merely the approach we chose for this paper.

\vspace{-0.75em}

\section{Experimental Results}
\vspace{-0.5em}
We now quantify the empirical benefits of \RE on several continuous control tasks from the the PyBullet suite \cite{coumans2019}. All the task we consider have long horizons ($H \approx 1000$) and we use relatively few demonstrations. ($\Nexp \leq 20$). We set $\Nrep$ as 100 \BC rollouts (Line 4 of \Cref{alg:re}). We test all four membership oracles from the previous section ($\mathcal{M}_{\texttt{EXP}}$ as an idealized target, $\mathcal{M}_{\texttt{RND}}$, $\mathcal{M}_{\texttt{VAR}}$, and $\mathcal{M}_{\texttt{MAX}}$ as practical solutions). In \figref{fig:res} (left), we see that with only twelve trajectories, \RE is able to reliably match expert performance for all oracles considered, while \MM is not. The environment considered in this experiment is nearly deterministic, indicating that \RE can help even when the environment is not stochastic. We hypothesize that the randomness in the initial state is sufficient for replay estimation to generate a significant improved estimate of the state-action visitation measure. This improvement is especially interesting considering both of the examples we study in \sref{sec:theory} were heavily stochastic. We see a similar result in \figref{fig:res} (center), where we add in i.i.d. noise to the environment dynamics at each timestep. This makes the problem significantly more challenging than the standard version of the Walker task. \RE is still able to match expert performance, with $\mathcal{M}_{\texttt{MAX}}$ working notably well. The correlation plot in \figref{fig:res} (right) shows us $\mathcal{M}_{\texttt{MAX}}$ appears to be best correlated with the idealized prefix weights, $\mathcal{M}_{\texttt{EXP}}$ under the state distribution induced by \BC. Because of its superior performance, we use $\mathcal{M}_{\texttt{MAX}}$ for the rest of our experiments. In the left half of \figref{fig:moreres}, we see \RE improve the performance of \MM. In the right half, we see \RE out-perform \BC in responding to an \textit{extremely tiny} amount of noise added to the initial velocity of the agent (similar to the experiments of \citet{reddy2019sqil})  -- we defer more details to \Cref{app:more_res}). These results indicate that \RE can out-perform \MM and \BC, agreeing with our theory. We release our code at \textbf{\texttt{\url{https://github.com/gkswamy98/replay_est}}}. \footnote{
After the initial publication of this paper, we continued to tune our baseline and fix bugs in our method's implementation. We then updated the result plots for both of the Noisy environments. Compared to their original performances, \texttt{MM} scores higher and \texttt{RE} scores lower. However, the \texttt{RE} still significantly out-performs \texttt{MM}.}

\section{Discussion and open problems}

In this paper we propose and analyze a new meta-algorithm for imitation learning called Replay Estimation (\RE). The key technique behind \RE is to improve the quality of moment estimates by rolling out artificial trajectories by using the knowledge of the expert policy at states. This approach can be implemented on top of any online moment-matching algorithm (e.g. \cite{ziebart2008maximum}, \cite{ho2016generative}, \cite{swamy2021moments}) to improve performance. We prove a general meta theorem bounding the imitation gap of the policies produced by the algorithm in terms of the parameter estimation error in offline classification. This results in a guarantee for IL with the optimal dependence on the horizon, $H$, and number of expert rollouts available, $\Nexp$, under the assumption that the parameter estimation guarantee matches the generalization guarantee for the underlying offline classification problem. Under these conditions, the analysis also suggests a natural measure of complexity for IL depending on the average Natarajan dimension and log-covering number (metric entropy) of the discriminator class. It is a significant open question to extend the analysis of \RE to depend on less stringent classification oracles, and only require constructing learners with bounded generalization error, as required by \BC. This is motivated by a briefly discussion on the drawbacks of the offline classification oracle we consider in \Cref{A:classorc}.
\paragraph{Drawbacks of the offline classification oracle (\Cref{A:classorc}).} Note that the offline classification oracle we consider in \Cref{A:classorc} requires a learner with bounded parameter estimation error, compared to one with bounded generalization error which is typically studied in practice \citep{DBLP:journals/corr/DanielySBS13}. While under certain conditions, both measures of error have the same asymptotic scaling in $n$ and the Natarajan dimension $\mathfrak{n}_{\Theta}$, in general they can be different. For instance, consider the case of linear (binary) classification under very poor coverage: denoting the true classifier as $\theta^* \in \mathbb{R}^d$, and with the input distribution as $s \sim \mathrm{Unif} ( \{e_1,-e_1\} )$ where $e_1 = (1,0,\cdots,0) \in \mathbb{R}^d$. Then, given $n = \Omega( \log(1/\delta) )$ samples, there exists a learner which with probability $1-\delta$ learns a classifier with $0$ generalization error.
However, regardless of the number of samples $n$, no learner can guarantee to learn $\theta^*$, since the remaining $d-1$ coordinates of $\theta^*$ aside from $\theta_1^*$ do not affect the labels of the inputs, which are what are observed by a learner. Thus, under poor coverage, the parameter $\theta^*$ cannot be learned consistently even though there exists a trivial classifier with bounded generalization error.
\paragraph{Designing practical algorithms which do not require parameter tuning.} The membership oracle resulting from the analysis in theory (\cref{eq:memdef}) as well as those implemented in practice (\cref{eq:memexp,eq:memrnd,eq:memvar,eq:memmax}) require tuning either the margin threshold $2 L \eNtH$, or the scale parameters $\mu$ and $\beta$. An interesting next direction would be to develop an algorithm which uses data dependent scales (e.g. \cite{haarnoja2018soft}) to reduce the effort required to fine-tuning these parameters.

\clearpage
\bibliography{arxiv-version}
\bibliographystyle{icml2022}

\newpage

%%%%%%%%%%%%%%%%%%%%%%%%%%%%%%%%%%%%%%%%%%%%%%%%%%%%%%%%%%%%%%%%%%%%%%%%%%%%%%%
%%%%%%%%%%%%%%%%%%%%%%%%%%%%%%%%%%%%%%%%%%%%%%%%%%%%%%%%%%%%%%%%%%%%%%%%%%%%%%%
% APPENDIX
%%%%%%%%%%%%%%%%%%%%%%%%%%%%%%%%%%%%%%%%%%%%%%%%%%%%%%%%%%%%%%%%%%%%%%%%%%%%%%%
%%%%%%%%%%%%%%%%%%%%%%%%%%%%%%%%%%%%%%%%%%%%%%%%%%%%%%%%%%%%%%%%%%%%%%%%%%%%%%%
\newpage
\appendix
\onecolumn

\newpage

\section{Proofs}
\label{app:proofs}

\subsection{Notation}

In this appendix, we use the notation $d_t^\pi (\cdot,\cdot)$ to indicate the state-action visitation measure induced by the policy $\pi$ at time $t$. We overload the notation $d_t^\pi (\cdot)$ to denote the state-visitation measure induced by the policy $\pi$ at time $t$. Likewise, the notations $d_t^D (\cdot,\cdot)$ and $d_t^D (\cdot)$ indicate the empirical visitation measures in the dataset $D$. For a function $g : \mathcal{X} \to \mathbb{R}$, the norm $\| g \|_\infty \triangleq \sup_{x \in \mathcal{X}} | g(x)|$.

Before discussing the proofs of the results, we also explain the instantiation of the function class in the tabular setting below.

\begin{remark} \label{rem:1}
In the tabular setting, we instantiate the discriminator class as $\mathcal{F}_t = \{ f_t : \| f_t \|_\infty \le 1 \}$ for each $t$, as the set of all $1$-bounded functions, and the policy class $\Pi$ as the set of all policies. \cref{eq:emm} corresponds to finding a policy which best matches the empirical state-action visitation measure observed in the dataset $D$ in \textit{total variation (TV) distance} (see \Cref{sec:MM-ub} for a proof).
\end{remark}

\subsection{Imitation gap upper bound on empirical moment matching (\Cref{theorem:folklore})} \label{sec:MM-ub}

Below we restate \Cref{theorem:folklore} and provide a proof of this result. The key observation is that since the learner $\piEMM$ best matches the empirical distribution in the dataset, which is in turn close to the population visitation measure induced by $\pi^E$, we can expect the visitation measure induced by $\pi^E$ and $\piEMM$ to be close. This in turns implies that both policies will collect a similar value under any reward function. Precisely characterizing the rates at which these distributions converge to one another results in the final bound.

\folklore*
\begin{proof}
Recall that the learner $\piEMM$ is the solution to the following optimization problem,
\begin{align}
    \argmin_\pi \sup_{ f \in \mathcal{F}} \left\{ \mathbb{E}_{\pi} \left[\frac{ \sum_{t=1}^H f_t (s_t,a_t)}{H} \right] - \mathbb{E}_{D} \left[\frac{ \sum_{t=1}^H f_t (s_t,a_t)}{H} \right] \right\}
\end{align}
Exchanging the summation and maximization operators and recalling from \Cref{rem:1} that in the tabular setting, the discriminator class $\mathcal{F}$ is instantiated as the set of all $1$-bounded functions $\bigoplus_{t=1}^H \{f_t : \| f_t \|_\infty \le 1 \}$, $\piEMM$ is a solution to
\begin{equation} \label{eq:altobj}
    \argmin_\pi \frac{1}{H} \sum_{t=1}^H \Big( \sup_{f : \| f \|_\infty \le 1} \mathbb{E}_{\pi} \left[ f_t (s_t,a_t) \right] - \mathbb{E}_{D} \left[ f_t (s_t,a_t) \right] \Big) = \argmin_\pi \frac{1}{H} \sum_{t=1}^H \tv \left( d_t^\pi , d_t^D \right)
\end{equation}
where the equation follows by the variational definition of the total variation distance, and where $d_t^\pi$ is the state-action visitation measure induced by $\dem$ and $d_t^D$ is the empirical state-action visitation measure in the dataset $D$. 
The imitation gap of this policy can be upper bounded by,
\begin{align}
    J (\pi^E) - J(\piEMM) &= \mathbb{E}_{\pi^E} \left[ \sum_{t=1}^H r_t (s_t,a_t) \right] - \mathbb{E}_{\piEMM} \left[ \sum_{t=1}^H r_t (s_t,a_t) \right] \\
    &\overset{(i)}{\le} \sum_{t=1}^H \sup_{r_t : \| r_t \|_\infty \le 1} \Big( \mathbb{E}_{\pi^E} \left[  r_t (s_t,a_t) \right] - \mathbb{E}_{\piEMM} \left[ r_t (s_t,a_t) \right] \Big) \\
    &\overset{(ii)}{=} \sum_{t=1}^H \tv \left( d_t^{\pi^E} (\cdot,\cdot) , d_t^{\piEMM} (\cdot,\cdot) \right)
\end{align}
where $(i)$ maximizes over the reward function which is assumed to lie in the interval $[0,1]$ pointwise. $(ii)$ again follows from the variational definition of total variation distance. This goes to show that in the tabular setting, \MM is equivalent to finding the policy which best matches (in TV-distance) the empirical state-action distribution observed in the dataset.

By an application of triangle inequality,
\begin{align}
    J (\pi^E) - J(\piEMM) &\le \sum_{t=1}^H \tv \left( d_t^{\pi^E} (\cdot,\cdot), d_t^D (\cdot,\cdot) \right) + \tv \left( d_t^D (\cdot,\cdot), d_t^{\piEMM} (\cdot,\cdot) \right) \\
    &\le 2\sum_{t=1}^H \tv \left( d_t^{\pi^E} (\cdot,\cdot), d_t^D (\cdot,\cdot) \right) \label{eq:final-decomp}
\end{align}
where $(i)$ follows from \cref{eq:altobj} which shows that $\piEMM$ is the policy which best approximates the empirical state-action visitation measure in total variation distance, and therefore $\tv \left( d_t^{\piEMM} (\cdot,\cdot), d_t^D (\cdot,\cdot) \right) \le \tv \left( d_t^{\pi^E} (\cdot,\cdot), d_t^D (\cdot,\cdot) \right)$. The final element is to identify the rate of convergence of the empirical visitation measure $d_t^D$, to the population $d_t^{\pi^E}$ in total variation distance. This result is known from Theorem~1 of \cite{han2015minimax}, which shows that $\mathbb{E} \left[ \tv \left( d_t^{\pi^E} (\cdot,\cdot), d_t^D (\cdot,\cdot) \right) \right] \lesssim \sqrt{\frac{|\mathcal{S}|}{\Nexp}}$ noting that $d_t^{\pi^E}$ is a distribution with support size $|\mathcal{S}|$ since $\dem$ is deterministic. Putting it together with \cref{eq:final-decomp} after taking expectations on both sides gives,
\begin{align}
    J (\pi^E) - J(\piEMM) &\lesssim \sum_{t=1}^H \sqrt{\frac{|\mathcal{S}|}{\Nexp}} = H \sqrt{\frac{|\mathcal{S}|}{\Nexp}}.
\end{align}
This completes the proof of the result.
\end{proof}

\subsection{Lower bounding the Imitation gap of Behavior Cloning} \label{sec:BC-lb}

Since \BC is an offline algorithm - namely, the learner does not interact with the MDP, any lower bound against offline algorithms applies for behavior cloning as well. The lower bound instance in \cite{rajaraman2020toward} is one such example. One state in the MDP is labelled as the ``bad'' state, $b$, which is absorbing and offers no reward. The remaining states each have a single ``good'' action which re-initializes the policy in a particular distribution $\rho$ on the set of good states, and offering a reward of $1$. The other actions at these states are ``bad'' and transition the learner to the bad state $b$ with probability $1$.

The key idea in the lower bound is that any offline algorithm does not know (i) which action the expert would have chosen at states unvisited in the dataset, and (ii) which action does what at these states. At best, the learner can correctly guess the good actions at a state with probability $1/|\mathcal{A}| \le 1/2$. So, at each state unvisited in the dataset, the learner has a constant probability of getting stuck at the bad state in the MDP and collecting no reward then on. On the other hand, the expert would never choose bad actions at states and collects the maximum possible reward. By carefully counting the probability mass on the unvisited states, the expected imitation gap of any offline IL algorithm (such as \BC) can be lower bounded by $\Omega (|\mathcal{S}| H^2/\Nexp)$ on these instances. Thus we have the following theorem,

\begin{theorem}[Theorem~6.1 of \cite{rajaraman2020toward}]
Consider any learner $\widehat{\pi}$ which carries out an offline imitation learning algorithm (such as behavior cloning). Then, there exists an MDP instance such that the expected imitation gap is lower bounded by,
\begin{align}
    \mathbb{E} \left[ J(\pi^E) - J(\widehat{\pi}) \right] \gtrsim \min \left\{ H, \frac{|\mathcal{S}|H^2}{\Nexp} \right\}.
\end{align}
\end{theorem}

\subsection{Lower bounding the imitation gap of Empirical Moment Matching} \label{sec:MM-lb}

In this section, we show that in the tabular setting, empirical moment matching is suboptimal compared for imitation learning in the worst-case. The main result we prove in this section is,

\tvlb*

\begin{remark}
It is known from \cite{rajaraman2020toward} that the \texttt{Mimic-MD} algorithm achieves an imitation gap of $\min \left\{ H, \frac{\mathcal{S}| H^{3/2}}{\Nexp}, H\sqrt{\frac{|\mathcal{S}|}{\Nexp}} \right\}$. This is always better than the worst case error bound incurred by TV distribution matching from \Cref{theorem:tvlb}. In fact when $\Nexp \gtrsim \sqrt{H}$ the bound $\frac{H^{3/2}}{\Nexp}$ is significantly better than $\frac{H}{\sqrt{\Nexp}}$ which decays as $1/\sqrt{\Nexp}$. This is illustrated in \Cref{fig:samp} .
\end{remark}

First note that the learner $\piEMM$ carries out empirical moment matching (\cref{eq:emm}), with the discriminator class $\mathcal{F}$ as initialized in \Cref{rem:1}. As shown in \cref{eq:altobj}, the empirical moment matching learner can be redefined as the solution to a distribution matching problem,
\begin{equation} \label{eq:altobj-restate}
    \argmin_\pi \frac{1}{H} \sum_{t=1}^H \tv \left( d_t^\pi (\cdot,\cdot) , d_t^D (\cdot,\cdot) \right)
\end{equation}

Consider an MDP instance with $2$ states and $2$ actions with a non-stationary transition and reward structure as described in Figure~\cref{fig:mm-lowerbound}. State $1$ effectively has a single action (i.e. two actions, $a_1$ and $a_2$ with both inducing the same next-state distribution and reward). One of the actions at state $2$ induces the uniform distribution over next states. The other action deterministically keeps the learner at state $2$. The reward function is $0$ at $t=1$, and the action $a_2$ at state $2$ is the only one which offers a reward of $1$. The initial state distribution is highly skewed toward the state $s=1$ and places approximately $1/\sqrt{\Nexp}$ mass on $s=2$ and the remaining on $s=1$.

\begin{figure}[h!]
\centering
\begin{tikzpicture}[shorten >=1pt,node distance=1.25cm,on grid,auto,good/.style={circle, draw=expert, very thick, minimum size=9mm, inner sep=0pt},bad/.style={circle, draw=black, fill=red!30, thick, minimum size=9mm, inner sep=0pt}]
\tikzset{every edge/.style={very thick, draw=error!90!black}}    \tikzset{every loop/.style={min distance=9mm,in=86,out=125, looseness=8, very thick, draw=red!90!black}}
    
\node[good, color=expert]     (s_1)                   {$s_1$};
\node           (dist1) [right = 2cm of s_1]  {$\mathrm{Unif} (\mathcal{S})$};
\node[good, color=expert]      (s_2)     [right = 4cm of s_1]  {$s_2$};

\path[->,>=stealth]
(s_1) edge [draw=expert] node {} (dist1)
(s_2)
      edge [draw=expert] node {} (dist1);
    \node [circle, minimum size=0.75cm](d) at ([{shift=(60:0.4)}]s_2){};
    \coordinate (e) at (intersection 2 of s_2 and d);
    \coordinate (f) at (intersection 1 of s_2 and d);
    \tikzAngleOfLine(d)(f){\AngleStart}
    \tikzAngleOfLine(d)(e){\AngleEnd}
    \draw[very thick,->, color=error]%
    let \p1 = ($ (d) - (f) $), \n2 = {veclen(\x1,\y1)}
    in
        (d) ++(60:0.4) node{}
        (f) arc (\AngleStart-360:\AngleEnd:\n2);
\end{tikzpicture}
\caption{MDP instance which shows that $L_1$ distribution matching is suboptimal. Here the transition structure is illustrated for $t=1$. Both states have one action which reinitializes in the uniform distribution. State $2$ has an additional action which keeps the state the same. The reward function is $0$ for $t=1$. For $t \ge 2$ the transition function is absorbing at both states; the reward function equals $1$ at the state $s=1$ for any action and is $0$ everywhere else.}

\label{fig:mm-lowerbound}
\end{figure}
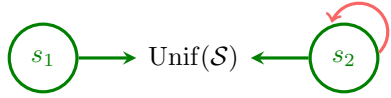

\textbf{MDP transition:} 
The state $2$ is the only one with two actions. Action $a_1$ induces the uniform distribution over states, while action $a_2$ transitions the learner to state $2$ with probability $1$. Namely,
\begin{align}
    P_1 (\cdot | s=1,a) &= \mathrm{Unif} (\mathcal{S}) \text{ for all } a \in \mathcal{A}\\
    P_1 (\cdot | s=2,a_1) &= \mathrm{Unif} (\mathcal{S}) \\
    P_1 (\cdot | s=2,a_2) &= \delta_2
\end{align}
From time $t=2$ onward, the actions are all absorbing. Namely, for all $t \ge 2$, $s \in \mathcal{S}$ and $a \in \mathcal{A}$,
\begin{align}
    P_t (\cdot|s,a) = \delta_s.
\end{align}

\textbf{Initial state distribution:} The initial state distribution $\rho = \left( 1-\frac{1}{\sqrt{\Nexp}}, \frac{1}{\sqrt{\Nexp}} \right)$.

\textbf{MDP reward function:}
The reward function of the MDP encourages the learner to stay at the state $s = 1$ from time $t=2$ onward. Namely,
\begin{align}
    &r_t (s,a) = \begin{cases}
    1, \qquad& \text{if } t \ge 2 \text{ and } s = 1 \\
    0, &\text{otherwise.}
    \end{cases} 
    \label{eq:reward}
\end{align}

\textbf{Expert policy:} At both states in the MDP, the expert picks the action $a_1$ to play, which induces the uniform distribution over actions at the next state. Namely, for each $t \in [H]$ and $s \in \mathcal{S}$,
\begin{align}
    \pi^E_t (\cdot|s) = \delta_{a_1}
\end{align}

The intuition behind the lower bound is as follows. The only action which affects the value of a policy is the choice made at $s=2$ at time $t=1$. At all other states, we may assume that there is effectively only a single action. 

By the absorbing nature of states for $t \ge 2$, it turns out that if the observed empirical distribution in the dataset at time $2$ is skewed toward state $2$ (which is possible because of the inherent randomness in the data generation process), the learner's behavior at time $1$ may be to \textit{ignore} the expert's action observed at state $s=2$, and instead pick the action $a_2$ which moves the learner to the state $s=2$ deterministically. The learner is willing to choose a different action because the loss function \cref{eq:altobj-restate} encourages the state-action distribution at time $t=2$ also to be well matched with what is observed in the dataset. Even if it comes at the cost of picking an action different from what the expert plays. By exploiting this fact, we are able to show that the error incurred by a learner which solves \cref{eq:altobj-restate} in this simple $2$ state example must be $\Omega (H/\sqrt{\Nexp})$.

Formally, we define $3$ events,
\begin{enumerate}
    \item[(i)] $\mathcal{E}_1$: All states in the MDP are visited in the dataset at each time $t=1,2,\cdots,H$.
    \item[(ii)] $\mathcal{E}_2$: State $2$ is visited at most $\sqrt{\Nexp}$ times at time $1$ in the dataset $D$. In other words, $d_1^D (s=2) = \delta'$ where $\delta' \le \frac{1}{\sqrt{\Nexp}}$.
    \item[(iii)] $\mathcal{E}_3$: At time $2$ in the dataset $D$, the empirical distribution over states is of the form $\left( \frac{1}{2} - \delta, \frac{1}{2} + \delta \right)$ for some $\delta \ge \frac{2}{\sqrt{\Nexp}}$.
\end{enumerate}

\begin{lemma} \label{lemma:7}
Jointly, the events $\mathcal{E}_1,\mathcal{E}_2$ and $\mathcal{E}_3$ occur with at least constant probability.
\begin{align}
    \mathrm{Pr} (\mathcal{E}_1 \cap \mathcal{E}_2 \cap \mathcal{E}_3) \ge C,
\end{align}
for some constant $C > 0$.
\end{lemma}
\begin{proof}
By the absorbing nature of states for $t \ge 2$, it suffices for both states of the MDP to be visited in the dataset at time $t=1,2$. At time $t=2$, the marginal state distribution under $\pi^E$ is the uniform distribution. By binomial concentration, both states are observed in the dataset at time $t=2$ with probability $\ge 1 - e^{-C_1 \Nexp}$ for some constant $C_1 > 0$. On the other hand, at time $t=1$, the marginal state distribution is $\rho = \left( 1 - \frac{1}{\sqrt{\Nexp}}, \frac{1}{\sqrt{\Nexp}} \right)$. Yet again, by binomial concentration, both states are observed with probability $\ge 1 - e^{- C_2 \sqrt{\Nexp}}$ for some constant $C_2 > 0$. By union bounding,
\begin{equation}
    \mathrm{Pr} (\mathcal{E}_1) \ge  1 - e^{-C_1 \Nexp} - e^{-C_2 \sqrt{\Nexp}}.
\end{equation}

Next we study $\mathcal{E}_2$ and $\mathcal{E}_3$ together. First of all, note that the state observed at $t=1$ and $t=2$ in a rollout of the expert policy are independent. This is because at both states at $t=1$, the next state distribution under $\pi^E$ is uniform. Because of this fact, $\mathcal{E}_2$ and $\mathcal{E}_3$ are independent. Next we individually bound the probability of the two events.

$\mathcal{E}_2$: The number of times $s=2$ is the initial state in trajectories the dataset $D$ is distributed as a binomial random variable with distribution $\mathrm{Bin} (\Nexp,q)$ with $q = \rho(s=2) = \frac{1}{\sqrt{\Nexp}}$. A median of a binomial random variable is $\Nexp q = \sqrt{\Nexp}$ (in fact any number in the interval $[ \lfloor \Nexp q \rfloor, \lceil \Nexp q \rceil ]$ is a median). Therefore, the probability that $s=2$ is visited $\le \sqrt{\Nexp}$ times in the dataset at time $1$ is at least $1/2$. In summary,
\begin{equation}
    \mathrm{Pr} (\mathcal{E}_2) \ge \frac{1}{2}
\end{equation}

$\mathcal{E}_3$: The marginal distribution over states at time $2$ in the dataset is uniform. Therefore, we expect the states $1$ and $2$ to be visited roughly $\Nexp/2$ times each in the dataset, but with a random variation of $\approx \sqrt{\Nexp}$ around this average. In other words, the empirical distribution fluctuates as $\left( \frac{1}{2} - \delta, \frac{1}{2} + \delta \right)$ with $\delta \ge \frac{2}{\sqrt{\Nexp}}$ with constant probability.

By the independence of $\mathcal{E}_2$ and $\mathcal{E}_3$ and union bounding to account for $\mathcal{E}_1$, the statement of the lemma follows.
\end{proof}

\begin{lemma} \label{lemma:8}
For each $t \ge 2$,
\begin{equation} \label{eq:1231200}
    \tv (d_t^\pi (\cdot,\cdot), d_t^D (\cdot,\cdot)) \ge \tv (d_2^\pi (\cdot), d_2^D (\cdot)).
\end{equation}
The RHS is the TV distance between the state-visitation measure at time $t=2$ under $\pi$ and that empirically observed in the dataset $D$. Conditioned on the events $\mathcal{E}_1$, $\mathcal{E}_2$ and $\mathcal{E}_3$ occuring, equality is met in \cref{eq:1231200} if any only if $\pi_t (\cdot|s) = \pi^E_t (\cdot|s)$ for all states $s \in \mathcal{S}$.
\end{lemma}
\begin{proof}
For any state $s \in \mathcal{S}$ and $t \ge 2$, observe that,
\begin{align}
    &\sum_{a \in \mathcal{A}} \left| d_t^\pi (s,a) - d_t^D (s,a) \right| \\
    &= d_t^\pi (s) (1 - \pi_t (a^*|s)) + \left| d_t^\pi (s) \pi_t (a^*|s) - d_t^D (s,a^*) \right|, \text{ where } a^* = \pi^E_t (s),\\
    &\overset{(i)}{=} d_2^\pi (s) (1 - \pi_t (a^*|s)) + \left| d_2^\pi (s) \pi_t (a^*|s) - d_2^D (s,a^*) \right|\\
    &\overset{(ii)}{\ge} \left| d_2^\pi (s) - d_2^D (s) \right|,
\end{align}
where $(i)$ follows by the fact that the states of the MDP are absorbing under $\pi$ for $t \ge 2$. $(ii)$ follows by triangle inequality and using the fact that $\pi^E$ is deterministic, so $d_t^D (s,a^*) = d_t^D (s)$. Equality is met only if $\pi_t (a^*|s) = 1$ (since $d_t^D (s,a^*) > 0$ conditioned on $\mathcal{E}_1$).
\end{proof}

The above lemma asserts the behavior of $\piEMM$ in \cref{eq:altobj-restate} for $t \ge 2$. Namely, conditioned on the event $\mathcal{E}_1$ which happens with very high probability, all states are visited in the MDP and therefore, $\piEMM_t (\cdot|s) = \pi^E_t (\cdot|s)$ for each state $s \in \mathcal{S}$ and time $t \ge 2$.

The only thing left to study is the \MM learner's behavior at $t=1$. We wish to show that with constant probability, the learner may choose to deviate from the expert policy in order to better match empirical state-action visitation measures. Conditioned on $\mathcal{E}_1$, the learner's policy at time $t=1$ can be computed by solving the following optimization problem,
\begin{align}
    \tv (d_1^\pi (\cdot,\cdot), d_1^D (\cdot,\cdot)) + (H-1) \textsf{TV} (d_2^\pi (\cdot), d_2^D (\cdot)).
\end{align}
This follows directly by simplifying the learner's objective using \Cref{lemma:8}.

Now, conditioned on the event $\mathcal{E}_1$, at time $t=1$, the learner policy only needs to be optimized at the state $s=2$. At the state $s=1$, we may assume that the learner picks the expert's action $\pi^E_1 (s=1)$. To this end, suppose the learner picks the action $a_1$ with probability $p$ and the action $a_2$ with probability $1-p$. 
\begin{align}
    \textsf{TV} ( d_1^\pi (\cdot,\cdot), d_1^D (\cdot,\cdot)) &= \sum_{a \in \mathcal{A}} \left| d_1^{\pi^E} (s=2,a) - d_1^D (s=2,a) \right|\\
    &= \left| \rho(2) p - \delta' \right| + \left| \rho(2) (1-p) - 0 \right| \\
    &= \left| \frac{p}{\sqrt{\Nexp}} - \delta' \right| + \frac{1-p}{\sqrt{\Nexp}}. \label{eq:12312}
\end{align}
which follows by plugging in $\rho(2) = \frac{1}{\sqrt{\Nexp}}$. And,
\begin{align}
    \textsf{TV} ( d_2^\pi (\cdot), d_2^D (\cdot)) &= \left| \left( \frac{1}{2} - \delta \right) - \frac{\rho(1)}{2} - \rho(2) \frac{p}{2} \right| + \left| \left( \frac{1}{2} + \delta \right) - \frac{\rho(1)}{2} - \rho(2) \left( (1-p) + \frac{p}{2} \right) \right|.
\end{align}
Plugging in $\rho (2) = \frac{1}{\sqrt{\Nexp}}$ and $\rho(1) = 1 - \frac{1}{\sqrt{\Nexp}}$, we get,
\begin{align} \label{eq:12313}
    \textsf{TV} ( d_2^\pi (\cdot), d_2^D (\cdot)) &= \left| \frac{1}{2\sqrt{\Nexp}} - \delta - \frac{p}{2\sqrt{\Nexp}} \right| + \left| \frac{p}{2\sqrt{\Nexp}} - \frac{1}{2\sqrt{\Nexp}} + \delta \right|.
\end{align}
Summing up \cref{eq:12312,eq:12313}, $p$ minimizes,
\begin{equation} \label{eq:minimizer}
    \underbrace{\left| \frac{p}{\sqrt{\Nexp}} - \delta' \right| + \frac{1-p}{\sqrt{\Nexp}}}_{(i)} + \underbrace{(H-1) \left( \left| \frac{p}{2\sqrt{\Nexp}} + \delta - \frac{1}{2\sqrt{\Nexp}} \right| + \left| \frac{1}{2\sqrt{\Nexp}} - \delta - \frac{p}{2\sqrt{\Nexp}} \right| \right)}_{(ii)}.
\end{equation}
Intuitively, term $(i)$ captures the error incurred by the learner in the loss \cref{eq:altobj-restate} by deviating from $\pi^E$ at the first time step. Term $(ii)$ captures the decrease in the error at every subsequent time step because of the same deviation, since the learner is able to better match the state distribution at future time steps. In the next lemma we show that under events that hold with at least constant probability, the empirical moment matching learner chooses to play the wrong action at time $t=1$ at the state $s=2$.

\begin{lemma} \label{lemma:9}
Conditioned on the events $\mathcal{E}_2$ and $\mathcal{E}_3$, for $H \ge 4$, the unique minimizer of \cref{eq:minimizer} for $p \in [0,1]$ is $p = 0$.
\end{lemma}
\begin{proof}
The first term of \cref{eq:minimizer} is $\left| \frac{p}{\sqrt{\Nexp}} - \delta' \right| + \frac{1-p}{\sqrt{\Nexp}}$, the error from not picking the expert's action at state $1$ at time $1$ decreases at most linearly with a slope of $\frac{2}{\sqrt{\Nexp}}$.

Conditioned on the event $\mathcal{E}_3$, $\delta \ge \frac{2}{\sqrt{\Nexp}}$. Therefore, $\left| \frac{p}{2\sqrt{\Nexp}} + \delta - \frac{1}{2\sqrt{\Nexp}} \right| = \frac{p}{2\sqrt{\Nexp}} + \delta - \frac{1}{2\sqrt{\Nexp}}$. Therefore, the decrease in error at future steps by deviating from $\pi^E$ at the time $t=1$, term $(ii)$ in \cref{eq:minimizer} is,
\begin{align}
    2(H-1) \left(\frac{p}{2\sqrt{\Nexp}} + \delta - \frac{1}{2\sqrt{\Nexp}} \right)
\end{align}
which is an increasing function of $p$ with slope $\frac{H-1}{\sqrt{\Nexp}}$. For $H \ge 4$ and the argument from the previous paragraph, this implies that term $(ii)$ increases more rapidly in $p$ than the rate at which term $(i)$ decreases. Therefore, the minimizer must be $p = 0$.
\end{proof}

Thus, we conclude from \Cref{lemma:8,lemma:9} that conditioned on the events $\mathcal{E}_1,\mathcal{E}_2$ and $\mathcal{E}_3$, the learner $\piEMM$ perfectly mimics $\pi^E$ at each time $t \ge 2$, but deviates from the action played by $\pi^E$ at the state $s=1$ at time $t=1$.

Finally, we bound the difference in value between $\pi^E$ and $\piEMM$ induced because of this deviation under the reward \cref{eq:reward}.

\begin{lemma}
Under the events $\mathcal{E}_1,\mathcal{E}_2$ and $\mathcal{E}_3$, under the reward \cref{eq:reward}, the empirical moment matching learner $\piEMM$ incurs imitation gap,
\begin{equation}
    J(\pi^E) - J(\piEMM) = \frac{H}{2\sqrt{\Nexp}}.
\end{equation}
\end{lemma}
\begin{proof}
Recall that under the events $\mathcal{E}_1,\mathcal{E}_2$ and $\mathcal{E}_3$, the learner $\piEMM$ is identical to $\pi^E$ except at the state $s=2$ where they perfectly deviate from each other. The state distribution induced by $\pi^E$ at each time $t \ge 2$ is the uniform distribution over states $\left( \frac{1}{2},\frac{1}{2} \right)$. On the other hand, for $t \ge 2$, the state distribution induced by $\piEMM$ at each time $t \ge 2$ is $\left( \rho(1) \frac{1}{2}, \rho(1) \frac{1}{2} + \rho(2) \right) = \left( \frac{1 - 1/\sqrt{\Nexp}}{2}, \frac{1 + 1/\sqrt{\Nexp}}{2} \right)$. Since the reward function is $1$ on state $1$, the difference in value between the expert and learner policy is,
\begin{equation}
    J(\pi^E) - J(\piEMM) = \frac{H}{2} - H \left( \frac{1-1/\sqrt{\Nexp}}{2} \right) = \frac{H}{2\sqrt{\Nexp}}.
\end{equation}
This completes the proof.
\end{proof}

Since $\mathcal{E}_1$, $\mathcal{E}_2$ and $\mathcal{E}_3$ jointly occur with constant probability by \Cref{lemma:7}, this completes the proof of \Cref{theorem:tvlb}.

\section{Imitation gap of \RE: Proof of \Cref{thm:main}} \label{sec:RE-ub}

In this section, we discuss a proof of a more general version of \Cref{thm:main}, where $\Nrep$ can be finite.
We prove the following result,

\begin{theorem} \label{thm:main-general}
Consider the policy $\piRE$ returned by \Cref{alg:re}. Assume that $\pi^E \in \Pi$ and the ground truth reward function $r_t \in \mathcal{F}_t$, which is assumed to be symmetric ($f_t \in \mathcal{F}_t \iff -f_t \in \mathcal{F}_t$) and bounded (For all $f_t \in \mathcal{F}_t$, $\| f_t \|_\infty \le 1$). Choose $|D_1|, |D_2| = \Theta (\Nexp)$ and suppose $\Nrep \to \infty$. With probability $\ge 1 - 3 \delta$,
\begin{align}
    J (\dem) - J (\piRE) \lesssim & \ \mathcal{L}_1 + \mathcal{L}_2 + \mathcal{L}_3 + \frac{\log \left( \Fmax H/\delta \right)}{\Nexp}
\end{align}
where $\Fmax \triangleq \max_{t \in [H]} |\mathcal{F}_t|$, and,
\begin{align}
    \mathcal{L}_1 \triangleq H^2 \ \mathbb{E}_{\dem} \! \left[ \frac{\sum_{t=1}^H \mathcal{M} (s_t,t) \tv \left( \dem_t (\cdot|s_t) , \piBC_t (\cdot|s_t) \right)}{H} \right],
\end{align}
\begin{align}
    \mathcal{L}_2 \triangleq H^{3/2} \sqrt{ \frac{\log \left( \Fmax H/\delta \right)}{\Nexp} \frac{\sum_{t=1}^H \mathbb{E}_{\dem} \left[ 1 - \mathcal{M} (s_t, t) \right]}{H}},
\end{align}
And,
\begin{align}
    \mathcal{L}_3 \triangleq H \sqrt{\frac{\log (\Fmax H /\delta)}{\Nrep}} + \frac{H \log (\Fmax H /\delta)}{\Nrep}.
\end{align}
\end{theorem}

Recall that the learner carrying out replay estimation returns the policy which minimizes the loss $\sup_{f \in \mathcal{F}} J_f (\pi) - \widehat{E} (f)$ over policies $\pi$, where $J_f (\pi) \triangleq \mathbb{E}_\pi \left[ \frac{1}{H} \sum_{t=1}^H f_t (s_t,a_t) \right]$ and abbreviating the notation $f = (f_1,\cdots,f_H)$. Note that,
\begin{align}
    J (\dem) - J (\piRE) &\overset{(i)}{\le} \sup_{f \in \mathcal{F}} J_f (\dem) - J_f (\piRE) \\
    &\le \sup_{f \in \mathcal{F}} \left| J_f (\dem) - \widehat{E} (f) \right| + \sup_{f \in \mathcal{F}} \left| \widehat{E} (f) - J_f (\piRE) \right| \\
    &\overset{(ii)}{\le} 2 \sup_{f \in \mathcal{F}} \left| J_f (\dem) - \widehat{E} (f) \right|. \label{eq:MD}
\end{align}
where $(i)$ uses the realizability assumption that the ground truth reward lies in $\mathcal{F}$, and $(ii)$ uses the fact that $\piRE$ is a minimizer of \cref{eq:pire} and the fact that $\mathcal{F}$ is symmetric (this implies that $\sup_{f \in \mathcal{F}} J_f (\dem) - \widehat{E} (f) = \sup_{f \in \mathcal{F}} - J_f (\dem) + \widehat{E} (f) = \sup_{f \in \mathcal{F}} \left| J_f (\dem) - \widehat{E} (f) \right|$).

Note that $\widehat{E} (f)$ can be decomposed into a sum of two parts,
\begin{align}
    &\widehat{E}^{(1)} (f) = \mathbb{E}_{\Drep} \left[ \frac{1}{H} \sum_{t=1}^H f_t(s_t,a_t) \left( 1 - \mathcal{P}(s_{1 \dots t-1}) \right) \right], \text{ and,}\\
    &\widehat{E}^{(2)} (f) = \mathbb{E}_{D_2} \left[ \frac{1}{H} \sum_{t=1}^H f_t(s_t,a_t) \left( 1 - \mathcal{P}(s_{1 \dots t-1}) \right) \right]
\end{align}
Likewise, we can decompose $J_f (\dem)$ into two terms,
\begin{align}
    J_f^{(1)} (\dem) &\triangleq \mathbb{E}_{\dem} \left[ \sum_{t=1}^H f_t (s_t,a_t) \mathcal{P}(s_{1 \dots t-1}) \right], \text{ and}\\
    J_f^{(2)} (\dem) &\triangleq \mathbb{E}_{\dem} \left[  \sum_{t=1}^H f_t (s_t,a_t) \left(1 - \mathcal{P}(s_{1 \dots t-1}) \right) \right]
\end{align}
Then, from \cref{eq:MD},
\begin{align}
    J (\dem) - J (\piRE) &\le 2 \sup_{f \in \mathcal{F}} \left| J_f (\dem) - \widehat{E} (f) \right| \\
    &\le 2 \underbrace{ \sup_{f \in \mathcal{F}} \left| J_f^{(1)} (\dem) - \mathbb{E} \left[ \widehat{E}^{(1)} (f) \middle| D_1 \right] \right|}_{\text{(I)}} + 2\underbrace{\sup_{f \in \mathcal{F}} \left| \mathbb{E} \left[ \widehat{E}^{(1)} (f) \middle| D_1 \right] - \widehat{E}^{(1)} (f) \right|}_{\text{(II)}} \nonumber\\
    &\hspace{16em}+ 2 \underbrace{ \sup_{f \in \mathcal{F}}  \left| J_f^{(2)} (\dem) - \widehat{E}^{(2)} (f) \right|}_{\text{(III)}}. \label{eq:final}
\end{align}
where the last line follows by triangle inequality. We bound each of these terms in the next 3 lemmas, starting with $\text{(I)}$.

\begin{lemma} \label{lemma:B1}
\begin{align}
    \sup_{f \in \mathcal{F}} \left| J_f^{(1)} (\dem) - \mathbb{E} \left[ \widehat{E}^{(1)} (f) \middle| D_1 \right] \right| &\le H \sum_{h=1}^H \mathbb{E}_{\dem} \left[ \mathcal{M} (s_h,h) \textsf{TV} \left( \dem_h (\cdot|s_h) , \piBC_h (\cdot|s_h) \right) \right] \label{eq:17}
\end{align}
\end{lemma}
\begin{proof}
The proof of this result closely follows the supervised learning reduction of \BC (cf. \cite{pmlr-v9-ross10a}). Note that,
\begin{align} \label{eq:194}
    \mathbb{E} \left[ \widehat{E}^{(1)} (f) \middle| D_1 \right] - V^{(1)}_f (\dem) &= \sum_{t=1}^H \mathbb{E}_{\piBC} \left[ f_t (s_t,a_t) \mathcal{P} (s_{1 \cdots t-1}) \right] - \mathbb{E}_{\dem} \left[ f_t (s_t,a_t) \mathcal{P} (s_{1 \cdots t-1}) \right].
\end{align}
Define $\pi^{(h)}$ as the policy which plays $\dem$ until (and including) time $h$ and $\piBC$ after time $h$. Then, by cascading,
\begin{align}
    &\mathbb{E}_{\piBC} \left[ f_t (s_t,a_t) \mathcal{P} (s_{1 \cdots t-1}) \right] - \mathbb{E}_{\dem} \left[ f_t (s_t,a_t) \mathcal{P} (s_{1 \cdots t-1}) \right] \nonumber\\
    &= \sum_{h=0}^{t-1} \mathbb{E}_{\pi^{(h)}} \left[ f_t (s_t,a_t) \mathcal{P} (s_{1 \cdots t-1}) \right] - \mathbb{E}_{\pi^{(h+1)}} \left[ f_t (s_t,a_t) \mathcal{P} (s_{1 \cdots t-1}) \right] \label{eq:18}
\end{align}
Define, the uncertainty weighted state visitation measure $d^{\mathcal{M}}$ and the uncertainty weighted look-forward reward $\rho^{\mathcal{M}}$ as follows,
\begin{align}
    &d^{\mathcal{M}}_{h+1} (s') \triangleq \mathbb{E}_{\dem} \left[ \mathbbm{1} (s_{h+1} = s') \prod_{t'=1}^h \mathcal{M} (s_{t'},t') \right] \\
    &\rho^{\mathcal{M}}_{h+1} (s',a') \triangleq \mathbb{E}_{\piBC} \left[ f_t (s_t,a_t) \prod_{t'=h+2}^t \mathcal{M} (s_{t'},t') \middle| s_{h+1} = s', a_{h+1} = a' \right]
\end{align}
By decomposing expectations along trajectories, using the fact that $\mathcal{P} (s_{1 \cdots t-1}) = \prod_{t'=1}^H \mathcal{M} (s_{t'} , t')$ some simplification results in the following equation,
\begin{align}
    &\left| \mathbb{E}_{\pi^{(h)}} \left[ f_t (s_t,a_t) \mathcal{P} (s_{1 \cdots t-1}) \right] - \mathbb{E}_{\pi^{(h+1)}} \left[ f_t (s_t,a_t) \mathcal{P} (s_{1 \cdots t-1}) \right] \right| \\
    &= \left| \sum_{s' \in \mathcal{S}} \sum_{a' \in \mathcal{A}} d^{\mathcal{M}}_{h+1} (s') \mathcal{M} (s',h+1) \left( \dem_{h+1} (a'|s') - \piBC_{h+1} (a'|s') \right) \rho^{\mathcal{M}}_{h+1} (s',a') \right| \\
    &\overset{(i)}{\le} \sum_{s' \in \mathcal{S}} d^{\mathcal{M}}_{h+1} (s') \mathcal{M} (s',h+1) \textsf{TV} \left( \dem_{h+1} (\cdot|s') , \piBC_{h+1} (\cdot|s') \right) \\
    &= \mathbb{E}_{\dem} \left[ \mathcal{M} (s_{h+1},h+1) \textsf{TV} \left( \dem_{h+1} (\cdot|s_{h+1}) , \piBC_{h+1} (\cdot|s_{h+1}) \right) \right].
\end{align}
where $(i)$ uses the fact that the membership oracle is a function $\in [0,1]$ and $f$ is bounded and lies in the interval $[0,1]$ (which implies that $\rho^{\mathcal{M}}$ also lies in $[0,1]$ pointwise). Plugging into \cref{eq:18} and subsequently into \cref{eq:194} completes the proof.
\end{proof}

Next we bound the 3rd term, $\text{(III)}$. This follows by an application of Bernstein's inequality.

\begin{lemma} \label{lemma:B2}
With probability $\ge 1-\delta$,
\begin{align}
    \sup_{f \in \mathcal{F}}  \left| J_f^{(2)} (\dem) - \widehat{E}^{(2)} (f) \right| \le H \sqrt{ \frac{\log (\Fmax H/\delta) \sum_{t=1}^{H-1} \mathbb{E}_{\dem} \left[ 1 - \mathcal{M} (s_t, t) \right]}{\Nexp}} + \frac{H\log (\Fmax H/\delta)}{\Nexp}
\end{align}
\end{lemma}
\begin{proof}
First observe that,
\begin{align} \label{eq:19}
    J_f^{(2)} (\dem) - \widehat{V}^{(2)}_f &= \sum_{t=1}^H \mathbb{E}_{\textsf{tr} \sim \mathrm{Unif} (D_2)} \left[ f_t (s_t,a_t) \left(1 - \mathcal{P} (s_{1\cdots t-1}) \right) \right] - \mathbb{E}_{\dem} \left[ f_t (s_t,a_t) \left(1 - \mathcal{P} (s_{1\cdots t-1}) \right) \right]
\end{align}
For each $t$, note that $f_t (s_t,a_t) \left(1 - \mathcal{P} (s_{1\cdots t-1}) \right)$ is bounded in the range $[0,1]$. Therefore, invoking Bernstein's inequality, with probability $\ge 1-\delta$,
\begin{align}
    &\left| \mathbb{E}_{\text{Unif} (D_2)} \left[ f_t (s_t,a_t) \left(1 - \mathcal{P} (s_{1\cdots t-1}) \right) \right] - \mathbb{E}_{\dem} \left[ f_t (s_t,a_t) \left(1 - \mathcal{P} (s_{1\cdots t-1}) \right) \right] \right|\\
    &\lesssim \sqrt{ \frac{\text{Var}_{\dem} ( f_t (s_t,a_t) \left(1 - \mathcal{P} (s_{1\cdots t-1}) \right) ) \log (1/\delta)}{\Nexp}} + \frac{\log (1/\delta)}{\Nexp} \\
    &\le \sqrt{ \frac{\mathbb{E}_{\dem} [ ( f_t (s_t,a_t) \left(1 - \mathcal{P} (s_{1\cdots t-1}) \right) )^2 ] \log (1/\delta)}{\Nexp}} + \frac{\log (1/\delta)}{\Nexp} \\
    &\overset{(i)}{\le} \sqrt{ \frac{\mathbb{E}_{\dem} [ f_t (s_t,a_t) \left(1 - \mathcal{P} (s_{1\cdots t-1}) \right) ] \log (1/\delta)}{\Nexp}} + \frac{\log (1/\delta)}{\Nexp} \\
    &\overset{(ii)}{\le} \sqrt{ \frac{\mathbb{E}_{\dem} \left[ 1 - \mathcal{P} (s_{1\cdots t-1}) \right] \log (1/\delta)}{\Nexp}} + \frac{\log (1/\delta)}{\Nexp} \label{eq:M-bound}
\end{align}
where $(i)$ uses the fact that $f_t (s_t,a_t) \left(1 - \mathcal{P} (s_{1\cdots t-1}) \right)$ is bounded in the range $[0,1]$, and $(ii)$ uses the fact that $0 \le f_t (s_t,a_t) \le 1$. Assuming $0 \le x_i \le 1$ for all $i \in [n]$, we have the inequality,
\begin{equation} \label{eq:simpl}
    1 - \prod_{i=1}^n x_i \le \sum_{i=1}^n 1 - x_i
\end{equation}
Applying this to \cref{eq:M-bound} for $1 - \mathcal{P} (s_{1\cdots t-1}) = 1 - \prod_{t'=1}^{t-1} \mathcal{M} (s_{t'}, t')$, we have,
\begin{align}
    &\left| \mathbb{E}_{\text{Unif} (D_2)} \left[ f_t (s_t,a_t) \left(1 - \mathcal{P} (s_{1\cdots t-1}) \right) \right] - \mathbb{E}_{\dem} \left[ f_t (s_t,a_t) \left(1 - \mathcal{P} (s_{1\cdots t-1}) \right) \right] \right| \\
    &\le \sqrt{ \frac{\sum_{t' = 1}^{t-1} \mathbb{E}_{\dem} \left[ 1 - \mathcal{M} (s_{t'}, t') \right] \log (1/\delta)}{\Nexp}} + \frac{\log (1/\delta)}{\Nexp} \label{eq:pre-unionbound}
\end{align}

Therefore, by union bounding, with probability $\ge 1 - \delta/H$, simultaneously for every $f_t \in \mathcal{F}_t$,
\begin{align}
    &\left| \mathbb{E}_{\text{Unif} (D_2)} \left[ f_t (s_t,a_t) \left(1 - \mathcal{P} (s_{1\cdots t-1}) \right) \right] - \mathbb{E}_{\dem} \left[ f_t (s_t,a_t) \left(1 - \mathcal{P} (s_{1\cdots t-1}) \right) \right] \right| \\
    &\lesssim \sqrt{ \frac{\log (|\mathcal{F}_t| H/\delta) \sum_{t' = 1}^{t-1} \mathbb{E}_{\dem} \left[ 1 - \mathcal{M} (s_{t'}, t') \right]}{\Nexp}} + \frac{\log (|\mathcal{F}_t| H/\delta)}{\Nexp}. \label{eq:in-effect}
\end{align}
This implies that the maximum over $f_t$ of the LHS is upper bounded by the RHS. Union bounding over $t = 1,\cdots,H$ and plugging into \cref{eq:19}, we have that with probability $\ge 1 - \delta$,
\begin{align}
    &\sup_{f \in \mathcal{F}} \left| J_f^{(2)} (\dem) - \widehat{V}^{(2)}_f \right| \\
    &\le \sum_{t=1}^H \left| \mathbb{E}_{\text{Unif} (D_2)} \left[ f_t (s_t,a_t) \left(1 - \mathcal{P} (s_{1\cdots t-1}) \right) \right] - \mathbb{E}_{\dem} \left[ f_t (s_t,a_t) \left(1 - \mathcal{P} (s_{1\cdots t-1}) \right) \right] \right| \label{eq:V2prefinal}\\
    &\lesssim H \sqrt{ \frac{\log (\Fmax H/\delta) \sum_{t=1}^{H-1} \mathbb{E}_{\dem} \left[ 1 - \mathcal{M} (s_t, t) \right]}{\Nexp}} + \frac{H\log (\Fmax H/\delta)}{\Nexp}.
\end{align}
\end{proof}

\begin{lemma} \label{lemma:B3}
With probability $\ge 1 - \delta$,
\begin{align}
    \sup_{f \in \mathcal{F}} \left| \mathbb{E} \left[ \widehat{V}_f^{(1)} \middle| D_1 \right] - \widehat{V}_f^{(1)} \right| \lesssim H \sqrt{ \frac{\log (\Fmax H/\delta)}{\Nrep}} + \frac{H\log (\Fmax H/\delta)}{\Nrep}.
\end{align}
\end{lemma}
\begin{proof}
The proof follows essentially the same structure as \Cref{lemma:B2} by decomposing $\widehat{V}_f^{(1)}$ into a sum of $H$ terms of the form $f_t (s_t,a_t) \mathcal{P} (s_{1 \cdots t-1})$, applying Bernstein's inequality to bound the deviation of each term from its mean and finally union bounding over the rewards $f_t \in \mathcal{F}_t$ to get the uniform bound over all discriminators $f \in \mathcal{F}$.
\end{proof}

Putting together \Cref{lemma:B1,lemma:B2,lemma:B3} with \cref{eq:final} completes the proof of \Cref{thm:main}.

\subsection{Bound in the tabular setting (\Cref{thm:RE-tabular-ub})} \label{sec:RE-tabular-ub}

In this section, we provide an upper bound on the imitation gap of \RE in the tabular setting when the expert is a deterministic policy. This recovers the bound on the imitation gap for \RE proved in \cite{rajaraman2020toward}.

\begin{theorem} \label{thm:RE-tabular-ub}
Consider an appropriately initialized version of \RE, and let the size of the replay dataset $\Nrep \to \infty$. For any tabular IL instance with $H \ge 10$, with probability $\ge 1 - 3\delta$,
\begin{equation}
    J(\dem) - J(\piRE) \lesssim \min \left\{ \frac{|\mathcal{S}| H^{3/2}}{\Nexp} , H \sqrt{\frac{|\mathcal{S}|}{\Nexp}} \right\} \log \left( \frac{|\mathcal{S}| H}{\delta} \right).
\end{equation}
\end{theorem}

Below we describe the implementation of \RE corresponding to \Cref{thm:RE-tabular-ub} in more detail.

The membership oracle we use in this setting for \RE is defined below,
\begin{equation} \label{eq:tabular}
    \mathcal{M} (s,t) = \begin{cases}
    1 \quad &\text{if } s \text{ is visited in } D_1 \text{ at time } t\\
    0 &\text{otherwise.}
    \end{cases}
\end{equation}
The function class $\mathcal{F}$ which we use is identical to that for empirical moment matching, which is described in \Cref{rem:1}.

Note that in the tabular setting, \BC simply mimics the deterministic expert's actions at states visited in the dataset $D_1$ and plays an arbitrary deterministic action on the remaining states. As a consequence of this definition, if $\mathcal{M} (s,t) = 1 \iff \piBC_t (\cdot|s) = \dem_t (\cdot|s)$ and $\mathcal{M} (s,t) = 0$ otherwise. We instantiate the family of discriminators as in \Cref{rem:1}, as $\mathcal{F} = \bigoplus_{t=1}^H \{ f_t : \| f_t \|_\infty \le 1 \}$ and the set of policies $\Pi$ optimized over is chosen as the set of all deterministic policies. While the guarantee of \Cref{thm:main} depends on $\Fmax = \max_{t \in [H]} |\mathcal{F}_t|$ which is unbounded (or $\exp (|\mathcal{S}| |\mathcal{A}|)$ by using a discretization of the reward space), note that we can improve the guarantee to effectively have $\Fmax \approx \exp (|\mathcal{S}|)$ noting the structure of the set of discriminators. Looking into the proof of \Cref{thm:main} we bring out this dependence below. We note that there are many ways of bringing out this dependence, including a careful net argument directly on top of the guarantee of \Cref{thm:main}. We simply present one such argument below. 

The critical step where the finiteness of the set of discriminators $\mathcal{F}$ is used, is in union bounding the gap between the population and the empirical estimate of $f_t (s_t,a_t) \left( 1 - \mathcal{P} (s_{1 \cdots t-1}) \right)$ in \cref{eq:pre-unionbound}.
\begin{align}
    &\left| \mathbb{E}_{\text{Unif} (D_2)} \left[ f_t (s_t,a_t) \left(1 - \mathcal{P} (s_{1\cdots t-1}) \right) \right] - \mathbb{E}_{\dem} \left[ f_t (s_t,a_t) \left(1 - \mathcal{P} (s_{1\cdots t-1}) \right) \right] \right| \label{eq:pre-ub}
\end{align}
In the next step of the proof of \Cref{thm:main}, we union bound over all $f_t \in \mathcal{F}_t$. However, note that for $\mathcal{F}_t = \{ f_t : \| f_t \|_\infty \le 1 \}$, we have that,
\begin{align}
    &\sup_{f_t : \| f_t \|_\infty \le 1} \left| \mathbb{E}_{\text{Unif} (D_2)} \left[ f_t (s_t,a_t) \left(1 - \mathcal{P} (s_{1\cdots t-1}) \right) \right] - \mathbb{E}_{\dem} \left[ f_t (s_t,a_t) \left(1 - \mathcal{P} (s_{1\cdots t-1}) \right) \right] \right| \\
    &\overset{(i)}{=} \sum_{s \in \mathcal{S}} \sum_{a \in \mathcal{A}} \left| \mathbb{E}_{\text{Unif} (D_2)} \left[ \mathbb{I} (s_t=s,a_t =a) \left(1 - \mathcal{P} (s_{1\cdots t-1}) \right) \right] - \mathbb{E}_{\dem} \left[ \mathbb{I} (s_t=s,a_t=a ) \left(1 - \mathcal{P} (s_{1\cdots t-1}) \right) \right] \right| \\
    &\overset{(ii)}{=} \sum_{s \in \mathcal{S}} \left| \mathbb{E}_{\text{Unif} (D_2)} \left[ \mathbb{I} (s_t=s) \left(1 - \mathcal{P} (s_{1\cdots t-1}) \right) \right] - \mathbb{E}_{\dem} \left[ \mathbb{I} (s_t=s ) \left(1 - \mathcal{P} (s_{1\cdots t-1}) \right) \right] \right| \\
    &\overset{(iii)}{\le} \sum_{s \in \mathcal{S}} \left| \mathbb{E}_{\text{Unif} (D_2)} \left[ \mathbb{I} (s_t=s) \left(1 - \mathcal{P} (s_{1\cdots t-1}) \right) \right] - \mathbb{E}_{\dem} \left[ \mathbb{I} (s_t=s ) \left(1 - \mathcal{P} (s_{1\cdots t-1}) \right) \right] \right| \label{eq:010123299}
\end{align}
where $(i)$ follows similar to the equivalence between the variational representation of TV distance ($\tv (P,Q) = \frac{1}{2}\sup_{f : \| f \|_\infty \le 1} \mathbb{E}_P [f] - \mathbb{E}_Q [f]$) and the relationship to the $L_1$ distance, $\tv (P,Q) = \frac{1}{2} L_1 (P, Q)$. On the other hand, $(ii)$ follows by noting that the expert is a deterministic policy (and $D_2$ is generated by rolling out $\dem$). $(iii)$ follows by an application of Holder's inequality. By subgaussian concentration, for each $s \in \mathcal{S}$, with probability $\ge 1 - \frac{\delta}{|\mathcal{S}|H}$,
\begin{align}
    &\left| \mathbb{E}_{\text{Unif} (D_2)} \left[ \mathbb{I} (s_t=s) \left(1 - \mathcal{P} (s_{1\cdots t-1}) \right) \right] - \mathbb{E}_{\dem} \left[ \mathbb{I} (s_t=s ) \left(1 - \mathcal{P} (s_{1\cdots t-1}) \right) \right] \right| \\
    &\lesssim \sqrt{\frac{\text{Var}_{\dem} \left( \mathbb{I} (s_t = s) \left(1 - \mathcal{P} (s_{1 \cdots t-1}) \right) \right) \log \left(\frac{|\mathcal{S}|H}{\delta} \right)}{|D_2|}} + \frac{\log \left(\frac{|\mathcal{S}|H}{\delta} \right)}{|D_2|} \\
    &\overset{(i)}{\le} \sqrt{\frac{\mathbb{E}_{\dem} \left[ \mathbb{I} (s_t = s) \left(1 - \mathcal{P} (s_{1 \cdots t-1}) \right) \right] \log \left(\frac{|\mathcal{S}|H}{\delta} \right)}{|D_2|}} + \frac{\log \left(\frac{|\mathcal{S}|H}{\delta} \right)}{|D_2|}
\end{align}
where $(i)$ uses the fact that $0 \le \mathbb{I} (s_t = s) \left(1 - \mathcal{P} (s_{1 \cdots t-1}) \right) \le 1$. Combining with \cref{eq:010123299}, union bounding and applying Cauchy Schwarz inequality, with probability $\ge 1 - \frac{\delta}{H}$,
\begin{align}
    &\sup_{f_t : \| f_t \|_\infty \le 1} \left| \mathbb{E}_{\text{Unif} (D_2)} \left[ f_t (s_t,a_t) \left(1 - \mathcal{P} (s_{1\cdots t-1}) \right) \right] - \mathbb{E}_{\dem} \left[ f_t (s_t,a_t) \left(1 - \mathcal{P} (s_{1\cdots t-1}) \right) \right] \right| \\
    &\lesssim \sqrt{|\mathcal{S}|} \sqrt{\sum_{s \in \mathcal{S}} \frac{\mathbb{E}_{\dem} \left[ \mathbb{I} (s_t = s) \left(1 - \mathcal{P} (s_{1 \cdots t-1}) \right) \right] \log \left(\frac{|\mathcal{S}|H}{\delta} \right)}{|D_2|}} + \frac{|\mathcal{S}| \log \left(\frac{|\mathcal{S}|H}{\delta} \right)}{|D_2|} \\
    &= \sqrt{|\mathcal{S}|} \sqrt{\frac{\mathbb{E}_{\dem} \left[ 1 - \mathcal{P} (s_{1 \cdots t-1}) \right] \log \left(\frac{|\mathcal{S}|H}{\delta} \right)}{|D_2|}} + \frac{|\mathcal{S}| \log \left(\frac{|\mathcal{S}|H}{\delta} \right)}{|D_2|} \\
    &\overset{(i)}{\le} \min \left\{ \sqrt{\frac{|\mathcal{S}| \log \left( \frac{|\mathcal{S}| H}{\delta} \right)}{|D_2|}} , \sqrt{ |\mathcal{S}|\frac{\sum_{t=1}^{H-1}\mathbb{E}_{\dem} \left[ 1 - \mathcal{M} (s_t,t) \right] \log \left(\frac{|\mathcal{S}|H}{\delta} \right)}{|D_2|}} \right\} + \frac{|\mathcal{S}| \log \left(\frac{|\mathcal{S}|H}{\delta} \right)}{|D_2|} \label{eq:in-effect2}
\end{align}
where $(i)$ follows by the same simplification as in \cref{eq:simpl}. Comparing with \cref{eq:in-effect}, this roughly corresponds to setting $\Fmax \approx \exp (|\mathcal{S}|)$. All in all, summing \cref{eq:in-effect2} over $t \in [H]$ and plugging into \cref{eq:V2prefinal}, with probability $\ge 1 - \delta$,
\begin{align}
    &\sup_{f \in \mathcal{F}} \left| J_f^{(2)} (\dem) - \widehat{V}^{(2)}_f \right| \\
    &\lesssim H \left\{ \sqrt{\frac{|\mathcal{S}| \log \left( \frac{|\mathcal{S}| H}{\delta} \right)}{|D_2|}} ,  \sqrt{|\mathcal{S}|\frac{\sum_{t=1}^{H-1}\mathbb{E}_{\dem} \left[ 1 - \mathcal{M} (s_t,t) \right] \log \left(\frac{|\mathcal{S}|H}{\delta} \right)}{|D_2|}} \right\} + \frac{|\mathcal{S}| \log \left(\frac{|\mathcal{S}|H}{\delta} \right)}{|D_2|}
\end{align}
Finally, we plug this into \cref{eq:final}, which is restated below,
\begin{align}
    J (\dem) - J (\piRE) &\le 2 \sup_{f \in \mathcal{F}} \left| J_f (\dem) - \widehat{E} (f) \right| \\
    &\le 2 \underbrace{ \sup_{f \in \mathcal{F}} \left| J_f^{(1)} (\dem) - \mathbb{E} \left[ \widehat{E}^{(1)} (f) \middle| D_1 \right] \right|}_{\text{(I)}} + 2\underbrace{\sup_{f \in \mathcal{F}} \left| \mathbb{E} \left[ \widehat{E}^{(1)} (f) \middle| D_1 \right] - \widehat{E}^{(1)} (f) \right|}_{\text{(II)}} \\
    &\hspace{16em}+ 2 \underbrace{ \sup_{f \in \mathcal{F}}  \left| J_f^{(2)} (\dem) - \widehat{E}^{(2)} (f) \right|}_{\text{(III)}}.
\end{align}
For the chosen membership oracle in \cref{eq:tabular}, the term (I) is $0$, since by \Cref{lemma:B1} it is upper bounded by $H \sum_{h=1}^H \mathbb{E}_{\dem} \left[ \mathcal{M} (s_h,h) \textsf{TV} \left( \dem_h (\cdot|s_h) , \piBC_h (\cdot|s_h) \right) \right]$. This is equal to $0$ since $\mathcal{M} (s,t) = 0$ wherever $\dem_t (\cdot|s) \ne \piBC_t (\cdot|s)$. On the other hand, $\Nrep \to \infty$ ensures that the term (III) goes to $0$ by the strong law of large numbers. Therefore, with probability $\ge 1 - 2\delta$,
\begin{align}
    &J(\dem) - J(\piRE) \\
    &\le 2\sup_{f \in \mathcal{F}} \left| \mathbb{E} \left[ \widehat{E}^{(1)} (f) \middle| D_1 \right] - \widehat{E}^{(1)} (f) \right|\\
    &\lesssim H \left\{ \sqrt{\frac{|\mathcal{S}| \log \left( \frac{|\mathcal{S}| H}{\delta} \right)}{|D_2|}} , \sqrt{|\mathcal{S}|\frac{\sum_{t=1}^{H-1}\mathbb{E}_{\dem} \left[ 1 - \mathcal{M} (s_t,t) \right] \log \left(\frac{|\mathcal{S}|H}{\delta} \right)}{|D_2|}} \right\} + \frac{|\mathcal{S}| H \log \left(\frac{|\mathcal{S}|H}{\delta} \right)}{|D_2|} \label{eq:final2}
\end{align}

Finally, we bound $\mathbb{E}_{\dem} [1 - \mathcal{M} (s_t,t)]$ for the membership oracle defined in \cref{eq:tabular}. By definition, this quantity is the same as $\mathrm{Pr}_{\dem} \left( s_t \text{ not visited in } D_1 \text{ at time t } \right)$. This is the probability that given $\Nexp$ samples from a distribution (the state visited at time $t$ in an expert rollout), the probability that a new sample from the same distribution is not in the support of the observed samples. This is known as the missing mass \cite{GT-MM-conc}. In Lemma A.3~\cite{rajaraman2020toward} it is shown that with probability $\ge 1 - \delta$,
\begin{align}
\sum_{t=1}^{H-1} \mathrm{Pr}_{\dem} \left( s_t \text{ not visited in } D_1 \text{ at time t } \right) \lesssim \frac{|\mathcal{S}| H}{|D_1|} + \frac{\sqrt{|\mathcal{S}|} H \log \left( \frac{|\mathcal{S}| H}{\delta}\right)}{|D_1|}
\end{align}
Finally, combining with \cref{eq:final2} and using the fact that that $|D_1|, |D_2| = \Theta (\Nexp)$, with probability $\ge 1 - 3 \delta$,
\begin{align}
    J(\dem) - J(\piRE) &\lesssim \min \left\{ H \sqrt{\frac{|\mathcal{S}| \log \left( \frac{|\mathcal{S}| H}{\delta} \right)}{\Nexp}},  \frac{|\mathcal{S}| H^{3/2} }{\Nexp} \log \left( \frac{|\mathcal{S}| H }{\delta} \right) \right\}.
\end{align}
This completes the proof of \Cref{thm:RE-tabular-ub}.

\subsection{Bound with parametric function approximation under Lipschitzness} \label{sec:RE-lipschitz-ub}

In this section, we provide an upper bound on the imitation gap of \RE in the presence of parametric function approximation under a Lipschitzness assumption on the function classes, and assuming access to a parameter estimation oracle for offline classification.

\begin{definition*}[IL with function-approximation]
In this setting, for each $t \in [H]$, there is a parameter class $\Theta_t \subseteq \mathbb{B}_2^d$, the unit $L_2$ ball in $d$ dimensions, and an associated function class $\{ f_{\theta_t} : \theta_t \in \Theta_t \}$. For each $t \in [H]$ there exists an unknown $\theta^E_t \in \Theta_t$ such that $\forall s \in \mathcal{S}$,
\begin{align}
    \pi^E_t (s) = \argmax_{a \in \mathcal{A}} f_{\theta^E_t} (s,a).
\end{align}
\end{definition*}

\begin{definition}[Policy induced by a classifier] \label{def:pilc-f}
Consider a set of parameters $\theta = \{ \theta_1,\cdots,\theta_H\}$ where $\theta_t \in \Theta_t$ for each $t$. A policy $\pi^\theta$ is said to be induced by the set of classifiers defined by $\theta$ if for all $s \in \mathcal{S}$ and $t \in [H]$,
\begin{equation}
    \pi^\theta_t (s) = \argmax_{a \in \mathcal{A}} f_{\theta_t}  (s,a).
\end{equation}
By this definition, $\dem = \pi^{\theta^E}$ where $\theta^E = \{ \theta^E_1,\cdots, \theta^E_H \}$.
\end{definition}

\begin{definition}[Lipschitz parameterization]
A function class $\mathcal{G} = \{ g_{\theta} : \theta \in \Theta \}$ where $g_{\theta} (\cdot) : \mathcal{X} \to \mathbb{R}$ is said to satisfy $L$-Lipschitz parameterization if, $\| g_{\theta} (\cdot) - g_{\theta'} (\cdot) \|_\infty \le L \| \theta - \theta' \|_2$ for all $\theta,\theta' \in \Theta$. In other words, for each $x \in \mathcal{X}$, $g_{\theta} (x)$ is an $L$-Lipschitz function in $\theta$, in the $L_2$ norm.
\end{definition}

\begin{assumption}[\Cref{A:lip:original} restated] \label{A:lip}
For each $t$, the class $\{ f_{\theta_t} : \theta_t \in \Theta_t \}$ is $L$-Lipschitz in its parameterization, $\theta_t \in \Theta_t$.
\end{assumption}

To deal with parametric function approximation, we assume that the learner has access to a proper offline classification oracle, which given a dataset of classification examples, guarantees to approximately return the underlying ground truth parameter. Namely,

\begin{assumption}[\Cref{A:classorc} restated] \label{A:class}
We assume that the learner has access to a multi-class classification oracle, which given $n$ examples of the form, $(s^i, a^i)$ where $s^i \overset{\text{i.i.d.}}{\sim} \mathcal{D}$ and $a^i = \argmax_{a \in \mathcal{A}} f_{\theta^*} (s^i,a)$, returns a $\hat{\theta} \in \Theta$ such that, with probability $\ge 1 - \delta$, $\| \hat{\theta} - \theta^* \|_2 \le \mathcal{E}_{\Theta,n,\delta}$.
\end{assumption}
This assumption implies that the parameter class $\Theta_t$ (and the associated function class $\{ f_{\theta_t} : \theta_t \in \Theta_t \}$) admits finite sample complexity guarantees for learning the parameter $\theta_t^*$ given classification examples from the underlying ground truth function $f_{\theta_t^*}$. As we discuss in more detail later, we will assume that this classification oracle is used by \RE to train the \BC policy in Line 3 of \Cref{alg:re}. 

Finally, we introduce the main assumption on the IL instances we study. We assume that the classification problems solved by \BC at each $t \in [H]$ satisfy a margin condition.

\begin{assumption}[\Cref{A:wm} restated] \label{A:wm:repeat}
For $\theta \in \Theta_t$, define $a^\theta_s = \argmax_{a \in \mathcal{A}} f_{\theta} (s,a)$ as the classifier output on the state $s$. The weak margin condition assumes that for each $t$, there is no classifier $\theta \in \Theta_t$ such that for a large mass of states, $f_{\theta} (s_t,a^\theta_{s_t}) - \max_{a \ne a^\theta_{s_t} } f_{\theta} (s_t,a)$, i.e. the ``margin'' from the nearest classification boundary, is small. Formally, the weak-margin condition with parameter $\mu$ states that, for any $\theta \in \Theta_t$ and $\eta \le 1/\mu$,
\begin{align} \label{eq:wmc-rep}
    \mathrm{Pr}_{\dem} \left( f_{\theta} (s_t, a^\theta_{s_t}) - \max_{a \ne a^\theta_{s_t} } f_{\theta} (s_t,a) \ge \eta \right) \ge e^{- \mu \eta}.
\end{align}
The weak margin condition only assumes that there is at least an exponentially small (in $\eta$) mass of states with margin at least $\eta$. Smaller $\mu$ indicates a larger mass away from any decision boundary. It suffices to assume that \cref{eq:wmc-rep} is only true for $\theta$ as the classifier in \Cref{A:class} for our guarantees (\Cref{thm:gfa:rep}) to hold.
\end{assumption}

\begin{remark}
Note that the weak margin condition is the multi-class extension of the Tsybakov margin condition of \cite{mammen-tsybakov,audibert-tsybakov} defined for the binary case. In particular, in \cref{eq:wmc-rep}, we may replace the RHS by $1 - \mu \eta$, or $1 - (\mu \eta)^\alpha$ for $\alpha > 0$ to get different analogs of the margin condition and the main guarantee, \Cref{thm:gfa}, as we discuss in \Cref{sec:RE-lipschitz-ub}.
\end{remark}

\begin{theorem} \label{thm:gfa:rep}
For IL with parametric function approximation, under Assumptions~\ref{A:lip} to \ref{A:wm:repeat}, appropriately instantiating \RE ensures that with probability $\ge 1 - 4\delta$,
\begin{align}
    J(\dem) - J(\piRE) \lesssim H^{3/2} \sqrt{ \frac{\mu  L \log \left( F_{\max} H/\delta \right)}{\Nexp} \frac{\sum_{t=1}^H \eNtH}{H}} + \frac{\log \left( F_{\max} H /\delta \right)}{\Nexp}.
\end{align}
Note that we assume the same conditions on $\mathcal{F}$ as required in \Cref{thm:main}.
\end{theorem}

\begin{remark}
The classification oracle in \Cref{A:class} asks for a stronger condition than just finding a classifier with small generalization error, which need not be close to the ground truth $\theta^*$ in the parameter space. Learning classifiers with small generalization error is studied in \cite{DBLP:journals/corr/DanielySBS13} who show that the \textit{Natarajan dimension}, up to log-factors in the number of classes (i.e. number of actions) is the right statistical complexity measure which characterizes the generalization error of the best learner. In the realizable case, the optimal generalization error guarantee scales as $\widetilde{O} (1/n)$ where $n$ is the number of classification examples. Under certain assumptions on the input distribution $\mathcal{D}$ and the function family (e.g. for linear families), we later show that the generalization error guarantee can be extended to approximately learning the parameter as well (up to problem dependent constants). Generally, under two conditions,
\begin{enumerate}
    \item Generalization error guarantees which scale as $\widetilde{O} (1/\Nexp)$ can be extended to parameter learning,
    \item The dependence of the generalization error on the failure probability $\delta$ scales as $\text{polylog} (1/\delta)$,
\end{enumerate}
the optimality gap for \RE which we prove in \Cref{thm:gfa:rep} scales as $\widetilde{O} \left( H^{3/2} / \Nexp \right)$. The constants in $\widetilde{O} (\cdot)$ here depend on the Natarajan dimension and the covering number of the function classes $\mathcal{F}_t$ among problem dependent constants.
\end{remark}

In proving this result, we first discuss the implementation of \RE.

\paragraph{Implementation of \RE (\Cref{alg:re})} We discuss the instantiation of \RE in the Lipschitz setting below. The underlying function class $\mathcal{F}$ is chosen arbitrarily (note that the guarantee we prove depends on this function class, and the only constraints on $\mathcal{F}$ are those in \Cref{thm:main} - the ground truth reward must belong in $\mathcal{F} = \otimes_{t=1}^H \mathcal{F}_t$, the function class is symmetric, i.e., $f_t \in \mathcal{F}_t \iff -f_t \in \mathcal{F}_t$ for each $t$ and for all $f_t \in \mathcal{F}_t$, $\| f_t \|_\infty \le 1$) This requires specifying the choice of the membership oracle $\mathcal{M}$ and describing the instantiation of \BC.

\paragraph{Implementation of \BC:} Recall that in \Cref{alg:re}, the learner trains \BC on the dataset $D_1$. In particular, under the offline classification oracle condition, \Cref{A:class}, the learner trains $H$ classifiers, one for each $t$, trained on the state-action pairs (i.e. state is the input, and the action at this state is the corresponding class) observed in the expert dataset at time $t$ using the offline classifier in \Cref{A:class}. We assume that each classifier is trained with the failure probability chosen as $\delta/H$. Denoting this set of $H$ classifiers as $\thetaBC = \left\{ \thetaBC_1, \cdots, \thetaBC_H \right\}$, this corresponds to the to the policy $\piBC = \pi^{\thetaBC}$ induced by the classifier $\thetaBC$ (\Cref{def:pilc-f}).

In particular, by union bounding, the classifiers $\thetaBC$ satisfy with probability $\ge 1-\delta$ simultaneously for each time $t \in [H]$,
\begin{equation} \label{eq:class-f}
    \| \theta^E_t - \hat{\theta}_t \|_2 \le \eNtH.
\end{equation}

\paragraph{Membership oracle:}
Fix a time-step $t \in [H]$. The membership oracle $\mathcal{M}$ is defined in \cref{eq:memdef} as,
\begin{align}
    \mathcal{M} (s,t) = \begin{cases}
    +1 \qquad & \text{if } \exists a \in \mathcal{A} \text{ such that, } \forall a' \in \mathcal{A},\ f_{\thetaBC_t} (s,a) - f_{\thetaBC_t} (s,a') \ge 2 L \eNtH \\
    0 &\text{otherwise.}
    \end{cases}
\end{align}

We first show that on the states such that the membership oracle is $1$, the expert policy perfectly matches the learner's policy.

\begin{lemma}
At every state $s$ such that $\mathcal{M} (s,t) = +1$, $\dem_t (s) = \piBC_t (s)$.
\end{lemma}
\begin{proof}
Note that $\theta^E_t$ satisfies $\| \theta_t^E - \thetaBC_t \|_2 \le \eNtH$ with probability $1-\delta$. Consider the action $a$ played by the learner, for any $a' \in \mathcal{A}$,
\begin{align}
    f_{\theta^E_t} (s,a) - f_{\theta^E_t} (s,a') &\ge f_{\thetaBC_t} (s,a) - \eNtH L - f_{\thetaBC_t} (s,a') - \eNtH L \\
    &\ge 0
\end{align}
where the first inequality follows by Lipschitzness of $f_\cdot (s,a)$ and the last inequality follows by definition of the set of states where $\mathcal{M} (s,t) = +1$: $\forall a' \in \mathcal{A},\ f_{\thetaBC_t} (s,a) - f_{\thetaBC_t} (s,a') \ge 2 \eNtH L$.

Since for this action $a$, $f_{\theta_t^E} (s,a) - f_{\theta_t^E} (s,a') \ge 0$ for all other actions $a' \in \mathcal{A}$, $a$ must be the action played by the expert policy. This completes the proof.
\end{proof}

Note that $\piBC$ always matches $\pi^E$ wherever the membership oracle $\mathcal{M}$ is non-zero. We run \Cref{alg:re}. Therefore, from \Cref{thm:main}, with probability $\ge 1 - 4\delta$, the imitation gap of the learner is bounded by,
\begin{align} \label{eq:discbound}
    J_r (\dem) - J_r (\piRE) \lesssim H^{3/2} \sqrt{ \frac{\log \left( F_{\max} H /\delta \right)}{\Nexp} \frac{\sum_{t=1}^H \mathbb{E}_{\dem} \left[ 1 - \mathcal{M} (s_t, t) \right]}{H}} + \frac{\log \left(F_{\max} H/\delta \right)}{\Nexp}.
\end{align}

To complete the proof, we must bound $\mathbb{E}_{\dem} \left[ 1 - \mathcal{M} (s_t, t) \right]$, which is the measure of states $s$ such that $\forall a \in \mathcal{A}, \exists a' \in \mathcal{A} : f_{\hat{\theta}_t} (s,a) - f_{\hat{\theta}_t} (s,a') \le 2 L \eNtH$, i.e. the mass of states which are very close to a decision boundary. The probability of this set of states is upper bounded by the weak margin condition. Indeed, for each $t \in [H]$, defining $a^*_s = \argmax_{a \in \mathcal{A}} f_{\thetaBC_t} (s,a)$,
\begin{align}
    \mathrm{Pr}_{\dem} \left( f_{\thetaBC_t} (s_t,a^*_{s_t}) - \max_{a \ne a^*_{s_t} } f_{\thetaBC_t} (s_t,a) \ge 2 L \eNtH \right) &\ge e^{- \mu L \eNtH} \\
    &\ge 1 - \mu L \eNtH. \label{eq:margin-instance}
\end{align}
Therefore,
\begin{align}
    \mathbb{E}_{\dem} \left[ 1 - \mathcal{M} (s_t, t) \right] \lesssim \mu L \eNtH.
\end{align}
Putting it together with \cref{eq:discbound}, and simplifying, with probability $\ge 1 - 4\delta$,
\begin{align}
    &J (\dem) - J (\piRE) \lesssim H^{3/2} \sqrt{ \frac{\mu  L \log \left( F_{\max} H/\delta \right)}{\Nexp} \frac{\sum_{t=1}^H \eNtH}{H}} + \frac{\log \left( F_{\max} H /\delta \right)}{\Nexp}.
\end{align}

Note that in \Cref{eq:margin-instance}, we only use the fact that the probability mass of states which are $\eta$-close to any decision boundary is not too high. Similar to \cite{audibert-tsybakov}, we may consider relaxations of the weak margin condition, as below.

\begin{assumption}[$\alpha$-weak margin condition] \label{A:alphawm}
Consider any $t \in [H]$ and $\theta \in \Theta_t$. For each $s \in \mathcal{S}$, define $a^*_s = \argmax_{a \in \mathcal{A}} f_{\theta} (s,a)$ as the classifier output under $f_\theta$. The $\alpha$ weak margin condition with parameter $\mu$ assumes that, for any $\eta \le 1/\mu$,
\begin{align} \label{eq:alphawmc}
    \forall \theta \in \Theta_t,\ \ \mathrm{Pr}_{\dem} \left( f_{\theta} (s_t,a^*_{s_t}) - \max_{a \ne a^*_{s_t} } f_{\theta} (s_t,a) \ge \eta \right) \ge 1-(\mu \eta)^\alpha.
\end{align}
When $\alpha = 1$, this condition is effectively equivalent to the weak margin condition in \Cref{A:wm:repeat}.
\end{assumption}

Following the proof of \Cref{thm:gfa:rep}, we may obtain the following result under the $\alpha$ weak margin condition for $\alpha \ne 1$.

\begin{theorem} \label{thm:alpha-gfa:rep}
For IL with parametric function approximation, under Assumptions~\ref{A:lip}, \ref{A:class} and \ref{A:alphawm}, appropriately instatiating \RE ensures that with probability $\ge 1 - 4\delta$,
\begin{align} \label{eq:alpha-gfa:rep}
    J(\dem) - J(\piRE) \lesssim H^{3/2} \sqrt{ \frac{(\mu  L)^\alpha \log \left( F_{\max} H/\delta \right)}{\Nexp} \frac{\sum_{t=1}^H (\eNtH)^\alpha}{H}} + \frac{\log \left( F_{\max} H /\delta \right)}{\Nexp}.\noindent
\end{align}
Once again, we assume the same conditions on $\mathcal{F}$ as required in \Cref{thm:main}.
\end{theorem}

\subsection{Extension to unbounded discriminator families}

Note that when the family of discriminators $\mathcal{F}$ does not have finite cardinality, it in fact suffices to just bound the imitation gap against a finite covering of $\mathcal{F}$. We spell out the details explicitly below.

In particular, we can replace $\Fmax$ by $\max_{t \in [H]} \mathcal{N} (\mathcal{F}_t, 1/\Nexp, \| \cdot \|_\infty)$, where $\mathcal{N} (\mathcal{G}, \xi, \| \cdot \|)$ denotes the covering number of $\mathcal{G}$ in the norm $\| \cdot \|$ as defined below.

\begin{definition}[Covering number] \label{def:covnum} For a function class $\mathcal{G}$, tolerance $\xi$ and norm $\| \cdot \|$, the covering number $\mathcal{N} (\mathcal{G}, \xi, \| \cdot \|)$ is defined as the cardinality of the smallest set of functions $\mathcal{G}^\xi$ such that for each $g \in \mathcal{G}$, there exists a $g' \in \mathcal{G}^\xi$,
\begin{align}
    \| g - g' \| \le \xi.
\end{align}
\end{definition}

\begin{corollary} \label{corr:linearcov}
When $\mathcal{G}$ is chosen as the set of $1$-bounded linear functions, $\mathcal{G} = \{ \{ \langle x, \theta \rangle : x \in \mathbb{B}_2^d \} : \theta \in \mathbb{B}_2^{d} \}$, where $\mathbb{B}_2^d$ denotes the $L_2$ unit ball in $\mathbb{R}^d$, $\mathcal{N} (\mathcal{G}, \xi, \| \cdot \|_\infty) \le \left( \frac{2\sqrt{d}}{\xi} + 1\right)^d$.
\end{corollary}
\begin{proof}
For any $g,g' \in \mathcal{G}$, where $g$ and $g'$ correspond to parameters $\theta,\theta' \in \mathbb{B}_2^d$,
\begin{align}
    \| g - g \|_\infty &\le \max_{x \in \mathcal{X}} \langle x , \theta - \theta' \} \\
    &\le \| x \|_2 \| \theta - \theta'\|_2 \\
    &\le \| \theta - \theta'\|_2.
\end{align}
Since the $L_2$ covering number of $\mathbb{B}_2^d$ is bounded by $\left( \frac{2\sqrt{d}}{\xi} + 1\right)^d$, the result immediately follows by defining the covering of $\mathcal{G}$ as $\{ \langle \theta, \cdot \rangle : \theta \in \mathcal{K} \}$ where $\mathcal{K}$ is the optimal covering of $\mathbb{B}_2^d$ in $L_2$ norm.
\end{proof}

\begin{definition}[Discretization of discriminator space] \label{def:discretization}
Define $\mathcal{F}_t^\xi$ as the optimal covering of $\mathcal{F}_t$ under the $L_\infty$ norm in the sense of \Cref{def:covnum}. The discretized family of discriminators we consider is, $\mathcal{F}^\xi = \otimes_{t=1}^H \mathcal{F}^\xi_t$.
\end{definition}

\begin{lemma} \label{lemma:pert}
Suppose for all functions $f' \in \mathcal{F}_t^{\xi_1/H}$, simultaneously $J_{f'} (\pi^E) - J_{f'} (\piRE) \le \xi_2$. Then, for all discriminators $f \in \mathcal{F}$, $J_f (\pi^E) - J_f (\piRE) \le 2\xi_1 + \xi_2$.
\end{lemma}
\begin{proof}
Consider any discriminator $f \in \mathcal{F}$. By construction, there exists an $f' \in \mathcal{F}^{\xi_1/H}$ such that,
\begin{align}
    \| f - f' \|_\infty \le \xi_1/H.
\end{align}
Since for any policy $\pi$, the value $J_f (\pi)$ under a discriminator $f \in \mathcal{F}$ is an $H$-Lipschitz function of $f$, we can make a statement about how well $J_{f'} (\pi)$ approximates $J_f (\pi)$ for an appropriately chosen $f' \in \mathcal{F}^{\xi_1/H}$. In particular, the nearest (in $L_\infty$ norm) $f' \in \mathcal{F}^{\xi_1 /H}$ to $f \in \mathcal{F}$ satisfies that for any policy $\pi$,
\begin{align}
    | J_f (\pi) - J_{f'} (\pi)| \le H \times \frac{\xi_1}{H}.
\end{align}
As a consequence, for any discriminator $f \in \mathcal{F}$,
\begin{align}
    J_f (\pi^E) - J_f (\piRE) &\le |J_f (\pi^E) - J_{f'} (\pi^E)| + J_{f'} (\pi^E) - J_{f'} (\piRE) + |J_{f'} (\piRE) - J_f (\piRE)| \\
    &\le \xi_1 + \xi_2 + \xi_1 = 2\xi_1 + \xi_2.
\end{align}
\end{proof}

In particular, this means that if we minimize $J_{f'} (\dem) - J_{f'} (\piRE) \le \xi_2$ for all $f' \in \mathcal{F}^{1/\Nexp H}$, then we can ensure that for all $f \in \mathcal{F}$,
\begin{align}
    J_f (\dem) - J_f (\piRE) \le \frac{2}{\Nexp} + \xi_2.
\end{align}

This implies the following theorem,
\begin{theorem} \label{thm:main:covering}
Consider the policy $\piRE$ returned by \Cref{alg:re} where $\mathcal{F}$ is instead chosen as $\mathcal{F}^{\frac{1}{H \Nexp}}$ (as defined in \Cref{def:discretization}). Assume that $\pi^E \in \Pi$, the ground truth reward function $r_t \in \mathcal{F}_t$ which is assumed to be bounded (For all $f_t \in \mathcal{F}_t$, $\| f_t \|_\infty \le 1$). Choose $|D_1|, |D_2| = \Theta (\Nexp)$ and suppose $\Nrep \to \infty$. With probability $\ge 1 - 3 \delta$,
\begin{align}
    J (\dem) - J (\piRE) \lesssim & \ \mathcal{L}_1 + \mathcal{L}_2 
    + \frac{\log \left( \Nmax H/\delta \right) + 1}{\Nexp}
\end{align}
where $\Nmax \triangleq \max_{t \in [H]} \mathcal{N} (\mathcal{F}_t, 1/H\Nexp, \| \cdot \|_\infty)$ corresponds to the maximal covering number of the function classes $\mathcal{F}_t$, and,
\begin{align}
    \mathcal{L}_1 \triangleq H^2 \ \mathbb{E}_{\dem} \! \left[ \frac{\sum_{t=1}^H \mathcal{M} (s_t,t) \tv \left( \dem_t (\cdot|s_t) , \piBC_t (\cdot|s_t) \right)}{H} \right],
\end{align}
\begin{align}
    \mathcal{L}_2 \triangleq H^{3/2} \sqrt{ \frac{\log \left( \Nmax H/\delta \right)}{\Nexp} \frac{\sum_{t=1}^H \mathbb{E}_{\dem} \left[ 1 - \mathcal{M} (s_t, t) \right]}{H}}. \nonumber
\end{align}
\end{theorem}

\begin{remark} \label{rem:Fmax-Nmax}
Note that this line of reasoning can be extended to \Cref{thm:gfa} and \Cref{thm:alpha-gfa:rep} to show that the same guarantees as \cref{eq:gfa} and \cref{eq:alpha-gfa:rep} respectively hold, but with $\Fmax$ replaced by $\Nmax$.
\end{remark}

\subsection{Bounds on \RE in the linear expert setting} \label{sec:RE-linear-ub}

In this section, we provide an upper bound on the imitation gap of \RE in the presence of linear function approximation, which is studied in \cite{rajaraman2021on}. This is a special case of the case of IL under parametric function approximation with Lipschitzness. To avoid any ambiguity, we formally define IL with linear function approximation, which is the case when $(i)$ the expert follows an unknown linear classifier in a known set of features, and $(ii)$ the reward function admits a linear parameterization.

The goal is to show that there exists a simple choice of the membership oracle such that the imitation gap of the resulting algorithm grows as $H^{3/2}$ and decay in the size of the expert dataset as $1/\Nexp$ up to logarithmic factors, breaking the error compounding barrier and achieving the optimal dependency on these parameters. We first introduce the linear setting below.

\begin{assumption}[Linear-expert setting] 
\label{A1}
For each $(s, a, t)$ tuple, the learner is assumed to have a feature representation $\phi_t (s, a) \in \mathbb{R}^d$. For each time $t$, there exists an unknown vector $\theta^E_t \in \mathbb{R}^d$ such that $\forall s \in \mathcal{S}$,
\begin{align}
    \pi^E_t (s) = \argmax_{a \in \mathcal{A}} \langle \theta^E_t, \phi_t (s,a) \rangle.
\end{align}
i.e., the expert policy is deterministic and realized by a linear classifier. We assume that $\theta_t^E \in \mathbb{S}^{d-1}$ without loss of generality.
\end{assumption}

\begin{definition}[Policy induced by a linear classifier] \label{def:pilc}
Consider a set of vectors $\theta = \{ \theta_1,\cdots,\theta_H\}$ where each $\theta_t \in \mathbb{R}^d$. A policy $\pi^\theta$ is said to be induced by the set of linear classifiers defined by $\theta$ if for all $s \in \mathcal{S}$ and $t \in [H]$,
\begin{equation}
    \pi^\theta_t (s) = \argmax_{a \in \mathcal{A}} \langle \theta_t, \phi_t (s,a) \rangle.
\end{equation}
By this definition, $\dem = \pi^{\theta^E}$.
\end{definition}

\begin{definition}[Linear reward setting] \label{A:linreward} Define $\Rlint$ as the family of linear reward functions (defined at the single time-step $t$) which takes the form of an unknown linear function of a set of the features,
\begin{align}
    \Rlint = \left\{ \{ r_t (s, a) = \langle \omega, \phi_t (s, a) \rangle : s \in \mathcal{S}, a \in \mathcal{A} \} : \omega \in \mathbb{R}^d, \| \omega \|_2 \le 1 \right\}.
\end{align}
For the rewards to be $1$-bounded, we assume the features satisfy $\| \phi_t (s, a) \|_2 \le 1$. Define $\Rlin = \otimes_{t=1}^H \Rlint$. The linear reward setting assumes the true reward function of the MDP, $r \in \Rlin$.
\end{definition}

\begin{remark}
Note that our guarantees in \Cref{thm:linear} hold even if the set of features in the definition of $\Rlint$ in \Cref{A:linreward} differ from those used to define the expert classifier \Cref{A1}. Regardless, we assume that both sets of features are known to the learner.
\end{remark}

In the case of parametric function approximation with Lipschitzness, note that we assume both the weak margin condition (\Cref{A:wm:repeat}), as well as the existence of a linear classification oracle (\Cref{A:class}). Below, in the linear expert case, we show a sufficient condition which implies both of these conditions. In particular, define the positive hemisphere with pole at $\theta$, i.e. $\{ x : \mathbb{B}_2^d : \langle \theta, x \rangle \ge 0 \}$ as $\mathbb{H}^d_{\theta}$. We abbreviate $\mathbb{H}^d_{\theta^E_t}$ as $\mathbb{H}^d_t$.

\begin{assumption}[Bounded density assumption] \label{A:bd}
For each time $t \in [H]$, state $s \in \mathcal{S}$, action $a \in \mathcal{A}$ and $\theta \in \Theta_t$, define $\overline{\phi}_t (s,a) = \phi_t (s,a_s^{\theta}) - \phi_t (s,a)$ where $a_s^{\theta} = \argmax_{a' \in \mathcal{A}} \langle \theta, \phi_t (s, a') \rangle$. Consider the measure $\mathrm{Pr}_{\dem} \left( \exists a \ne a_{s_t}^\theta : \overline{\phi}_t (s_t,a) \in \cdot \right)$. Let $\overline{d}^{E}_t$ represent the Radon-Nikodym derivative of this measure against the uniform measure on $\mathbb{H}^{d-1}_t$. The bounded density assumption states that for each $t \in [H]$ there are constants $c_{\min} > 0$ and $c_{\max} < \infty$ such that for all $x \in \mathbb{H}^d_t$,
\begin{align}
    c_{\min} \le \overline{d}^{E}_t (x) \le c_{\max}.
\end{align}
\end{assumption}

We now state the main result we prove for IL in the linear setting.

\begin{theorem} \label{thm:linear} Under Assumptions~\ref{A1} and \ref{A:bd}, appropriately instantiating \RE ensures that with probability $\ge 1 - \delta$,
\begin{equation} J (\dem) - J (\piRE) \lesssim
    \sqrt{\frac{c_{\max}}{c_{\min}}} \frac{H^{3/2} d^{5/4} \log^{\frac{3}{2}} (\Nexp d H/\delta)}{\Nexp}. \nonumber
\end{equation}
\end{theorem}

The proof of this result follows by showing that under \Cref{A:bd}, both the weak margin condition (\Cref{A:wm:repeat}) is satisfied, and the classification oracle (\Cref{A:class}) can be constructed. We begin by showing the former.

\begin{lemma}
Under \Cref{A:bd}, the $\alpha$-weak margin condition is satisfied with $\alpha=1$ and $\mu = 2 c_{\max} \sqrt{d}$. In particular, for all $\theta \in \mathbb{S}^{d-1}$,
\begin{align}
    \mathrm{Pr}_{\dem} \left( \langle \theta , \phi_t (s,a_{s_t}^\theta) \rangle - \max_{a \ne a_{s_t}^\theta} \langle \theta, \phi_t (s,a) \rangle \ge \eta \right) \ge 1 - \left( 2 c_{\max} \sqrt{d} \right) \eta.
\end{align}
where $a_{s_t}^\theta \triangleq \argmax_{a \in \mathcal{A}} \langle \theta , \phi_t (s,a) \rangle$.
\end{lemma}
\begin{proof}
Observe that,
\begin{align}
    &\mathrm{Pr}_{\dem} \left( \exists a \ne a^\theta_{s_t} : \langle \theta , \phi_t (s_t,a^\theta_{s_t}) \rangle - \langle \theta, \phi_t (s_t,a) \rangle \le \eta \right) \\
    &\quad = \mathrm{Pr}_{\dem} \left( \exists a \ne a^\theta_{s_t} : \phi_t (s_t,a^\theta_{s_t}) - \phi_t (s_t,a) \in \{ x \in \mathbb{H}_\theta^d : \langle x , \theta \rangle \le \eta \} \right) \\
    &\quad \overset{(i)}{=} \mathrm{Pr}_{\dem} \left( \exists a \ne a^\theta_{s_t} : \overline{\phi}_t (s_t,a) \in \{ x \in \mathbb{H}_\theta^d : \langle x , \theta \rangle \le \eta \} \right) \\
    &\quad \overset{(ii)}{\le} c_{\max} \mathrm{Pr} ( \langle U,\theta \rangle \le \eta) \label{eq:0199100}
\end{align}
where in $(i)$, $\overline{\phi}_t$ is as defined in \Cref{A:bd} and in $(ii)$, $U$ is uniformly distributed on the unit hemisphere, $\mathbb{H}^d_\theta$. Note that $(ii)$ follows from the bounded density condition, \Cref{A:bd}. Note that the RHS essentially corresponds to the volume (probability measure) of a disc of height $\eta$ cut out of a sphere from the center. Up to normalization factors, this can be upper bounded by the surface area of the base of the disc, multiplied by the height of the disc. Namely,
\begin{align}
    \frac{\eta \times \frac{\pi^{\frac{d-1}{2}}}{\Gamma \left( \frac{d-1}{2}+1 \right)}}{\frac{1}{2} \frac{\pi^{\frac{d}{2}}}{\Gamma \left( \frac{d}{2}+1 \right)}}
\end{align}
Using Gautschi's inequality, for any $x \ge 0$ and $\ell \in (0,1)$ $x^{1-\ell} \le \frac{\Gamma (x+1)}{\Gamma (x+\ell)} \le (1+x)^{1-\ell}$. With $\ell = \frac{1}{2}$, $\frac{\Gamma(\frac{d}{2}+1)}{\Gamma(\frac{d+1}{2})} \le \sqrt{d}$. Combining with \cref{eq:0199100} results in,
\begin{align}
    \mathrm{Pr}_{\dem} \left( \exists a \ne a^\theta_{s_t} : \langle \theta , \phi_t (s_t,a^\theta_{s_t}) \rangle - \langle \theta, \phi_t (s_t,a) \rangle \le \eta \right) \le 2 c_{\max} \sqrt{d} \eta
\end{align}
Therefore the probability of the complement event is lower bounded by $1 - 2c_{\max} \sqrt{d} \eta$, completing the proof.
\end{proof}

The final thing to show is that the bounded density assumption can also be used to construct a classification oracle in the sense of \Cref{A:class}.

\begin{figure}
  \centering
  \includegraphics[width=0.49\textwidth]{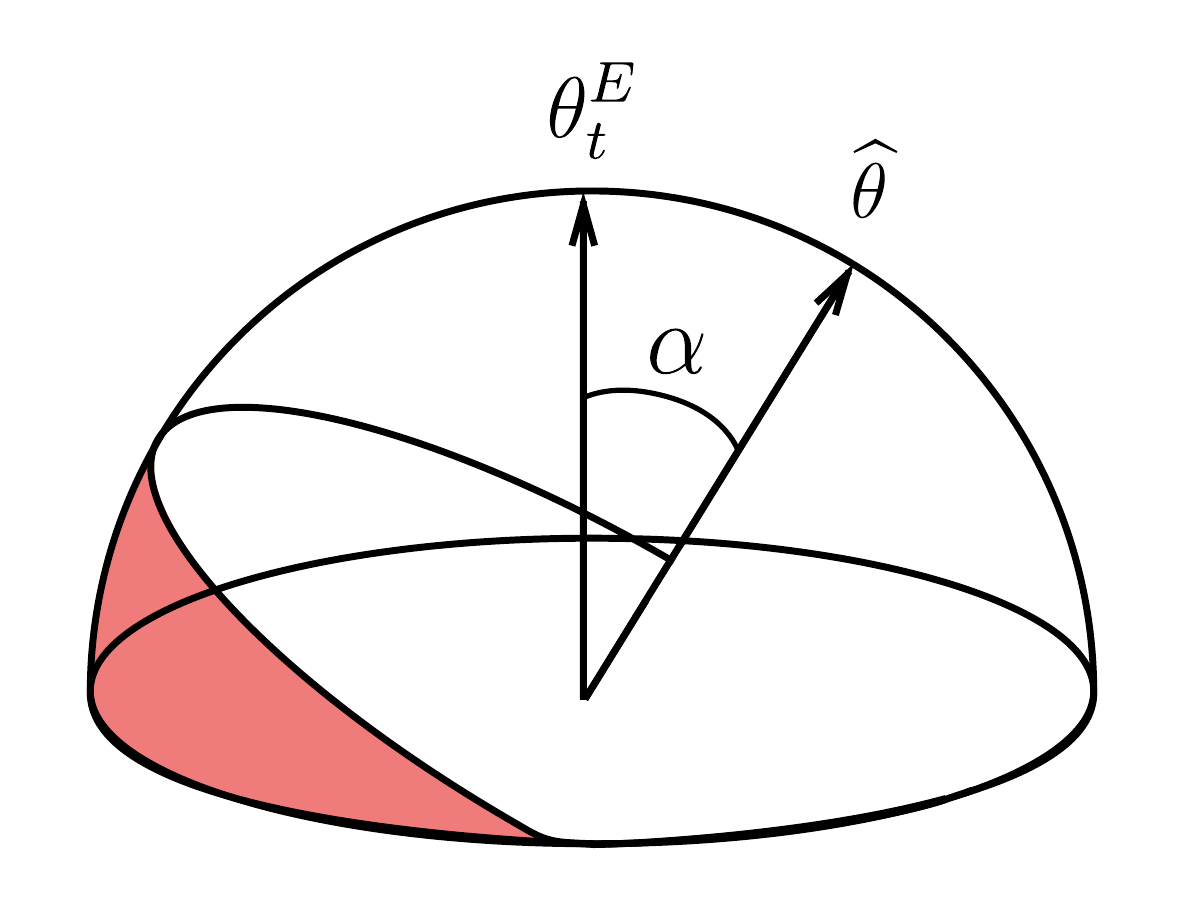}
  \caption{If at any time $t \in [H]$ and state $s$, $\phi_t (s,\dem_t (s)) - \phi_t (s,a)$ lies in the red shaded region for some action $a$, then, the action played by $\piBC$ and $\dem$ at this state are different.}
    \label{fig:disagreement}
\end{figure}

In particular, as discussed in the main paper, we show that algorithms for minimizing the generalization error, can be used to construct a classification oracle. The compression based algorithm of \cite{multiclass} provides a guarantee on the generalization error. From Theorem 5 of~\cite{multiclass}, in the realizable setting, for linear classification, the resulting classifier $\hat{\theta}$ has expected $0$-$1$ loss upper bounded by $(d+\log(1/\delta)) \log(n)/n$, given $n$ classification examples. Namely, in the notation of \Cref{A:class}, the resulting classifier $\hat{\theta}$ satisfies with probability $\ge 1 - \delta$,
\begin{align}
    \mathrm{Pr}_{s \sim \mathcal{D}} \left( \argmax_{a \in \mathcal{A}} f_{\theta^*} (s^i,a) \ne \argmax_{a \in \mathcal{A}} f_{\hat{\theta}} (s^i,a) \right) \le \frac{(d + \log(1/\delta) \log (n)}{n}. \label{eq:generr}
\end{align}
Next we show that under \Cref{A:bd}, this equation can be used to bound the error in the parameter space, $\| \theta^* - \hat{\theta} \|_2$. Namely, in \Cref{A:class}, we may choose $\mathcal{E}_{\mathbb{B}_2^d,n,\delta}$ as $\asymp  \frac{(d + \log(1/\delta) \log (n)}{n}$, up to constants depending on $c_{\min}$.

\begin{lemma} \label{lemma:Etheta}
Consider the compression based learner $\thetaBC_t = \hat{\theta}_t$ of \cite{multiclass} for multi-class linear classification. Then, under \Cref{A:bd}, with probability $\ge 1-\delta$,
\begin{align}
    \| \thetaBC_t - \theta^*_t \|_2 \le \frac{2 \pi}{c_{\min}} \frac{(d + \log(1/\delta) \log (\Nexp)}{\Nexp}
\end{align}
\end{lemma}
\begin{proof}
Fix $t \in [H]$. The generalization error of $\thetaBC_t = \hat{\theta}_t$ can be written as,
\begin{align}
    &\mathrm{Pr}_{\dem} \left( \argmax_{a \in \mathcal{A}} \langle \theta^*_t, \phi_t (s_t,a) \rangle \ne \argmax_{a \in \mathcal{A}} \langle \thetaBC_t, \phi_t (s_t,a) \rangle \right) \nonumber \\
    &\quad = \mathrm{Pr}_{\dem} \left( \exists a \ne \dem_t (s_t) : \phi_t (s_t, \dem_t (s_t)) - \phi_t (s_t,a) \in \mathcal{C} \right), \label{eq:kkll}
\end{align}
where $\mathcal{C}$ is illustrated in \cref{fig:disagreement} and is formally defined as,
\begin{align}
    \mathcal{C} \triangleq \{ x \in \mathbb{H}^d_t : \langle x, \thetaBC_t \rangle \le 0 \}.
\end{align}
On the states which ``belong'' to $\mathcal{C}$ (i.e. at those states $s$ where $\exists a \ne \dem_t (s_t) : \phi_t (s, \dem_t (s_t)) - \phi_t (s,a) \in \mathcal{C}$), there exists an action $a$ such that $\thetaBC_t$ is better correlated with this action than $a^*_s$. In other words, $\thetaBC_t$ and $\theta^*$ play different actions at this state. Note that $\mathcal{C}$ is essentially the set difference of two hemispheres with different poles. By the bounded density condition, \Cref{A:bd}, and \cref{eq:kkll},
\begin{align}
    \mathrm{Pr}_{\dem} \left( \dem_t (s_t) \ne \argmax_{a \in \mathcal{A}} \langle \thetaBC_t, \phi (s,a) \rangle \right) \ge c_{\min} \mathrm{Pr} \left( U \in \mathcal{C} \right), \label{eq:orpp}
\end{align}
where $U$ is uniformly distributed over $\mathbb{H}^d_\theta$. Referring to \cref{fig:disagreement}, we have that,
\begin{align}
    \mathrm{Pr} \left( U \in \mathcal{C} \right) = \frac{\alpha}{\pi}
\end{align}
where $\alpha$ is the angle between $\thetaBC_t$ and $\theta_t^E$. In particular, from \cref{eq:orpp},
\begin{align}
    \mathrm{Pr}_{\dem} \left( a^*_s \ne \argmax_{a \in \mathcal{A}} \langle \thetaBC_t, \phi (s,a) \rangle \right) \ge c_{\min} \frac{\alpha}{\pi} \ge c_{\min} \frac{\| \theta^* - \thetaBC_t \|_2}{\pi},
\end{align}
where in the last inequality, we use the fact that $\| \theta^* \|_2 = \| \thetaBC_t \|_2 = 1$ without loss of generality. By the generalization error bound on $\thetaBC_t = \hat{\theta}_t$ in \cref{eq:generr}, with probability $\ge 1-\delta$,
\begin{align}
    \| \theta^* - \thetaBC_t \|_2 \le \frac{\pi}{c_{\min}}\frac{(d + \log(1/\delta) \log (\Nexp)}{\Nexp}
\end{align}
\end{proof}

\Cref{lemma:Etheta} shows that under the bounded density condition \Cref{A:bd}, the compression based learner $\hat{\theta}$ of \cite{multiclass} essentially induces a classification oracle for linear classification with $\mathcal{E}_{\mathbb{B}_2^d,n,\delta} = \frac{\pi}{c_{\min}}\frac{(d + \log(1/\delta) \log (n)}{n}$. Finally, from \Cref{corr:linearcov}, we have a bound on the covering number of linear families. Putting together all of these results with \Cref{thm:gfa} (noting that we can replace $\Fmax$ by $\Nmax$ from \Cref{rem:Fmax-Nmax}) results in \Cref{thm:linear}.

\section{Additional experimental results}
\label{app:more_res}

We include some additional experimental results in this section of the paper.

In \cref{fig:pdists}, we plot the distributions of the prefix weights generated by each membership oracle on simulated \BC rollouts on WalkerBulletEnv. Note that $\mathcal{M}_{\texttt{VAR}}$ is significantly overconfident in prefix weights compared to $\mathcal{M}_{\texttt{EXP}}$, as indicated by the heavier right-tail. On the other hand, $\mathcal{M}_{\texttt{RND}}$ and $\mathcal{M}_{\texttt{MAX}}$ are less overconfident and better overlap with the idealized prefix weights induced by $\mathcal{M}_{\texttt{EXP}}$. This aligns with the correlation plot between the various membership oracles in \figref{fig:res}. Moreover, in terms of policy performance, this further justifies the superior behavior of $\mathcal{M}_{\texttt{MAX}}$ compared to $\mathcal{M}_{\texttt{VAR}}$.

\begin{figure}[!h]
    \centering
    \includegraphics[width=0.3\textwidth]{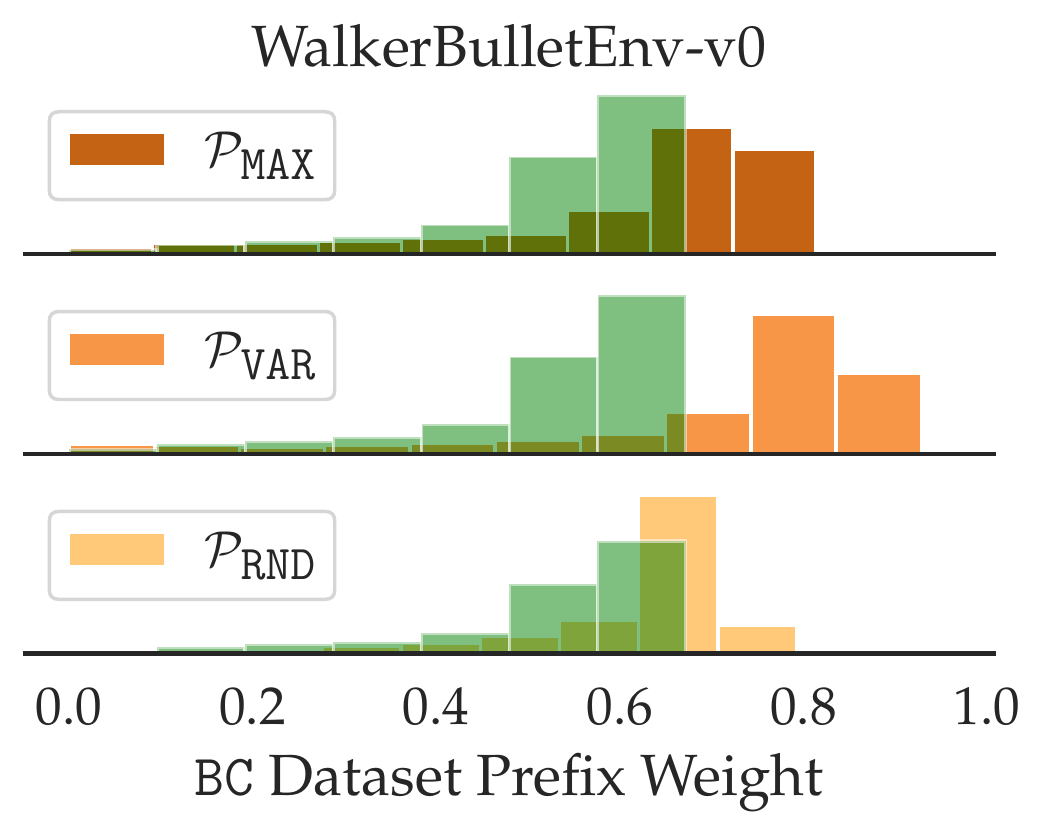}
    \caption{Histogram of prefix weights generated by rolling out trajectories from \BC. The green superimposed histogram represents prefix weights generated by $\mathcal{M}_{\texttt{EXP}}$}
    \label{fig:pdists}
\end{figure}

In \figref{fig:ablation}, we consider how each of the changes we described in the main section of the paper lead to improved performance of our \RE baseline. The first, using a Wasserstein distance, leads to lower expected return but is required for solving the full moment-matching problem -- see \citet{swamy2021moments} for more details. Switching from PPO to the more sample-efficient SAC \citep{haarnoja2018soft} leads to fast learning. Adding in gradient penalties for discriminator stability \citep{swamy2021moments, gulrajani2017improved} also improves final performance and learning speed. The last change we employ, using Optimistic Mirror Descent \citep{daskalakis2017training} in both the discriminator and RL algorithm also (slightly) improves performance. To our knowledge, we are the first to utilize this technique in the imitation learning literature and reccomend it as best practice for future moment-matching algorithms. We refer interested readers to the work of \citet{syrgkanis2015fast} for theoretical details of why OMD enables superior performance.

\begin{figure}[h]
    \centering
    \includegraphics[width=0.6\textwidth]{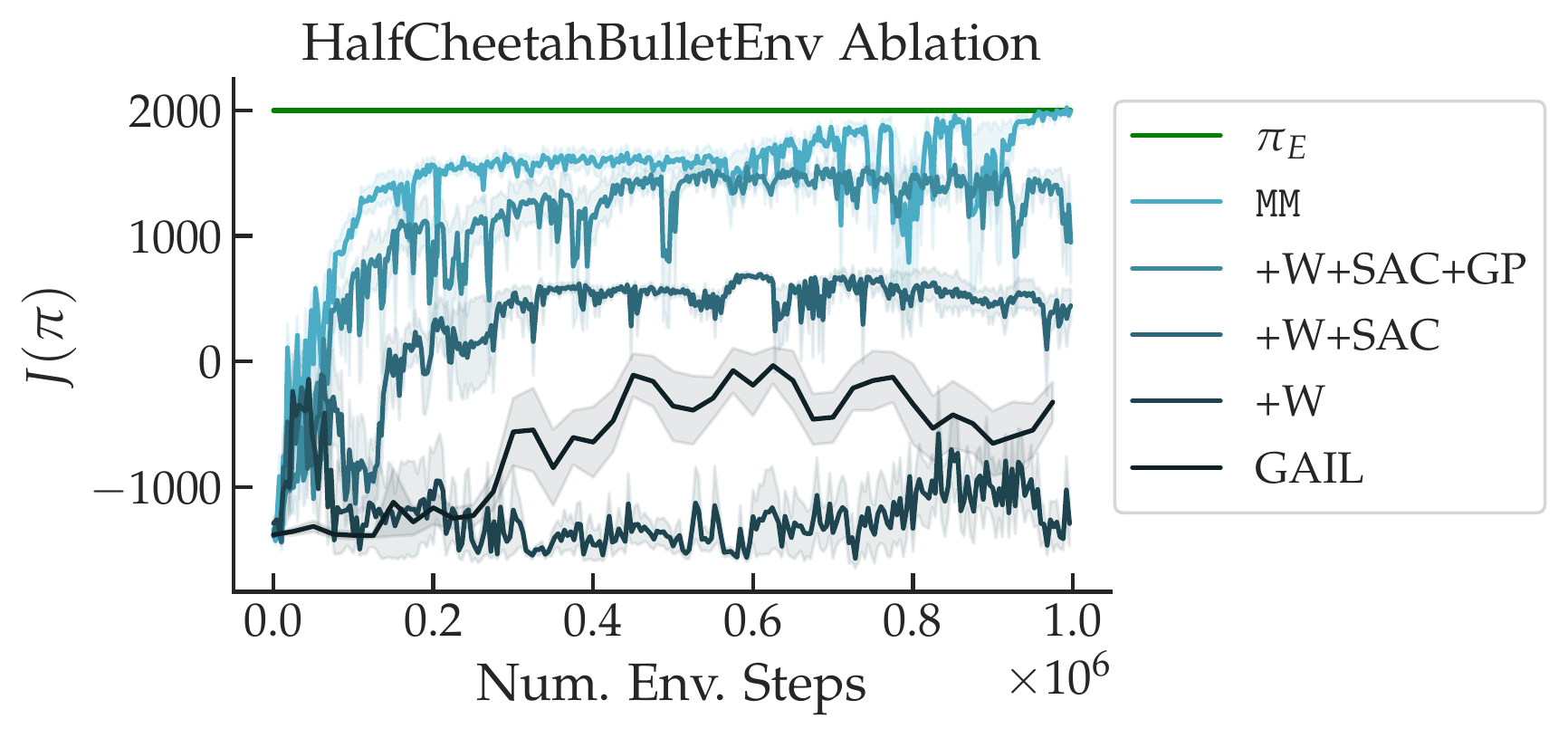}
    \caption{We ablate the four key changes we made to off-the-shelf GAIL to improve performance / theoretical guarantees. We see that each improved performance, with \MM significantly out-performing options with fewer changes. Our improvements upon \MM with the Replay Estimation technique are therefore improving upon an already strong baseline.}
    \label{fig:ablation}
\end{figure}

\newpage

\section{Experimental Setup}

We begin with the hyperparameters for our Standard Bullet and Noisy Bullet experiments.

\subsection{Expert}
We use the Stable Baselines 3 \citep{stable-baselines3} implementation of PPO \citep{DBLP:journals/corr/SchulmanWDRK17} or SAC \citep{haarnoja2018soft} to train experts for each environment. For the most part, we use already tuned hyperparameters from \citep{rl-zoo3} in the implementation. The modifications we used are are shown in table \ref{table:stbsln3params}.

\begin{table}[h]
\begin{center}
\begin{small}
\begin{sc}
\setlength{\tabcolsep}{2pt}
\begin{tabular}{lccccccccccr}
\toprule
 Parameter & Value \\
\midrule
 buffer size & 300000 \\
 batch size & 256 \\
 $\gamma$ & 0.98 \\
 $\tau$ & 0.02 \\
 Training Freq. & 64 \\
 Gradient Steps & 64 \\
 Learning Rate & Lin. Sched. 7.3e-4 \\
 policy architecture & 256 x 2 \\
 state-dependent exploration & true \\
 training timesteps & 1e6 \\
\bottomrule
\end{tabular}
\end{sc}
\end{small}
\end{center}
\caption{\label{table:stbsln3params} Expert hyperparameters for Walker Bullet Task and Hopper Bullet Task}
\end{table}

\subsubsection{Noisy Experts}
In addition to the default Bullet Tasks, we test performance of algorithms on noisy environments. Namely, we generate noisy expert data by re-training expert policies with Gaussian noise added to the actions of the expert during the exploration phase while training. We then re-generate expert data by sampling from the expert policies trained on noisy data to analyze the performance of our method under stochasticity. Table \ref{table:noisyexpert} lists the standard deviation of the (i.i.d.) noise we applied to the actions in the different environments.

\begin{table}[h]
\begin{center}
\begin{small}
\begin{sc}
\setlength{\tabcolsep}{2pt}
\begin{tabular}{lcccr}
\toprule
 env. & Noise Distribution. \\
\midrule
 hopper & $\mathcal{N}(0, 0.1)$ \\
 walker & $\mathcal{N}(0, 0.5)$\\
\bottomrule
\end{tabular}
\end{sc}
\end{small}
\end{center}
\caption{\label{table:noisyexpert} Noise we applied to all policies in each environment.}
\end{table}

\subsection{Baseline}
We average over 5 runs and use a common architecture of 256 x 2 with ReLU activations for both our method and the \texttt{MM} baseline we compare against. For each datapoint, the cumulative reward is averaged over 10 trajectories. For all tasks, we train on $\{6, 12, 18\}$ expert trajectories with a maximum of 400k iterations of the optimization procedure. Table \ref{table:learnerparams} shows the hyperparameters we used for \texttt{MM}. Empirically, smaller learning rates, large batch sizes, and gradient penalties were critical for the stable convergence of our method.

\begin{table}[ht]
\begin{center}
\begin{small}
\begin{sc}
\setlength{\tabcolsep}{2pt}
\begin{tabular}{lcccccr}
\toprule
 Parameter & Value \\
\midrule
 Batch Size & 2048* \\
 Learning Rate & Linear Schedule of 8e-3* \\
 $f$ Update Freq. & 5000 \\
 $f$ gradient target & 0.4 \\
 $f$ gradient penalty weight & 10 \\
\bottomrule
\end{tabular}
\end{sc}
\end{small}
\end{center}
\caption{\label{table:learnerparams} Learner hyperparameters for \MM. * indicates the parameter was different for the Hopper Initial State shift experiments (4096 for batch size and Linear Schedule of 8e-4, respectively.).}
\end{table}

We note that \texttt{MM} requires careful tuning of \textsc{$f$ Update Freq.} for strong performance. We searched over step sizes of $\{ 1250, 2500, 5000 \}$ and selected the one which achieved the most stable updates. In practice, we recommend evaluating a trained policy on a validation set to set this parameter.

We also used similar parameters for training SAC, also from the Stable Baselines 3 \citep{stable-baselines3} implementation, as we did for training the expert policy. Table \ref{table:sacparams} shows the choice of hyperparmeters we used for training SAC. We directly added in the optimistic mirror descent optimizers \citep{daskalakis2017training} for both the critic and actor objectives of SAC.

\begin{table}[ht]
\begin{center}
\begin{small}
\begin{sc}
\setlength{\tabcolsep}{2pt}
\begin{tabular}{lcccccr}
\toprule
 Parameter & Value \\
\midrule
 $\gamma$ & 0.98 \\
 $\tau$ & 0.02 \\
 Training Freq. & 64 \\
 Gradient Steps & 64 \\
 Learning Rate & Linear Schedule of 7.3e-4 \\
 policy architecture & 256 x 2 \\
\bottomrule
\end{tabular}
\end{sc}
\end{small}
\end{center}
 \caption{\label{table:sacparams} Leaning hyperparameters for the SAC component of \MM}
\end{table}

Table \ref{table:bcparams} shows the learning hyperparameters for any \BC policies used for generating simulated data for the membership oracles. Table \ref{table:trainingsteps} shows the number of training steps per task we used for both the baseline and our method.

\begin{table}[ht]
\begin{center}
\begin{small}
\begin{sc}
\setlength{\tabcolsep}{2pt}
\begin{tabular}{lcccr}
\toprule
 Parameter & Value \\
\midrule
 entropy weight & 0 \\
 l2 weight & 0 \\
 training timesteps & 1e5 \\
\bottomrule
\end{tabular}
\end{sc}
\end{small}
\end{center}
\caption{\label{table:bcparams} Learner hyperparameters for Behavioral Cloning}
\end{table}

\begin{table}[ht]
\begin{center}
\begin{small}
\begin{sc}
\setlength{\tabcolsep}{2pt}
\begin{tabular}{lcccr}
\toprule
 env. & training steps \\
\midrule
 walker (no noise) & 400000 \\
 walker (with noise) & 400000 \\
 hopper (no noise) & 400000 \\
 hopper (with noise) & 400000 \\
\bottomrule
\end{tabular}
\end{sc}
\end{small}
\end{center}
\caption{\label{table:trainingsteps} Number of training steps for the different tasks}
\end{table}

\subsection{Our Algorithm}
In this section, we use \textbf{bold text} to highlight sensitive hyperparameters. We use the same network architecture choices as the  \texttt{MM} baseline. For all environments, we generated 100 trajectories of simulated behavior cloning data to use with our method.

% \begin{table}[h]
% \begin{center}
% \begin{small}
% \begin{sc}
% \setlength{\tabcolsep}{2pt}
% \begin{tabular}{lcccr}
% \toprule
%  env. & learning rate & batch size \\
% \midrule
%  hopper & 8e-5 & 512 \\
%  walker & 8e-4 & 2500 \\
% \bottomrule
% \end{tabular}
% \end{sc}
% \end{small}
% \end{center}
% \caption{\label{table:batchandlrparams} Sensitive hyperparameters for our improved version of GAIL}
% \end{table}

For all tasks, we rolled out \textbf{100 trajectories} from a \BC trained network to use with our membership oracle. Table \ref{table:datasplit} shows how we partitioned our dataset between the \BC training set and the expert membership oracle dataset. We also use the full dataset for moment matching, not just $D_2$, as we found this lead to slightly better performance.

\begin{table}[!h]
\begin{center}
\begin{small}
\begin{sc}
\setlength{\tabcolsep}{2pt}
\begin{tabular}{lcccr}
\toprule
 Expert Size & $D_1$ & $D_2$ \\
\midrule
 6 trajs & 4 & 2 \\
 12 trajs & 10 & 2 \\
 18 trajs & 16 & 2 \\
\bottomrule
\end{tabular}
\end{sc}
\end{small}
\end{center}
\caption{\label{table:datasplit} Partition of trajectories into $D_1$ and $D_2$ based on the number of expert trajectories provided. For the Noisy Walker experiments, we used $5, 10, 14$ trajectories for $D_1$ instead of the above.}
\end{table}

\subsection{Membership Oracle Parameters}
For both $\mathcal{M}_{\texttt{VAR}}$ and $\mathcal{M}_{\texttt{MAX}}$, we use {5 \BC networks} in the ensemble. 
%In training the \BC imitator network for $\mathcal{M}_{\texttt{RND}}$, we generate a new set of expert data from the allowed number of expert trajectories for the oracle and included the original expert states combined with the predicted actions on those states from a \BC trained policy. 
We followed the exact same parameters in Table \ref{table:bcparams} to train each \BC imitator.
\begin{table}[!h]
\begin{center}
\begin{small}
\begin{sc}
\setlength{\tabcolsep}{2pt}
\begin{tabular}{lccccr}
\toprule
 env & parameter & $\mathcal{M}_{EXP}$ & $\mathcal{M}_{\texttt{RND}}$ & $\mathcal{M}_{\texttt{VAR}}$ & $\mathcal{M}_{\texttt{MAX}}$ \\
\midrule
 walker & $\beta$ & 0.1 & 0.1 & 0.01 & 0.1\\
 walker & $\mu$ & 0.33 & 0.22 & 0.015 & 0.35\\
 hopper & $\beta$ & 0.8 & 0.25 & 0.08 & 0.1\\
 hopper & $\mu$ & 0.68 & 0.4 & 0.05 & 0.25\\
\bottomrule
\end{tabular}
\end{sc}
\end{small}
\end{center}
\caption{\label{table:betamuparams} Membership oracle hyperparameters across different environments}
\end{table}
Table \ref{table:betamuparams} shows the choice of $\mu$ and $\beta$ values we used for each membership oracle.

% For both $\mathcal{M}_{\texttt{VAR}}$ and $\mathcal{M}_{\texttt{MAX}}$, we use \textbf{5 \BC networks} in the ensemble. In training the \BC imitator network $\widehat{\piBC}(s)$ for $\mathcal{M}_{\texttt{RND}}$, we generate a new set of expert data (within the allowed budget of expert trajectories) and include the original expert states combined with the predicted actions on those states from a \BC trained policy. We follow the exact same choice of parameters in Table \ref{table:bcparams} to train the \BC imitator network.

\subsection{Initial State Shift Experiments}
For these experiments, we used demonstrations generated by an expert trained on the standard Bullet tasks but subject the learner (both at train and test time) to a initial velocity perturbation of a zero-mean Gaussian with variance ($\sigma=1e-7$). We refer interested readers to our code for our precise method of injecting noise as we believe it might be of interest for future experiments. We note that in all the demonstrations, we see the expert start from rest. Despite this relatively small shift, we see \BC performance drop significantly, as is characteristic of real-world problems where it significantly under-performs on-policy IL methods. All results are averaged over five seeds.

For all environments, we train \BC for 1e5 steps (as well as for the query policies for \RE).

For \RE, we train $5$ policies and use the max-distance approximate membership oracle. We use the above parameters for \MM for our base moment-matcher.

% \begin{table}[h]
% \begin{center}
% \begin{small}
% \begin{sc}
% \setlength{\tabcolsep}{2pt}
% \begin{tabular}{lcccccr}
% \toprule
%  Parameter & Value \\
% \midrule
%  Batch Size & 2048 \\
%  Learning Rate & 8e-5 \\
%  $f$ Update Freq. & 2500 \\
%  $f$ gradient target & 0.4 \\
%  $f$ gradient penalty weight & 10 \\
% \bottomrule
% \end{tabular}
% \end{sc}
% \end{small}
% \end{center}
% \caption{\label{table:learnerparams2} Learner hyperparameters for \MM for initial state shift experiments.}
% \end{table}

% \begin{table}[h]
% \begin{center}
% \begin{small}
% \begin{sc}
% \setlength{\tabcolsep}{2pt}
% \begin{tabular}{lcccr}
% \toprule
%  Expert Size & $D_1$ & $D_2$ \\
% \midrule
%  6 trajs & 2 & 4 \\
%  12 trajs & 2 & 10 \\
%  18 trajs & 2 & 14 \\
% \bottomrule
% \end{tabular}
% \end{sc}
% \end{small}
% \end{center}
% \caption{\label{table:datasplit2} Partition of trajectories into $D_1$ and $D_2$ based on the number of expert trajectories provided for initial state shift experiments.}
% \end{table}

\begin{table}[!h]
\begin{center}
\begin{small}
\begin{sc}
\setlength{\tabcolsep}{2pt}
\begin{tabular}{lccccr}
\toprule
 env & parameter & $\mathcal{M}_{\texttt{MAX}}$ \\
\midrule
 walker & $\beta$ & 0.01\\
 walker & $\mu$ & 0.0001\\
 hopper & $\beta$ & 0.01\\
 hopper & $\mu$ & 0.0001\\
\bottomrule
\end{tabular}
\end{sc}
\end{small}
\end{center}
\caption{\label{table:betamuparams2} Membership oracle hyperparameters across different initial state shift environments.}
\end{table}

\end{document}